\documentclass[11pt]{article}

\usepackage[T1]{fontenc}
\usepackage[utf8]{inputenc}
\usepackage{amsmath}
\usepackage[charter,cal=cmcal]{mathdesign}
\usepackage{iftex}
\ifPDFTeX
  \usepackage[activate={true,nocompatibility},final]{microtype}
\else
  \usepackage[protrusion=true,final]{microtype}
\fi

\usepackage[margin=1in]{geometry}

\usepackage{authblk}

\usepackage{mathtools}
\usepackage{bm}
\usepackage{nicefrac}
\allowdisplaybreaks

\usepackage{amsthm}
\usepackage{graphicx}
\graphicspath{{figures/}{./}}
\usepackage{booktabs}
\usepackage{multirow,makecell,array}
\usepackage[font=small,labelfont=bf]{caption}
\usepackage{subcaption}
\usepackage{algorithm}
\usepackage{algpseudocode}

\usepackage{enumitem}
\setlist[itemize]{leftmargin=2.2em,itemsep=2pt,topsep=2pt}
\setlist[enumerate]{leftmargin=2.2em,itemsep=2pt,topsep=2pt}

\usepackage{xcolor}
\definecolor{LinkColor}{rgb}{0.10,0.40,0.75}
\definecolor{CiteColor}{rgb}{0.70,0.25,0.20}
\definecolor{UrlColor} {rgb}{0.20,0.50,0.50}
\definecolor{TodoColor}{rgb}{0.80,0.30,0.10}

\usepackage{tikz}
\usetikzlibrary{positioning,calc,arrows.meta}

\usepackage[round,sort&compress]{natbib}

\usepackage{url}
\usepackage{hyperref}
\hypersetup{
  colorlinks=true,
  linkcolor=LinkColor,
  citecolor=CiteColor,
  urlcolor=UrlColor,
  breaklinks=true,
  bookmarksnumbered=true,
}
\usepackage{bookmark}
\usepackage[capitalise,nameinlink,noabbrev,sort&compress]{cleveref}

\numberwithin{equation}{section}

\newcommand{\R}{\mathbb{R}}

\newcommand{\E}{\mathbb{E}}
 
\renewcommand{\P}{\mathbb{P}}

\newcommand{\eps}{\varepsilon}

\newcommand{\defeq}{\coloneqq}

\DeclareMathOperator{\sign}{sign}

\DeclareMathOperator{\supp}{supp}

\DeclareMathOperator{\poly}{poly}

\DeclarePairedDelimiter{\abs}{\lvert}{\rvert}
\DeclarePairedDelimiter{\norm}{\lVert}{\rVert}

\DeclarePairedDelimiter{\ceil}{\lceil}{\rceil}
\DeclarePairedDelimiter{\floor}{\lfloor}{\rfloor}
\DeclarePairedDelimiterX{\inner}[2]{\langle}{\rangle}{#1,#2}

\newcommand{\Dx}{\mathcal{D}_X}
\newcommand{\Dcal}{\mathcal{D}}

\newcommand{\Sclean}{S_C}
\newcommand{\Sdirty}{S_D}
\newcommand{\linsumnorm}{\mathrm{LSN}}

\newcommand{\pancake}{\mathcal{P}}
\newcommand{\err}{\mathrm{err}}
\newcommand{\TV}{\mathrm{TV}}
\providecommand{\eps}{\varepsilon}
\providecommand{\sgn}{\operatorname{sgn}}

\providecommand{\poly}{\operatorname{poly}}
\newcommand{\Esudo}{\tilde{\mathbb{E}}}
\DeclareMathOperator{\hinge}{hinge}
\providecommand{\D}{\mathcal{D}}

\theoremstyle{plain}
\newtheorem{theorem}{Theorem}[section]
\newtheorem{proposition}[theorem]{Proposition}
\newtheorem{lemma}[theorem]{Lemma}
\newtheorem{corollary}[theorem]{Corollary}

\newtheorem{fact}[theorem]{Fact}

\theoremstyle{definition}
\newtheorem{definition}[theorem]{Definition}

\newtheorem{example}[theorem]{Example}

\theoremstyle{remark}
\newtheorem{remark}[theorem]{Remark}

\crefname{theorem}{Theorem}{Theorems}          \Crefname{theorem}{Theorem}{Theorems}
\crefname{proposition}{Proposition}{Propositions}
\Crefname{proposition}{Proposition}{Propositions}
\crefname{lemma}{Lemma}{Lemmas}                \Crefname{lemma}{Lemma}{Lemmas}
\crefname{corollary}{Corollary}{Corollaries}   \Crefname{corollary}{Corollary}{Corollaries}
\crefname{conjecture}{Conjecture}{Conjectures} \Crefname{conjecture}{Conjecture}{Conjectures}
\crefname{fact}{Fact}{Facts}                   \Crefname{fact}{Fact}{Facts}
\crefname{definition}{Definition}{Definitions} \Crefname{definition}{Definition}{Definitions}
\crefname{assumption}{Assumption}{Assumptions} \Crefname{assumption}{Assumption}{Assumptions}
\crefname{example}{Example}{Examples}          \Crefname{example}{Example}{Examples}
\crefname{problem}{Problem}{Problems}          \Crefname{problem}{Problem}{Problems}
\crefname{remark}{Remark}{Remarks}             \Crefname{remark}{Remark}{Remarks}
\crefname{claim}{Claim}{Claims}                \Crefname{claim}{Claim}{Claims}
\crefname{algorithm}{Algorithm}{Algorithms}    \Crefname{algorithm}{Algorithm}{Algorithms}

\usepackage{thm-restate}

\title{\vspace{-3em}Sum-of-Squares Degree Barriers for the\\Reweighted-Hinge Method in Robust Halfspace Learning:\\[0.2em]A Christoffel-Function Characterization}

\author{Xiaoyu Li}
\affil{University of New South Wales}
\affil{\texttt{xiaoyu.li2@unsw.edu.au}}
\date{\vspace{-1.3em}\today}

\begin{document}
\maketitle
\thispagestyle{empty}
\vspace{-2em}
\enlargethispage{2.4cm}

\begin{abstract}
A certificate that removes outliers sees the data only through its low-degree moments, and an adversary exploits exactly this, hiding corruption where the clean data already looks typical, in the blind spot no bounded-degree test resolves. That blind spot turns out to have an exact size: the Christoffel function of the clean marginal, the very quantity modern data analysis thresholds to \emph{detect} outliers, here read from the adversary's side as the corruption a bounded-degree certificate cannot \emph{remove}. We turn this inversion into the organizing principle of the reweighted-hinge approach to robustly learning $\gamma$-margin halfspaces under malicious noise \citep{Shen25,ZengShen25}: the governing resource is the Sum-of-Squares \emph{degree} of the outlier-removal certificate, and the \emph{resolution principle} states that the maximal corruption mass which can hide at a center $c$ from a degree-$2t$ certificate is exactly the Christoffel function $\lambda_{t+1}(c)$ of the clean marginal. Three consequences follow, all against the certificate method (not information-theoretic). A \emph{margin--degree tradeoff}: certifying the dense pancake to error $\eps$ costs SoS degree $\Omega(\log(1/\eps))$ or margin $\Omega(\sqrt{\log(1/\eps)}/\sqrt d)$, explaining why the $\log(1/\eps)$ margin \citet{Shen25} records is \emph{forced}, with a weighted-Chebyshev reduction making the threshold $2t=\Theta((\abs{c}/s)^2)$ tight modulo one classical weighted-extremal estimate. A \emph{degree-$2$ outlier barrier}: the resolution principle realized as an explicit instance on which degree $2$ is stuck at $\eta^{1/2}$ while degree $4$ escapes, locating the method's small breakdown rate in the degree, not the analysis. And a \emph{degree-$2t$ algorithm} tracing the frontier $\eta^{1-1/2t}$ (recovering \citet{Shen25} at $t=1$), whose gain is an explicit constant, capped by the pancake density and shown unimprovable by the degree-$2$ barrier. An information-theoretic floor of $\eta/(2(1-\eta))$ completes the picture: it is matched exactly from above, and under a hard margin its two-point realizations provably require $\Theta(1/\eta)$ mixture components.

\end{abstract}

\begin{center}
\resizebox{0.86\textwidth}{!}{%
\begin{tikzpicture}[>=Stealth,line cap=round]
  \def\pw{6.0}
  \def\ph{3.6}
  \def\gap{1.7}

  \begin{scope}[shift={(0,0)}]
    \draw[->,black] (0,0) -- (\pw+0.15,0) node[right,font=\scriptsize] {$c$};
    \draw[->,black] (0,0) -- (0,\ph+0.25) node[left,font=\scriptsize] {$\lambda_{t+1}(c)$};
    \draw[black,thick]
       (0.15,3.25)
       .. controls (1.1,3.15) and (1.8,2.5) ..
       (2.6,1.65)
       .. controls (3.5,0.92) and (4.4,0.6) ..
       (5.7,0.45);
    \draw[black!55,densely dotted] (1.0,0) -- (1.0,2.98);
    \fill[black] (1.0,2.98) circle (1.6pt);
    \node[font=\scriptsize,anchor=west,fill=white,inner sep=1pt] at (2.05,2.98) {large $\lambda$ here};
    \draw[->,black] (2.0,2.98) -- (1.12,2.98);
    \node[font=\scriptsize,align=center,anchor=center] at (3.7,2.05)
         {\textbf{keep the point}\\(it is typical)};
    \fill[black] (5.0,0.555) circle (1.6pt);
    \node[font=\scriptsize,align=center,anchor=south] at (5.2,0.7) {small $\lambda$: outlier};
    \node[font=\scriptsize,align=center,anchor=north] at (\pw/2,-0.55)
         {\textbf{data analysis}\\large $\lambda$ $=$ typical (keep the point)};
  \end{scope}

  \draw[black!35] (\pw+\gap/2,-1.15) -- (\pw+\gap/2,\ph+0.05);
  \node[black!60,font=\scriptsize,align=center,anchor=south] at (\pw+\gap/2,\ph+0.45)
        {\emph{same $\lambda_{t+1}$,}\\\emph{opposite reading}};

  \begin{scope}[shift={(\pw+\gap,0)}]
    \draw[->,black] (0,0) -- (\pw+0.15,0) node[right,font=\scriptsize] {$c$};
    \draw[->,black] (0,0) -- (0,\ph+0.25) node[right,font=\scriptsize] {$\lambda_{t+1}(c)$};
    \fill[black!16] (0.15,0) -- (0.15,3.25)
       .. controls (1.1,3.15) and (1.8,2.5) .. (2.0,2.18) -- (2.0,0) -- cycle;
    \draw[black,thick]
       (0.15,3.25)
       .. controls (1.1,3.15) and (1.8,2.5) ..
       (2.6,1.65)
       .. controls (3.5,0.92) and (4.4,0.6) ..
       (5.7,0.45);
    \draw[black!55,densely dotted] (1.0,0) -- (1.0,2.98);
    \fill[black] (1.0,2.98) circle (1.6pt);
    \node[font=\scriptsize,anchor=west,fill=white,inner sep=1pt] at (2.05,2.98) {large $\lambda$ here};
    \draw[->,black] (2.0,2.98) -- (1.12,2.98);
    \node[font=\scriptsize,align=center,anchor=center] at (3.8,2.05)
         {\textbf{adversary hides}\\(impossible to clear)};
    \fill[black] (5.0,0.555) circle (1.6pt);
    \node[font=\scriptsize,align=center,anchor=south] at (5.35,0.7) {small $\lambda$: certifiable};
    \node[font=\scriptsize,align=center,anchor=north] at (\pw/2,-0.55)
         {\textbf{this paper}\\large $\lambda$ $=$ the adversary can hide here};
  \end{scope}

\end{tikzpicture}%
}
\captionsetup{font=footnotesize}%

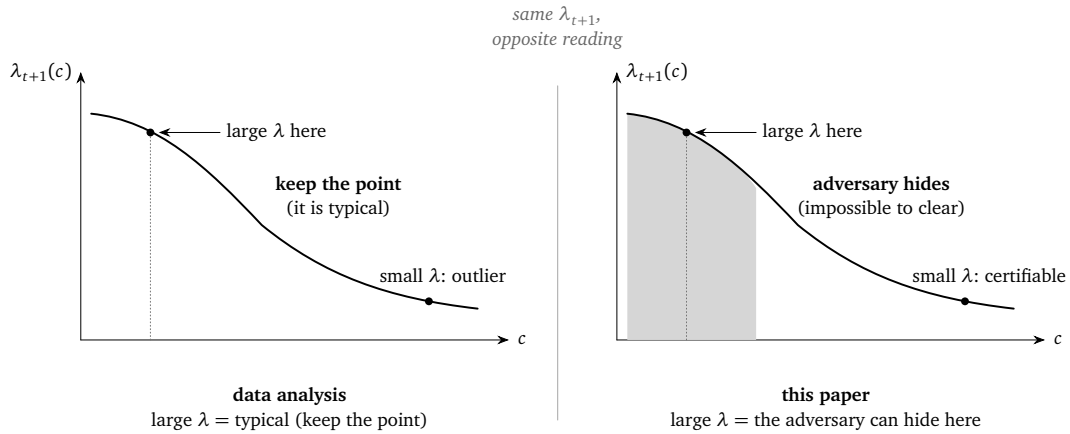
\captionof{figure}{\textbf{The inversion.}
The same Christoffel curve $\lambda_{t+1}(c)$ carries two opposite meanings.
In moment-based data analysis a large value marks a typical point to keep; here a large
value is exactly the corruption mass an adversary can hide at $c$ that no degree-$2t$
certificate can clear. One object, read as a detector on the left and as a barrier on the right.}
\label{fig:inversion}
\end{center}

\newpage
{\renewcommand{\addvspace}[1]{\vspace{0.4em}}\linespread{1.0}\selectfont
\tableofcontents}
\newpage

\section{Introduction}
\label{sec:intro}

The problem is to learn a halfspace efficiently when a constant fraction of the training data has been corrupted by an adversary. In the malicious-noise model of \citet{KearnsLi93}, an adversary observes the clean sample, then replaces an $\eta$-fraction of the (instance, label) pairs by arbitrary points of its choosing, with full knowledge of the learner and the clean distribution; the goal is to output $\hat w$ whose error on the \emph{clean} inliers is at most $\eps$. This is the strongest of the standard corruption models, and the central question is how large a noise rate $\eta$ a polynomial-time learner can tolerate.

One influential line of work answers this through a single algorithmic device: hinge-loss minimization, later reweighted to suppress outliers. \citet{Talwar20} observed that plain hinge-loss minimization already resists noise when the clean marginal places dense mass in a band (a ``pancake'') around every point's projection onto the separating direction. \citet{Shen25} turned this into a constant-noise-rate guarantee for log-concave-mixture marginals by \emph{reweighting} the hinge: the weights come from a degree-$2$ outlier-removal program (a variance constraint), and the analysis controls the dirty contribution to the hinge gradient through the quantity $\linsumnorm(q \circ \Sdirty) = \sup_{\norm{w}\le 1} \sum_{i \in \Sdirty} q_i \abs{w \cdot x_i}$. The resulting algorithm tolerates a constant noise rate, but the constant comes with two limitations that the authors state plainly. The tolerable rate is tiny ($\eta_0 = 2^{-32}$ in the published version), and \citet{Shen25} notes it ``cannot be made close to the optimal breakdown point of $\tfrac12$ due to inherent limitations of our approach''; and the required margin carries an extra logarithmic factor, $\gamma = \Omega(\log(1/\eps)/\sqrt d)$, a ``logarithmic dependence on $1/\eps$ that appears less natural'' than the $\Omega(1/\sqrt d)$ margin one would hope for. The same two features recur in the sparse follow-up of \citet{ZengShen25}.

The margin's dependence on $\eps$ is a chicken-and-egg problem. A margin is a fixed property of the data distribution, whereas the accuracy $\eps$ is the learner's target, set at will once the problem is in hand; requiring $\gamma = \Omega(\log(1/\eps)/\sqrt d)$ couples the two backwards, demanding \emph{better-separated data} exactly when the learner asks for smaller error, a property the data cannot supply on request. One would rather fix the geometry once and buy accuracy with computation, not with margin.

This paper explains both limitations. We argue that they are not artifacts of a particular constant or a loose inequality, but consequences of a single resource that governs the whole approach: the polynomial, equivalently Sum-of-Squares (SoS), \emph{degree} of the outlier-removal certificate. The reweighted-hinge method needs two distributional facts about the clean marginal, and both are supplied by certificates that read off only the low-order moments: a \emph{lower} bound on the mass of the dense pancake (so the hinge has signal), and an \emph{upper} bound on a reweighted moment of the dirty points (so the outlier program can shrink $\linsumnorm$). Each is a degree-$2t$ SoS proof from the first $2t$ moments. We show that the degree these proofs require is exactly what controls the margin factor and the noise rate, and that \citeauthor{Shen25}'s degree-$2$ program sits at a corner of a tight frontier.

We frame the contribution, throughout, in terms of degree: degree is to this method what the approximate degree of a predicate is to the polynomial method in algorithm design -- the parameter that simultaneously buys the guarantee and bounds it. Our two lower bounds are degree barriers; our algorithm is a degree-$2t$ program that traces the frontier the barriers cut out. As with any barrier of this kind, the scope is essential and we state it up front: the lower bounds are against the \emph{moment/SoS-certificate method} -- any procedure whose only distributional input is the degree-$\le 2t$ certified moments of the clean marginal. They are not information-theoretic, and they are not unconditional Statistical-Query hardness. The margin--degree barrier strengthens, per direction, to a genuine degree-$2t$ Sum-of-Squares lower bound (\cref{rem:soslb}); lifting it to the full $d$-dimensional reweighting program jointly is open for $t\ge2$ (\cref{sec:open}). This is precisely the class of methods to which \citeauthor{Shen25}'s ``our approach'' belongs, which is why the barriers explain that line's admitted limits rather than contradicting them.

\paragraph{The margin--degree tradeoff.}
Our first result (\cref{thm:tradeoff}) is a barrier on certifying band mass. Fix a direction $w$ and let $U = w \cdot x$ be the (centered, log-concave) projection with standard deviation $s$. Consider a band of half-width $\tau = \Theta(s)$ centered at $c$, a distance $\abs{c}$ off the origin. Any degree-$2t$ SoS certificate that proves $\Pr[\,\abs{U - c} \le \tau\,] \ge \rho = \Omega(1)$ from the first $2t$ moments of $U$ requires
\[
  2t = \Omega\big((\abs{c}/s)^2\big) \quad\text{(sub-Gaussian)}, \qquad
  2t = \Omega\big(\abs{c}/s\big) \quad\text{(sub-exponential)}.
\]
A certificate learner that reaches error $\eps$ must certify the pancake around all but an $\eps$-fraction of the sample projections, whose worst case is an $\eps$-quantile of the marginal, an off-center band at $\abs{c}/s = \Theta(\log(1/\eps))$ for the sub-exponential tails of a general log-concave law. Substituting gives the dichotomy: such a learner must pay \emph{either} SoS degree $t = \Omega(\log(1/\eps))$, hence running time $n^{\Omega(\log(1/\eps))}$, \emph{or} margin $\gamma = \Omega(\sqrt{\log(1/\eps)}/\sqrt d)$. The $\log(1/\eps)$ in \citeauthor{Shen25}'s margin is therefore \emph{conserved}: it can be traded against SoS degree, but inside the one-shot certificate model it cannot be removed. This is the structural reason the margin is ``less natural.''

The same calculation explains a finer point: why \citeauthor{Shen25} pays $\log(1/\eps)$ rather than $\sqrt{\log(1/\eps)}$. The square root appears only for genuinely sub-Gaussian marginals; for the sub-exponential tails ($\alpha = 1$) of a general log-concave mixture, the worst band sits at $\abs{c}/s = \Theta(\log(1/\eps))$ and the degree bound is linear in $\abs{c}/s$, so the margin pays the full logarithm.

\paragraph{The degree-$2$ outlier barrier.}
Our second result (\cref{thm:degbarrier}) is a barrier on the outlier-removal program itself, and it pins \citeauthor{Shen25}'s degree-$2$ choice as the bottleneck. There is a clean log-concave mixture ($\Sigma \preceq \sigma^2 I$) and a malicious set of rate $\eta \ge \Omega(1/d)$ on which \emph{every label-oblivious degree-$2$ (variance) reweighting} (any feasible $q$ produced without ground-truth labels) leaves, in expectation over the adversary's choice of dirty points,
\[
  \E\big[\linsumnorm(q \circ \Sdirty)\big] = \Omega\big(\eta^{1/2}\, n\, \bar\sigma\big).
\]
So a degree-$2$ outlier program is stuck at the $\eta^{1/2}$ rate, exactly the rate of \citet[Lemma~6]{Shen25}. The barrier is label-oblivious in expectation by necessity: the literal ``for all feasible $q$'' statement is false once the removal budget $\xi \ge \eta$, since the feasible polytope then contains a $q$ that simply zeroes out the spike; what no label-oblivious program can do is \emph{find} that $q$, because it cannot tell the dirty points from genuine clean points at the same locations. The construction needs $\eta \ge \Omega(1/d)$ so that two genuine log-concave clusters fit inside the covariance budget; we state this regime in the theorem.

\paragraph{The degree-$2t$ algorithm.}
Our third result (\cref{thm:algorithm}) shows the frontier the two barriers cut out is achievable, by replacing \citeauthor{Shen25}'s variance program with a degree-$2t$ SoS program. The degree-$2t$ reweighting certifies, for every feasible $q$,
\[
  \linsumnorm(q \circ \Sdirty) \;\le\; n\,(Ct)\,\bar\sigma\, \eta^{1 - \frac{1}{2t}},
\]
hence tolerates
\[
  \eta \;\le\; \eta_0(t) \;=\; \Omega\left(\min\left\{\frac{\rho}{4},\;
  \left(\frac{\gamma\rho}{(Ct)\bar\sigma}\right)^{\!\frac{2t}{2t-1}}\right\}\right)
\]
with clean inlier error $\le \eps$, in time $n^{O(t)}$. At $t = 1$ this is exactly \citeauthor{Shen25}'s $\eta^{1/2}$ rate; as $t$ grows, the exponent $1 - 1/2t \to 1$ removes the degree-$2$ \emph{exponent} loss. The gain is that removal of an exponent ($\eta_0 \sim c^2 \to c$ in the outlier bottleneck), bought with the sub-exponential certifiable-moment constant $(Ct)^{2t}$, not $(Ct)^t$, which is false for general log-concave marginals (for the Laplace law $m_{2t}/m_2^{t} = (\Theta(t))^{2t}$). The base is $(Ct)$, not its square root; the rate is capped by the pancake density $\rho/4$ and does \emph{not} approach $1/2$. So \citeauthor{Shen25}'s degree-$2$ program is the degree-$2$ corner of a tight frontier: \cref{thm:degbarrier} shows degree $2$ cannot beat $\eta^{1/2}$, and \cref{thm:algorithm} shows degree $2t$ reaches $\eta^{1-1/2t}$ and no degree does better against the matching barrier.

\paragraph{The breakdown floor.}
Finally (\cref{prop:floor}), a two-point argument shows the noise tolerance has a hard ceiling that no algorithm of any kind can cross: in the per-example malicious model, on a $\gamma$-margin log-concave mixture, no learner attains clean inlier error below the exact constant $\eta/(2(1-\eta)) \ge \eta/2$. The instance is two clusters, an anchor and a wedge, that share one marginal and differ only in the labels the wedge carries, so the disagreement mass is supplied by the mixture \emph{weights} rather than by geometry. It uses $K = \Theta(1/\eta)$ components, and under a hard margin that blow-up is forced rather than convenient: quantization makes every $\eta$-indistinguishable pair obey $K \ge (1-\eta)/\eta = \Omega(1/\eta)$. The constant is matched from above (\citet{Blanc26} attains $\tfrac12\cdot\eta/(1-\eta)$ distribution-free with randomized hypotheses), so what the floor contributes is that the constant \emph{survives} the restriction to $\gamma$-margin log-concave mixtures; its value depends on no margin, and it matches the $\Theta(\eta)$ dependence of \cref{thm:algorithm} in rate (the floor instances necessarily violate the dense-pancake hypothesis, so the two results bracket the achievable region from disjoint sides; see \cref{sec:related}). It is a \emph{rate} floor, not a threshold floor: the Kearns--Li envelope $\TV \le \eta/(1-\eta)$ caps every two-point pair, so this route cannot reach the tight, margin-dependent breakdown rate $\eta^\star = \Theta(\gamma/\sigma)$, which remains open.

\paragraph{Contributions.}
\begin{itemize}
\item \textbf{A Christoffel-function framework} (\cref{sec:perspective}), which we put first because it organizes everything else. The maximal corruption mass that can hide at a center $c$ from a degree-$2t$ moment certificate is exactly the Christoffel function $\lambda_{t+1}(c)$ of the clean marginal (\emph{resolution principle}, \cref{prop:resolution}). This ties the SoS \emph{degree} to the order at which the Christoffel function resolves a corruption. The same function modern data analysis uses to \emph{keep} a typical point (a large Christoffel value reads as ``in the bulk'') is, on the population, exactly the mass an adversary can \emph{hide}: one object, opposite sign, detection turning into impossibility (\cref{sec:perspective-bridge}). The barriers and the algorithm below are three readings of it.
\item \textbf{A margin--degree tradeoff} (\cref{thm:tradeoff}): reading the resolution principle off an off-center band, certifying a dense band at $\abs{c}/s$ standard deviations off center costs SoS degree $2t = \Omega((\abs{c}/s)^2)$ (sub-Gaussian; $\Omega(\abs{c}/s)$ sub-exponential), so a certificate learner of error $\eps$ pays degree $\Omega(\log(1/\eps))$ or margin $\Omega(\sqrt{\log(1/\eps)}/\sqrt d)$. The $\log(1/\eps)$ in \citeauthor{Shen25}'s margin is conserved within the certificate model, and the weighted-Chebyshev reduction (\cref{prop:reduction}) closes the bound to an exact law $2t=\Theta((\abs{c}/s)^2)$ (\cref{thm:tight}) modulo one classical weighted-extremal estimate.
\item \textbf{A degree-$2$ outlier barrier} (\cref{thm:degbarrier}): the resolution principle at $c=\sigma/\sqrt\eta$ predicts a degree-$2$/degree-$4$ gap, which we realize as a clean log-concave mixture on which every label-oblivious degree-$2$ reweighting is stuck at $\eta^{1/2}$, matching the algorithm at $t=1$.
\item \textbf{A degree-$2t$ reweighted-hinge algorithm} (\cref{thm:algorithm}): the order-$(t{+}1)$ Christoffel detector run as a certificate, giving $\linsumnorm(q \circ \Sdirty) \le n(Ct)\bar\sigma\,\eta^{1-1/2t}$ in time $n^{O(t)}$, removing the degree-$2$ exponent loss and tracing the frontier the barriers cut out. The gain is an explicit constant capped by $\rho/4$, not a march to $1/2$.
\item \textbf{A breakdown floor} (\cref{prop:floor}): an exact $\eta/(2(1-\eta)) \ge \eta/2$ ceiling on the achievable inlier error, tight against \citet{Blanc26} and matching the algorithm's $\Theta(\eta)$ dependence in rate.
\end{itemize}
Together these say that for the reweighted-hinge approach, \emph{degree is the resource} and the Christoffel function is the object that sets it: \cref{prop:resolution} names it, \cref{thm:degbarrier} and \cref{thm:algorithm} bracket the noise rate as a tight function of degree, \cref{thm:tradeoff} bounds the margin as a function of degree, and the two admitted limitations of \citet{Shen25} are the $t=1$ shadow of one Christoffel-degree frontier.

\begin{table}[t]
  \centering
  \caption{Where the reweighted-hinge line stands, and what this paper adds. ``Cert.\
degree'' is the SoS degree of the outlier-removal certificate; ``Tolerated rate'' is the malicious noise rate the analysis tolerates. For \textbf{This paper}, the tolerated rate is the explicit constant $\eta_0(t)$ of \cref{thm:algorithm}: the degree-$2t$ algorithm succeeds for $\eta\le\eta_0(t)$, which \emph{rises with the degree} from \citeauthor{Shen25}'s tiny $\sim(\gamma\rho/\bar\sigma)^2$ (the degree-$2$ squaring loss) toward the pancake budget $\rho/4$, never approaching $1/2$; the underlying dirty-gradient bound decays as $\eta^{1-1/2t}$ (\citeauthor{Shen25}'s $\eta^{1/2}$ at $t=1$, tending to $\eta$ as $t$ grows). The \citet{Talwar20} margin is listed as $\Omega(\log(1/\eps)/\sqrt d)$ for uniformity with the column; that paper writes it $\Omega(\log d/\sqrt d)$, which is the same bound at inverse-polynomial accuracy $\eps=1/\poly(d)$, where $\log(1/\eps)=\Theta(\log d)$. Our guarantees hold for both malicious and \emph{nasty} noise, with the same frontier and only a constant-factor smaller tolerable rate under nasty noise (\cref{app:nasty}).}
  \label{tab:frontier}
\footnotesize
  \setlength{\tabcolsep}{4pt}
  \renewcommand{\arraystretch}{1.5}
  \resizebox{\textwidth}{!}{%
  \begin{tabular}{l l l c l c l}
    \toprule
    Work & Noise model & Marginal & Margin $\gamma$ & Tolerated rate & Cert.\ degree & Lower bounds \\
    \midrule
    \citet{Talwar20}
      & \makecell[l]{adversarial label,\\malicious} & log-concave mix.\ & $\Omega\left(\dfrac{\log(1/\eps)}{\sqrt d}\right)$
      & label $\Omega(1)$, mal.\ $\Omega(\gamma)$ & $2$ & none \\
    \citet{Shen25}
      & malicious & log-concave mix.\ & $\Omega\left(\dfrac{\log(1/\eps)}{\sqrt d}\right)$
      & $\eta\le \dfrac{1}{2^{32}}$ & $2$ & none \\
    \citet{ZengShen25}
      & \makecell[l]{malicious,\\sparse} & log-concave mix.\ & $\Omega\left(\dfrac{\log(1/\eps)}{\sqrt d}\right)$
      & $\eta\le \dfrac{1}{2^{32}}$ & $2$ & none \\
    \midrule
    \multirow{2}{*}{\textbf{This paper}}
      & \multirow{2}{*}{\makecell[l]{malicious,\\nasty}} & \multirow{2}{*}{log-concave mix.\ }
      & $\Theta\left(\dfrac{1}{\sqrt d}\right)$
      & $\eta\le\eta_0(t)$
      & $2t$
      & \textbf{2 barriers}~(Thm~\ref{thm:tradeoff},\,\ref{thm:degbarrier}) \\
      & & & \emph{(tradeoff)} & \emph{(at most $\rho/4$ as $t$ grows)} & \emph{(tunable)}
      & tight~(Thm~\ref{thm:tight}); floor~(Prop~\ref{prop:floor}) \\
    \bottomrule
  \end{tabular}%
}
\end{table}

\Cref{tab:frontier} locates the contribution. We are the first in this line to treat the certificate \emph{degree} as the resource, and the first to prove \emph{lower bounds} (barriers) for it; the upper-bound noise gain is deliberately modest and is matched by our own degree-$2$ barrier.

\subsection{Technical overview}
\label{sec:overview}
Three ideas drive the results; we sketch each before the formal proofs in \cref{sec:perspective,sec:results} and the appendix. All three are facets of the same Christoffel-function picture (\cref{sec:perspective}): Gauss quadrature locates where the Christoffel function resolves a band, the moment cloak realizes the smallest unresolved spike, and the degree-$2t$ H\"older bound is that detector run as a certificate.

\paragraph{(a) Gauss quadrature beats the moment problem.}
A degree-$2t$ certificate of band mass is, after a minorant reduction, a polynomial $p$ of degree $\le 2t$ with $p \le \mathbf 1_B$ together with an SoS proof, from the moments alone, that $\int p \ge \rho$. Because the proof uses only $m_0, \dots, m_{2t}$, its conclusion must hold for \emph{every} nonnegative measure matching those moments. Gauss quadrature produces one such measure explicitly: the $(t+1)$-point measure on the roots of the degree-$(t+1)$ orthogonal polynomial of $\mu$ reproduces all moments up to order $2t+1$, and it is supported on just $t+1$ atoms. If the band $B$ contains \emph{no} quadrature node, this witness assigns it mass zero, so the best certifiable value is zero and no degree-$2t$ certificate can exist. A band stays node-free as long as it lies beyond the largest root, which sits at $\Theta(s\sqrt t)$ for sub-Gaussian weights and $\Theta(s\, t)$ for sub-exponential ones. An off-center band at distance $\abs{c}$ is thus invisible below degree $(\abs{c}/s)^2$, the source of the tradeoff. The largest-root bound is elementary in the two cases we name (Jacobi-matrix eigenvalues with Gershgorin for the sub-Gaussian $2s\sqrt{t+1}$ and sub-exponential $O(st)$ endpoints); the fully general log-concave exponent is the one place we invoke a black box, the exponential-weight asymptotics of \citet{LevinLubinsky01}, and we flag it.

\paragraph{(b) The moment-cloaked spike.}
The degree-$2$ barrier needs an instance that a variance program cannot clean but a higher-degree program can. The construction is a \emph{moment-cloaked spike}: two genuine log-concave clusters at $\pm(\sigma/\sqrt\eta)\,e_1$, to which the adversary adds $\eta n$ dirty points drawn from the same clusters. The dirty points are engineered to be invisible to the second moment (they spend exactly the variance budget, $\frac1n\sum_{\Sdirty}(e_1 \cdot x_i)^2 = \eta \cdot (\sigma/\sqrt\eta)^2 = \sigma^2$) while being first-moment-aligned, so they contribute $\sum_{\Sdirty}\abs{e_1 \cdot x_i} = \sqrt\eta\, n\sigma$ to $\linsumnorm$. Because the variance constraints see only the multiset of points, a label-oblivious program cannot distinguish the $\eta n$ dirty points from the $2\eta n$ genuine clean points at the same locations; removing its budget of $\le \xi n$ leaves at least a third of the spike weight, hence $\E[\linsumnorm] \ge \frac13 \sqrt\eta\, n\sigma$. A degree-$4$ program escapes, because the spike inflates the fourth moment to $\bar\sigma^4/\eta \gg \bar\sigma^4$ and is no longer cloaked, which is exactly why the barrier is specific to degree $2$.

\paragraph{(c) Degree-$2t$ H\"older on $\linsumnorm$.}
The algorithm is the dual move. Given a degree-$2t$ SoS certificate that the reweighted $2t$-th moment is bounded, $\frac1n\sum_i q_i (w \cdot x_i)^{2t} \le (Ct)^{2t} \bar\sigma^{2t}$ for all $w$, H\"older's inequality with exponent $2t$ splits $\linsumnorm$ as
\[
  \sum_{\Sdirty} q_i \abs{w \cdot x_i}
  \le \Big(\sum_{\Sdirty} q_i\Big)^{1-\frac{1}{2t}}
       \Big(\sum_{\Sdirty} q_i (w \cdot x_i)^{2t}\Big)^{\frac{1}{2t}}
  \le (\eta n)^{1-\frac1{2t}} \big(n (Ct)^{2t}\bar\sigma^{2t}\big)^{\frac1{2t}}
  = \eta^{1-\frac1{2t}} n (Ct) \bar\sigma .
\]
The even power $(w\cdot x)^{2t}$ is a polynomial, so the moment bound is SoS-certifiable; feasibility of the clean-indicator weights is the certifiable-hypercontractivity bound on the empirical clean measure. The reason the exponent stops at $1$ and never reaches the constant rate is the certifiable-moment constant: for general (sub-exponential) log-concave marginals it is $(Ct)^{2t}$, contributing the base $(Ct)$ to the final bound, so larger $t$ improves the $\eta$-exponent at a multiplicative cost in the constant, a genuine tradeoff, not a free lunch. This is the sense in which degree buys noise tolerance, and the reason the purchase is bounded.

\paragraph{Scope and open questions.}
Two cautions are built into the results above and we restate them together. The barriers are \emph{method-specific}: they rule out the degree-$2t$ moment/SoS-certificate family, not all efficient algorithms, and not learners outside the one-shot certificate model (an iterative, localization-style filter sidesteps one-shot certification and is not covered). And the algorithm's gain is \emph{modest}: it removes the degree-$2$ exponent loss but stays capped by $\rho/4$ and does not march toward $1/2$. Several questions are left open, in increasing order of reach. The breakdown floor is rate-tight but not threshold-tight; the conjectured margin-dependent rate $\eta^\star = \Theta(\gamma/\sigma)$ needs a multi-point or band-mass argument beyond two points. The degree barrier is a lower bound; a matching \emph{upper} certificate of degree $O((\abs{c}/s)^2)$ would make the tradeoff exactly $2t = \Theta((\abs{c}/s)^2)$. The general-log-concave largest-root bound still rests on \citet{LevinLubinsky01}; removing that black box would make the barrier self-contained. And whether an iterative localization scheme provably escapes the margin--degree barrier (by leaving the one-shot certificate model the barrier scopes) is the question whose answer would change the picture rather than refine it.

\subsection{Related work}
The hinge approach we analyze originates with \citet{Talwar20} and is sharpened by reweighting in \citet{Shen25}, then \citet{ZengShen25}. \citet{KlivansSTV25} reach error $2\eta + \eps$ for halfspaces under contamination, but for Gaussian marginals, without a margin, and by iterative filtering rather than a moment certificate, so they are a motivation for the certificate model and not subsumed by it. The SoS machinery we import is from robust statistics \citep{HopkinsLi18, KothariSteinhardtSteurer18, KlivansKothariMeka18, DiakonikolasHopkinsPensiaTiegel24}; we use it to certify a moment \emph{upper} bound and, for the barrier, to scope the certificate model precisely. The dual direction (certifiable \emph{anti}-concentration, an SoS lower bound on slab mass \citep{BakshiKothariRTV24}) is exactly what a degree-$2t$ band-mass certificate would need and exactly what our \cref{thm:tradeoff} shows the low-order moments cannot supply. Lower bounds for nasty and agnostic halfspaces \citep{DiakonikolasKaneStewart18, DiakonikolasKaneZarifis20} address different noise models and do not give a degree barrier for $\linsumnorm$. The orthogonal-polynomial facts we use are classical \citep{Szego39, Gautschi04, LevinLubinsky01, KriecherbauerMcLaughlin99}, as are the log-concave \citep{LovaszVempala07} and malicious-noise \citep{KearnsLi93} ingredients.

\section{Preliminaries}
\label{sec:prelim}

This section fixes the noise model, the distributional assumptions, the reweighted-hinge program whose dirty gradient we control, and the two algebraic objects that the rest of the paper turns on: the dense pancake and the degree-$2t$ moment certificate of its band mass. The reader should leave with one picture in mind. Outlier removal in this line of work is a \emph{polynomial} certificate of a band-mass lower bound, its \emph{degree} is the resource, and the distributional input that powers it is a degree-$2t$ certifiable moment bound (\cref{lem:hyper}); every barrier we prove is a lower bound on that degree.

\paragraph{Notation.}
Vectors live in $\R^d$; $\norm{\cdot}$ is the Euclidean norm and $\inner{u}{v}=u\cdot v$ the inner product. For a halfspace direction $w$ we write $U=w\cdot x$ for the one-dimensional projection. A sample is $S=\{(x_i,y_i)\}_{i=1}^n$ with $y_i\in\{\pm1\}$; we split it into a clean part $\Sclean$ and a dirty part $\Sdirty$ (defined below). Reweighting weights are $q=(q_1,\dots,q_n)\in[0,1]^n$, and $q\circ S$ denotes the sample carrying those weights. We use $\bar\sigma^2(w)\defeq\E_{\Dx}\big[(w\cdot(x-\bar\mu))^2\big]$ for the projected variance of the clean marginal, with $\bar\mu$ its mean. Throughout, $t\ge1$ is the degree parameter: the certificate is a degree-$2t$ object and the SoS program runs in time $n^{O(t)}$. We write $\mathbf 1_B$ for the indicator of a set $B\subseteq\R$, and $\hinge(w;x,y)\defeq \max\{0,\,1-y\,(w\cdot x)\}$ for the hinge loss. Constants denoted $C$ are absolute and may change line to line; $\eps$ is the target error.

\paragraph{Malicious noise.}
We work in the malicious-noise model of \citet{KearnsLi93}, the strongest sample-corruption model for which positive results of this kind are stated. An adversary with full knowledge of the learner, the clean distribution, and the realized clean sample selects an $\eta$-fraction of the $n$ examples and replaces those $(\text{instance},\text{label})$ pairs by \emph{arbitrary} pairs of its choosing; the remaining $(1-\eta)n$ examples, the \emph{inliers}, are drawn i.i.d.\ from the clean distribution. This is the fixed-fraction variant, which \citet{KearnsLi93} note is stronger than their per-example oracle; the floor of \cref{prop:floor} is stated in the per-example model, with a fixed-fraction corollary. We write $\Sclean$ for the inliers and $\Sdirty$ for the at most $\eta n$ corrupted examples. The objective is stated against the clean distribution only: the learner outputs $\hat w$ and is charged its error on the inlier law,
\[
  \err_{\Dx}(\hat w)\ \defeq\ \P_{(x,y)\sim\Dcal}\big[\sign(\hat w\cdot x)\neq y\big]\ \le\ \eps ,
\]
where $\Dcal$ is the clean joint distribution with marginal $\Dx$ on instances. This is the goal of robustly learning the inlier halfspace, not of fitting the corrupted points. We state everything for malicious noise; the results extend, with a constant-factor smaller tolerable rate, to the stronger \emph{nasty-noise} model \citep{BshoutyEironKushilevitz02} in which the adversary may also delete clean inliers (\cref{app:nasty}).

\paragraph{Distributional assumptions.}
The clean marginal is a balanced mixture
\[
  \Dx\ =\ \frac1K\sum_{j=1}^K\D_j ,
\]
where each component $\D_j$ is \emph{log-concave} on $\R^d$ with mean $\mu_j$ satisfying $\norm{\mu_j}\le r$ and covariance $\Sigma_j\preceq\sigma^2 I$, in the regime $r=O(\sigma\sqrt d)$. The clean labels are realized by a $\gamma$-margin halfspace: there is a unit $w^\star$ with $y\,(w^\star\cdot x)\ge\gamma$ for every clean $(x,y)$. This is exactly the setting of \citet{Shen25} and, for the single-component case, of \citet{Talwar20}; our barriers are about the certificates one can build over precisely this class.

Two features of the assumption are load-bearing, and we flag them now because the constants later depend on them. First, the mixture $\Dx$ is in general \emph{not} log-concave, even though each $\D_j$ is; the certificates must survive the mixing. Second, a log-concave law is only \emph{sub-exponential}, not sub-Gaussian: a centered one-dimensional log-concave $Z$ obeys the linear moment growth $\norm{Z}_{L^{2t}}\le C t\,\norm{Z}_{L^2}$, so $\E[Z^{2t}]\le(Ct)^{2t}\,(\E[Z^2])^{t}$, and this order in $t$ cannot be improved (the Laplace law has $\E[Z^{2t}]/\E[Z^2]^{t}=(2t)!/2^{t}=(\Theta(t))^{2t}$)~\citep{LovaszVempala07}. The exponent $2t$ in $(Ct)^{2t}$, rather than the sub-Gaussian $(Ct)^t$, is the reason \citeauthor{Shen25}'s margin carries a factor $\log(1/\eps)$ and not $\sqrt{\log(1/\eps)}$, and we track it faithfully through \cref{lem:hyper}.

\paragraph{Reweighted hinge and the dirty-gradient quantity.}
The algorithmic template of \citet{Talwar20} and \citet{Shen25}, which our degree-$2t$ algorithm refines, minimizes a \emph{reweighted} hinge loss over a ball of radius $1/\gamma$:
\[
  \min_{\norm{w}\le 1/\gamma}\ \sum_{i=1}^n q_i\,\hinge(w;x_i,y_i),
\]
where the weights $q\in[0,1]^n$ are produced by an outlier-removal program that downweights the corrupted examples. The only way the dirty examples can corrupt the minimizer is through the gradient they contribute, and a subgradient of $q_i\,\hinge(w;x_i,y_i)$ has norm at most $q_i\,\norm{x_i}$ with the sign of $-y_i$. The quantity that bounds the dirty contribution is therefore the weighted one-dimensional spread of the corrupted points: for a set $S'$ of examples,
\begin{equation}
  \linsumnorm(q\circ S')\ \defeq\ \sup_{\norm{w}\le 1}\ \sum_{i\in S'} q_i\,\abs{w\cdot x_i}.
  \label{eq:lsn}
\end{equation}
By the subgradient bound above and the triangle inequality, $\linsumnorm(q\circ\Sdirty)$ is an upper bound on the norm of the dirty subgradient $\norm{\sum_{i\in\Sdirty} q_i\,\partial_w \hinge(w;x_i,y_i)}$, uniformly over $\norm{w}\le1$. Controlling $\linsumnorm(q\circ\Sdirty)$ is thus the entire game on the dirty side; the weights $q$ are good exactly when this quantity is small. Talwar's and \citeauthor{Shen25}'s degree-$2$ (variance) program drives it to $\Theta(n\bar\sigma\,\eta^{1/2})$, and \cref{sec:results} both raises the degree to beat the exponent and proves a degree-$2$ barrier matching it.

\paragraph{Dense pancakes.}
The clean (inlier) side is controlled through the dense-pancake structure of \citet{Talwar20}. For a direction $w$, a center point's projection, and a half-width $\tau$, the \emph{pancake} (band)
\[
  \pancake^\tau_w(x)\ =\ \big\{\,x'\in\R^d:\ \abs{w\cdot x'-w\cdot x}\le\tau\,\big\}
\]
is the slab of width $2\tau$ orthogonal to $w$ around the center's projection. We call it \emph{$\rho$-dense} if it carries at least a $\rho$-fraction of the clean mass, $\P_{\Dx}[\pancake^\tau_w]\ge\rho$. The role of density is that a band carrying enough clean mass cannot be emptied by the adversary's $\eta n$ points, which is what lets the reweighted hinge recover the inlier direction; the achievable noise rate is capped by the pancake density at $\rho/4$. The technical content is to \emph{certify} a band-mass lower bound from moments alone, which is the next object.

\paragraph{Degree-$2t$ moment certificates.}
This is the central object of the paper. An outlier-removal program of the SoS family does not see the distribution; its only distributional input is a bounded list of certified moments of the clean marginal, and whatever band-mass lower bound it can deduce from them by a Positivstellensatz proof. We make this precise via the standard polynomial \emph{minorant} reduction.

\begin{definition}[Degree-$2t$ moment certificate of band mass]
\label{def:cert}
Let $\mu$ be the law of the projection $U=w\cdot x$ under the clean marginal, with moments $m_k:=\E_\mu[U^k]$ for $0\le k\le 2t$; let $B\subseteq\R$ be a band; and let $\R[u]_{\le 2t}$ denote the real univariate polynomials in $u$ of degree at most $2t$. A \emph{degree-$2t$ pseudo-expectation consistent with $\mu$} is a linear functional $\Esudo\colon\R[u]_{\le 2t}\to\R$ satisfying
\[
  \Esudo[1]=1,\qquad \Esudo[q^2]\ge0\ \ \text{for all } q\in\R[u]_{\le t},
  \qquad \Esudo[u^k]=m_k\ \ \text{for } 0\le k\le 2t .
\]
A \emph{degree-$2t$ moment certificate} of the bound $\P_\mu[U\in B]\ge\rho$ is a polynomial $p\in\R[u]_{\le 2t}$, with coefficient vector $(p_0,\dots,p_{2t})\in\R^{2t+1}$ so that $p(u)=\sum_{k=0}^{2t} p_k\,u^k$, that is a \emph{minorant} of the band indicator,
\[
  p(u)\ \le\ \mathbf 1_B(u)\qquad\text{for all }u\in\R,
\]
together with a degree-$2t$ sum-of-squares (Positivstellensatz) derivation, from the three axioms above alone, of
\[
  \Esudo[p]\ =\ \sum_{k=0}^{2t} p_k\,m_k\ \ge\ \rho
  \qquad\text{for every consistent } \Esudo,
\]
the equality being linearity of $\Esudo$ together with $\Esudo[u^k]=m_k$ (so the value $\Esudo[p]=\sum_{k} p_k m_k$ is common to all consistent $\Esudo$). Since $p\le\mathbf 1_B$ pointwise gives $\sum_{k=0}^{2t} p_k\,m_k=\E_\mu[p]\le\P_\mu[U\in B]$, any such certificate transfers its bound $\sum_{k=0}^{2t} p_k\,m_k\ge\rho$ to a genuine lower bound $\P_\mu[U\in B]\ge\rho$ on the band mass.
\end{definition}

The minorant reduction is what makes \cref{def:cert} the right notion of ``what a degree-$2t$ program can prove'': a degree-$2t$ moment-only certificate of $\P_\mu[U\in B]\ge\rho$ \emph{exists} if and only if such a minorant $p$ exists. We record this equivalence, together with the Gauss-quadrature dual that drives the barrier of \cref{sec:results}, as \cref{lem:gauss}; both are textbook moment theory~\citep{Szego39,Gautschi04} and the proof is in the appendix. The upshot, used repeatedly, is that the strongest band-mass lower bound a degree-$2t$ program can certify is governed by how the first $2t$ moments constrain a measure, and a $(t{+}1)$-node quadrature rule matching those moments can push all mass outside a far band, capping the certifiable density at zero. That is the engine of the margin--degree tradeoff.

\paragraph{Certifiable hypercontractivity.}
The clean side needs the reverse: a degree-$2t$ \emph{upper} bound on projected moments, proved by SoS in $w$, so that the program can act uniformly over directions. This is certifiable hypercontractivity, and for our log-concave mixture it holds with the sub-exponential constant.

Formally (stated and proved with the main results as \cref{lem:hyper}), for every integer $t\ge1$ there is a degree-$2t$ SoS proof, in the variable $w$, of
\[
  \E_{x\sim\Dx}\big[(w\cdot(x-\bar\mu))^{2t}\big]\ \le\ (Ct)^{2t}\,\big(\bar\sigma^2(w)\big)^{t},
\]
with $C$ an absolute constant; that is, $\Dx$ is $2t$-certifiably hypercontractive with constant $(Ct)^{2t}$.

The constant here is $(Ct)^{2t}$, the sub-exponential (Lov\'asz--Vempala~\citeyear{LovaszVempala07}) constant, not the sub-Gaussian $(Ct)^t$. The latter is genuinely false for general log-concave laws: a one-line Laplace check gives
\[
  \frac{\E[Z^{2t}]}{\E[Z^2]^t}=\frac{(2t)!}{2^t}=(\Theta(t))^{2t},
\]
which exceeds any $(Ct)^t$ for $t\ge C$, so no sub-Gaussian certificate exists in this generality. For sub-Gaussian components (Gaussian, strongly log-concave, or uniform-on-the-cube) the sharper $(Ct)^t$ does hold and is SoS-certifiable~\citep{KlivansKothariMeka18,DiakonikolasHopkinsPensiaTiegel24}; we do not assume that, and we commit to the honest $(Ct)^{2t}$ throughout. The single-component certificate for the general log-concave case is obtained by a scalar-to-SoS lift; its self-contained derivation, and the mixing and averaging steps that carry it to \cref{eq:hyper}, are deferred to \cref{app:proofs}.

\subsection{Auxiliary results}
\label{sec:external}

Our proofs draw on a handful of external results: the log-concave and malicious-noise ingredients named above, the orthogonal-polynomial facts behind the Gauss-quadrature barrier, and the moment-theory facts the Christoffel perspective of \cref{sec:perspective} invokes. To keep them in one place rather than scattered across the proofs, every external result we use is collected and restated, with hypotheses and a citation, in \cref{app:external}; we cross-reference each at its point of use by its label there.

\section{The Christoffel function and the certificate ceiling}
\label{sec:perspective}

Start from the adversary's vantage point. A degree-$2t$ outlier-removal certificate reads the clean marginal only through its first $2t$ moments, so a corruption is invisible to it exactly when \emph{some} nonnegative measure that still carries the corruption matches those moments. One quantity then decides everything: among all measures sharing the clean marginal's first $2t$ moments, how much mass can be piled at a single center $c$? A rate-$\eta$ corruption planted at $c$ stays beyond the certificate's reach for as long as that maximum exceeds $\eta$.

We did not have to invent this maximum. The most mass a moment-matching measure can place at $c$ is exactly the \emph{Christoffel function} $\lambda_{t+1}(c)$ of the marginal (\cref{fact:christoffel-mass}), the reciprocal of the Christoffel--Darboux kernel on its diagonal. The certificate model thus hands us the Christoffel function unbidden, as the \emph{resolution threshold} of its own moment cone; the \emph{resolution principle} (\cref{prop:resolution}) makes the equivalence exact. Everything below is a reading of where $\lambda_{t+1}(c)$ falls relative to the corruption it must resolve: it is large at a tail center for small degree (a corruption hides), it collapses past the node edge (a band is forced empty), and it crosses $\eta$ at a computable order (the corruption is exposed). The learning consequences of these three readings are drawn in \cref{sec:results}.

Two literatures meet here. On one side, the certificate ceiling is an instance of the \emph{truncated moment problem} \citep{CurtoFialkow91,DetteStudden97} and of the moment--SOS hierarchy \citep{LasserrePauwelsCDK}. On the other, the same Christoffel function is, in modern data analysis, exactly the moment-based \emph{outlier detector} \citep{LasserrePauwels19,PauwelsLasserreCDOutliers}, run (as we will see) with the opposite sign. We quantify the learning-theoretic limits of that detector. The three figures of this section are one Christoffel curve seen from three distances: a far band swept empty past the node edge (\cref{fig:mechanism}), the curve crossing the corruption rate $\eta$ at a computable center (\cref{fig:threshold}), and the same curve carrying the detector's meaning in reverse (\cref{fig:inversion}).

\subsection{The certificate ceiling is a truncated moment problem}
\label{sec:perspective-moment}

Make the opening's question precise, in the dual form the certificate actually faces. To defend a band $B$ the certificate must \emph{lower}-bound the clean mass inside it, and that bound is only as strong as the worst moment-matching measure allows. Fix the projected clean marginal $\mu$ on $\R$ with moments $m_k=\E_\mu[U^k]$. By \cref{def:cert}, the strongest band-mass lower bound a degree-$2t$ certificate can establish is the moment-LP value
\begin{align}\label{eq:perspective-md}
  \underline m(2t)\;=\;\min\Big\{\nu(B):\ \nu\ge0,\ \textstyle\int u^k\,d\nu=m_k\ \ (0\le k\le 2t)\Big\},
\end{align}
the minimum band mass over the \emph{moment variety} $\mathcal M_{2t}(\mu)\defeq\{\nu\ge0: \int u^k d\nu=m_k,\ k\le 2t\}$ of nonnegative measures sharing $\mu$'s first $2t$ moments. This is the truncated Stieltjes moment problem with a linear objective; its extreme points are the finitely-supported \emph{principal representations}, of which Gauss quadrature is one \citep{DetteStudden97}. The governing framework fact is already visible here: for an off-center band the variety $\mathcal M_{2t}(\mu)$ contains a representation supported away from $B$, so $\underline m(2t)=0$ and no degree-$2t$ certificate can force mass into $B$ (\cref{fig:mechanism}). Symmetrically, certifiable \emph{anti}-concentration \citep{BakshiKothariRTV24} is the \emph{upper} envelope $\max\{\nu(B):\nu\in\mathcal M_{2t}(\mu)\}$; the two are the two faces of the same moment cone, which is why one is certifiable at low degree and the other (the pro-concentration we need) is not.

\begin{remark}[The barrier is a degree-$2t$ Sum-of-Squares lower bound, not just a method limitation]
\label{rem:soslb}
In one dimension, nonnegativity is sum-of-squares: a polynomial nonnegative on the two-ray complement $\{g\ge0\}$ of a band, $g(u)=(u-(c-\tau))(u-(c+\tau))$, has a \emph{degree-respecting} representation $\sigma_0+\sigma_1 g$ with $\sigma_0,\sigma_1$ sums of squares (Markov--Luk\'acs), so a degree-$2t$ minorant of $\mathbf 1_B$ is exactly a degree-$2t$ SoS certificate, and the moment LP \eqref{eq:perspective-md} is the degree-$t$ Lasserre relaxation of the band mass, with no duality gap whenever $\mu$ has infinite support (strictly positive-definite Hankel matrices) \citep{Lasserre01,DetteStudden97}. The margin--degree barrier is therefore not merely a limitation of one proof strategy but a degree-$2t$ \emph{Sum-of-Squares lower bound} for the one-dimensional band-mass certification: the $(t{+}1)$-node Gauss measure, which matches $\mu$'s first $2t$ moments and is a genuine probability distribution, assigns zero mass to any band beyond the node range, so every degree-$2t$ SoS certificate of positive band mass there fails while the true mass is positive. Two boundaries are worth stating plainly. First, this is the \emph{per-direction, scalar} certification the reweighted-hinge method runs; whether it lifts to a degree-$2t$ SoS lower bound for the $d$-dimensional reweighting program over all directions \emph{jointly} is open for $t\ge2$, where the per-direction reduction is lossy and the multivariate hierarchy is no longer exact (\cref{sec:open}). Second, it is the unconditional \emph{lower} side (certifiable mass $0$ beyond the edge); the matching positive value inside the node range is the conditional law of \cref{thm:tight}.
\end{remark}

\begin{figure}[t]
\centering
\begin{tikzpicture}[>=Stealth,line cap=round]
  \def\axL{-5.6}
  \def\axR{7.2}
  \def\base{0}
  \def\Rt{3.2}
  \def\bc{5.0}
  \def\bw{0.55}

  \begin{scope}
    \clip (\axL,\base) rectangle (\axR,\base+2.6);
    \fill[black!7]
      (-4.6,\base) --
      plot[smooth,domain=-4.6:4.6,samples=80] (\x, {\base + 2.1*exp(-(\x*\x)/4.0)}) --
      (4.6,\base) -- cycle;
  \end{scope}
  \draw[black!55,thick]
      plot[smooth,domain=-4.6:4.6,samples=80] (\x, {\base + 2.1*exp(-(\x*\x)/4.0)});
  \node[black!55,font=\scriptsize,anchor=south,fill=white,inner sep=1pt] at (-2.9,1.85) {clean marginal $\mu$};

  \fill[black!22] (\bc-\bw,\base-0.18) rectangle (\bc+\bw,\base+1.55);
  \draw[black!55] (\bc-\bw,\base-0.18) rectangle (\bc+\bw,\base+1.55);
  \node[font=\scriptsize,anchor=south west] at (\bc+\bw+0.05,\base+1.55) {band $B=[c-\tau,\,c+\tau]$};
  \draw[black,very thick] (\bc,\base) -- (\bc,\base+1.30);
  \fill[black] (\bc,\base+1.30) circle (1.6pt);
  \node[font=\scriptsize,align=center,anchor=north] at (\bc,\base-0.30) {planted mass $\eta$ at $c$};

  \draw[->,black] (\axL,\base) -- (\axR,\base);

  \draw[black] (-\Rt,\base-0.10) -- (-\Rt,\base+0.10);
  \draw[black] ( \Rt,\base-0.10) -- ( \Rt,\base+0.10);
  \draw[<->,black!70] (-\Rt,\base-0.70) -- (\Rt,\base-0.70);
  \node[black!70,font=\scriptsize,fill=white,inner sep=1pt,anchor=center] at (0,\base-0.70)
        {node range $[-R_{t+1},\,R_{t+1}]$,\ \ $R_{t+1}\asymp s\sqrt{t}$};

  \foreach \u in {-3.05,-2.45,-1.55,-0.55,0.55,1.55,2.45,3.05}{
     \fill[black] (\u,\base) circle (2.0pt);
  }
  \node[font=\scriptsize,anchor=south] at (0.0,\base+0.30) {Gauss nodes $u_1,\dots,u_{t+1}$};

  \node[draw=black,fill=white,rounded corners=2pt,font=\scriptsize,align=center]
        (msg) at (2.55,2.30)
        {\textbf{no node in $B$}\\$\Rightarrow$ certifiable mass $=0$};
  \draw[->,black] (msg.east) .. controls (4.0,2.15) and (\bc-1.0,1.55) .. (\bc-\bw-0.03,1.05);

\end{tikzpicture}
\caption{\textbf{The impossibility mechanism.}
A degree-$2t$ certificate sees the clean marginal $\mu$ only through its first $2t$ moments, which are reproduced exactly by the $(t{+}1)$-node Gauss measure supported on the nodes $u_1,\dots,u_{t+1}\in[-R_{t+1},R_{t+1}]$ with $R_{t+1}\asymp s\sqrt{t}$. A band $B$ pushed past the node range contains no node, so every measure matching the moments places zero mass there, and a corruption of any rate $\eta$ planted at $c$ stays invisible to the certificate.}
\label{fig:mechanism}
\end{figure}
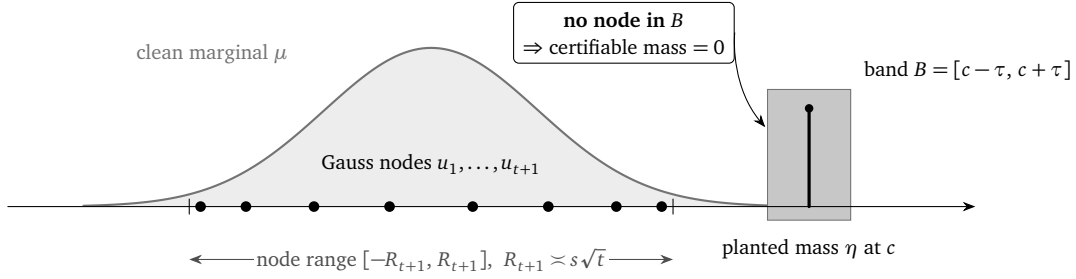

\subsection{The Christoffel function governs the resolution}
\label{sec:perspective-christoffel}

Let $\{\pi_k\}_{k\ge0}$ be the orthonormal polynomials of $\mu$ and
\begin{align}\label{eq:christoffel}
  K_{n}(x,y)=\sum_{k=0}^{n-1}\pi_k(x)\pi_k(y),
  \qquad
  \lambda_n(x)\defeq K_n(x,x)^{-1}=\Big(\sum_{k=0}^{n-1}\pi_k(x)^2\Big)^{-1}
\end{align}
the Christoffel--Darboux kernel and the Christoffel function of order $n$. The single fact we need is classical and gives $\lambda_n$ its meaning as a \emph{maximal mass}.

\begin{fact}[Christoffel function as maximal mass; \citealp{Nevai86,DetteStudden97}]
\label{fact:christoffel-mass}
For every $x_0\in\R$, \( \lambda_{t+1}(x_0)=\min\{\int p^2\,d\mu:\ \deg p\le t,\ p(x_0)=1\} =\max\{\nu(\{x_0\}):\ \nu\in\mathcal M_{2t}(\mu)\}. \) The two equalities are dual: the left is the extremal problem defining $\lambda_{t+1}$, and the right is the largest atom a measure with $\mu$'s first $2t$ moments may carry at $x_0$. In particular the Gauss weight at a node $u_i$ of $\pi_{t+1}$ equals $\lambda_{t+1}(u_i)$. Equivalently, in the finite-dimensional reproducing-kernel space $(\R[u]_{\le t},\langle\cdot,\cdot\rangle_\mu)$ with reproducing kernel $K_t$, $\lambda_{t+1}(x_0)=\|\delta_{x_0}\|^{-2}$ is the reciprocal squared norm of point-evaluation $p\mapsto p(x_0)$: the resolution threshold is the point-evaluation energy at $x_0$.
\end{fact}

\begin{proof}
For the right equality, let $\nu\in\mathcal M_{2t}(\mu)$ carry mass $w$ at $x_0$, and let $p$ have $\deg p\le t$, $p(x_0)=1$. Since $\deg p^2\le 2t$, $\int p^2\,d\nu=\int p^2\,d\mu$; dropping all of $\nu$ but the atom at $x_0$ gives $w=w\,p(x_0)^2\le\int p^2\,d\nu=\int p^2\,d\mu$. Minimizing the right side over admissible $p$ yields $w\le\lambda_{t+1}(x_0)$, and the Gauss measure with $x_0$ adjoined as a node attains it.
\end{proof}

The maximal-mass reading turns both barriers into one inequality. A corruption is hidden from a degree-$2t$ moment program exactly when the corrupted sample still lies in $\mathcal M_{2t}(\mu)$, i.e.\ when its moments are consistent with the clean ones. Placing a planted mass at a tail center $c$ is possible within $\mathcal M_{2t}(\mu)$ up to $\lambda_{t+1}(c)$, and no further.

\begin{proposition}[Christoffel resolution principle]
\label{prop:resolution}
Let a corruption of mass $\eta$ be planted at center $c$ (an atom, or a width-$O(s)$ cluster, at $c$). It is consistent with the clean marginal's first $2t$ moments (hence invisible to every degree-$2t$ moment certificate) if and only if
\begin{align}\label{eq:resolution}
  \eta\ \le\ \lambda_{t+1}(c).
\end{align}
Equivalently, the least certificate degree that \emph{resolves} a mass-$\eta$ corruption at center $c$ is $2t^\star(c,\eta)$ with $t^\star=\min\{t:\lambda_{t+1}(c)<\eta\}$.
\end{proposition}

\begin{proof}
Immediate from \cref{fact:christoffel-mass}: a representation in $\mathcal M_{2t}(\mu)$ carrying the planted mass $\eta$ at $c$ exists iff $\eta\le\lambda_{t+1}(c)$, and such a representation is precisely a clean-moment-consistent explanation of the corrupted sample, against which no degree-$2t$ certificate can act. (For a width-$O(s)$ cluster the same holds up to the $\Theta(1)$ factor by which $\lambda_{t+1}$ varies over the cluster.)
\end{proof}

Both barriers are already visible in \cref{prop:resolution}: they are the two regimes of where $\lambda_{t+1}(c)$ sits relative to the planted mass. The first is sharp enough to record on its own.

\begin{fact}[The degree-$2$/degree-$4$ gap at the $\eta$-tail]
\label{fact:gap}
At the threshold center $c=\sigma/\sqrt\eta$ of a variance-$\sigma^2$ marginal,
\begin{align}\label{eq:resolution-spike}
  \lambda_2(c)=\frac{1}{1+c^2/\sigma^2}=\frac{\eta}{1+\eta}=\Theta(\eta),
  \qquad
  \lambda_3(c)\le\frac{1}{\pi_2(c)^2}=\Theta(\eta^2).
\end{align}
Hence a mass-$\eta$ corruption at $c$ is invisible to degree $2$ ($\eta\le\lambda_2(c)$) but resolved by degree $4$ ($\eta>\lambda_3(c)$): a gap of one moment order, with a sharp constant, in $\mu$ and the certificate order alone.
\end{fact}

\begin{proof}
Apply \cref{prop:resolution} with the first orthonormal polynomials of a variance-$\sigma^2$ law, $\pi_0\equiv1$, $\pi_1(u)=u/\sigma$, $\pi_2(u)=(u^2-\sigma^2)/(\sigma^2\sqrt2)$ (up to the $\Theta(1)$ fourth-moment constant), and evaluate $\lambda_{k+1}(c)=\big(\sum_{j\le k}\pi_j(c)^2\big)^{-1}$ at $c=\sigma/\sqrt\eta$.
\end{proof}

% claim: R02a
\begin{remark}[The whole ladder, not just one rung]\label{rem:hiding-ladder}
\Cref{fact:gap} is one rung of a ladder the resolution principle produces for free. Fix a marginal $\mu$ with infinite support and finite moments through order $2t$, and write $\kappa_t>0$ for the leading coefficient of its degree-$t$ orthonormal polynomial $\pi_t$ (not $\gamma_t$: in this paper $\gamma$ is the margin). In \eqref{eq:christoffel} the degree-$t$ term dominates as $\abs c\to\infty$, so $\lambda_{t+1}(c)\,c^{2t}\to\kappa_t^{-2}$, and inverting the hiding condition $\eta\le\lambda_{t+1}(c)$ of \cref{prop:resolution} gives the \emph{hiding radius}
\[
  R_t(\eta)\;\defeq\;\sup\{\abs c:\ \lambda_{t+1}(c)\ge\eta\}
  \;=\;(1+o(1))\,\big(\kappa_t^{2}\,\eta\big)^{-1/2t}
  \qquad(\eta\to0),
\]
finite for every $\eta\in(0,\lambda_{t+1}(0)]$ and valid at every degree $t\ge1$ on the same $\mu$, with the $o(1)$ and $\kappa_t$ depending on $(\mu,t)$; no uniformity in $t$ is claimed. One law, every order: the reach grows like $\eta^{-1/2t}$, so each extra degree flattens the exponent from $1/2t$ to $1/(2t+2)$ (the prefactor $\kappa_t^{-1/t}$ is degree-dependent, which is why we read the exponent and not the constant), and \cref{fact:gap} is the $t=1$ versus $t=2$ instance of it at $c=\sigma/\sqrt\eta$. The ladder of \emph{hideable} corruption is therefore automatic. What \cref{thm:degbarrier} adds at $t=1$ is the harder step, from ``consistent with the moment class'' to ``defeats the realized reweighting program''; its general-$t$ analogue is open (\cref{sec:open}).
\end{remark}

The same principle fixes the ceiling for an off-center \emph{band}: it carries forced certificate mass exactly when $\lambda_{t+1}(c)$ has dropped to the target $\rho$, i.e.\ exactly when $c$ enters the node range $|c|\lesssim R_{t+1}$, since the smallest Christoffel value at the edge is $\Theta(\rho)$ exactly at the largest node. Below the node range the ceiling is $0$ unconditionally; inside it, $\Theta(\mu(B))$ modulo the weighted-extremal estimate isolated in \cref{sec:perspective-tight}. Both regimes are forced on a learner in \cref{sec:results}.

\begin{example}[The hiding game]\label{ex:hiding}
Make the numbers concrete. Let $\mu$ be a clean, centered, unit-variance log-concave law ($s=1$), and let the adversary corrupt a rate $\eta=10^{-4}$ of the data. The resolution principle says the adversary's best move is to plant that mass as far out as it can still hide: the threshold center is $c=\sigma/\sqrt\eta=100$ standard deviations from the origin, a spike of weight $\eta=10^{-4}$ at $u=100$. Audit it the way a degree-$2$ certificate does, through the mean and the variance. The spike shifts the mean by $\eta c=10^{-2}$ (absorbed into recentering) and adds $\eta c^2=1$ to the variance: it spends \emph{exactly one unit of variance budget}, the whole budget the clean law already owns. Through the first two moments the corrupted sample is indistinguishable from a clean unit-variance law, so $\lambda_2(c)=\Theta(\eta)$ and no degree-$2$ outlier program can so much as name the spike. Degree $4$ changes the ledger. The spike contributes $\eta c^4=10^{4}$ to the fourth moment, while a clean unit-variance log-concave law carries $\E[U^4]=\Theta(1)$: the fourth moment is inflated by a factor $\sim 1/\eta$, a discrepancy the first two moments hid completely. This is exactly the $\lambda_3(c)=\Theta(\eta^2)\ll\eta$ of \cref{fact:gap}: the atom that saturates the degree-$2$ ceiling sits four orders \emph{above} the degree-$4$ ceiling, and a degree-$4$ certificate exposes it. One spike, two verdicts; all that changed is the order of the moments one is allowed to look through.
\end{example}

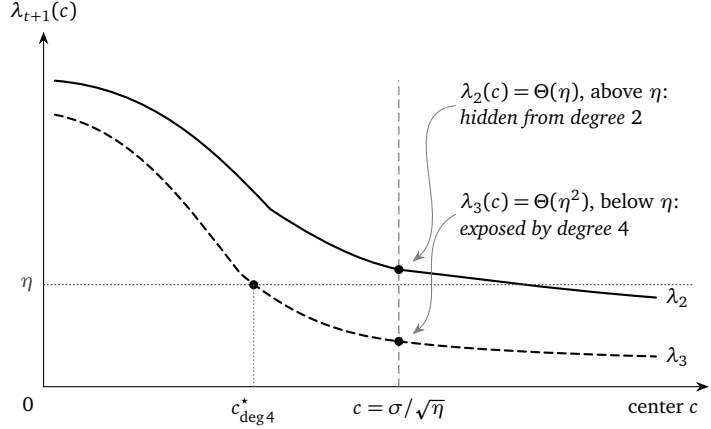
\begin{figure}[t]
\centering
\begin{tikzpicture}[>=Stealth,line cap=round]
  \def\W{8.6}
  \def\H{4.4}
  \draw[->,black] (0,0) -- (\W+0.2,0);
  \node[font=\scriptsize,anchor=north east] at (\W+0.2,-0.06) {center $c$};
  \draw[->,black] (0,0) -- (0,\H+0.3) node[above,font=\scriptsize] {$\lambda_{t+1}(c)$};
  \node[below left,font=\scriptsize] at (0,0) {$0$};

  \def\etaY{1.35}
  \def\cstar{4.7}

  \draw[black,thick]
    (0.15,4.05)
    .. controls (1.4,3.95) and (2.2,3.2) ..
    (3.0,2.35)
    .. controls (3.9,1.75) and (4.4,1.62) ..
    (\cstar,1.55)
    .. controls (5.6,1.45) and (6.6,1.30) ..
    (8.1,1.18);
  \node[font=\scriptsize,black,anchor=west] at (8.1,1.18) {$\lambda_{2}$};

  \draw[black,thick,densely dashed]
    (0.15,3.6)
    .. controls (1.2,3.4) and (1.9,2.4) ..
    (2.6,1.5)
    .. controls (3.3,0.92) and (3.9,0.7) ..
    (\cstar,0.6)
    .. controls (5.5,0.5) and (6.6,0.44) ..
    (8.1,0.40);
  \node[font=\scriptsize,black,anchor=west] at (8.1,0.40) {$\lambda_{3}$};

  \draw[black!75,densely dotted] (0,\etaY) -- (\W,\etaY);
  \node[left,font=\scriptsize,black!75] at (0,\etaY) {$\eta$};

  \def\crA{2.78}
  \fill[black] (\crA,\etaY) circle (1.7pt);
  \draw[black!55,densely dotted] (\crA,0) -- (\crA,\etaY);
  \node[below,font=\scriptsize] at (\crA,-0.02) {$c^\star_{\deg 4}$};

  \draw[black!60,densely dashed] (\cstar,0) -- (\cstar,4.05);
  \node[below,font=\scriptsize] at (\cstar,-0.02) {$c=\sigma/\sqrt{\eta}$};
  \fill[black] (\cstar,1.55) circle (1.7pt);
  \fill[black] (\cstar,0.6)  circle (1.7pt);
  \node[anchor=west,font=\scriptsize,align=left,fill=white,inner sep=1.5pt] (rdtwo) at (5.45,3.72)
        {$\lambda_2(c)=\Theta(\eta)$, above $\eta$:\\ \emph{hidden from degree} $2$};
  \node[anchor=west,font=\scriptsize,align=left,fill=white,inner sep=1.5pt] (rdthr) at (5.45,2.28)
        {$\lambda_3(c)=\Theta(\eta^2)$, below $\eta$:\\ \emph{exposed by degree} $4$};
  \draw[black!50,->,shorten >=1.5pt] (rdtwo.west) to[out=186,in=40]  (\cstar+0.14,1.62);
  \draw[black!50,->,shorten >=1.5pt] (rdthr.west) to[out=208,in=40]  (\cstar+0.14,0.66);

\end{tikzpicture}
\caption{\textbf{The resolution threshold.}
The certificate ceiling $\lambda_{t+1}(c)$ decays from the bulk, and a corruption of rate $\eta$ at center $c$ is exposed exactly where the curve drops below the dotted line $\eta$. Raising the degree from $2$ to $4$ pulls the curve down so that at the tail center $c=\sigma/\sqrt{\eta}$ the value falls from $\Theta(\eta)$ (still hidden) to $\Theta(\eta^2)$ (now resolved): the degree-$2$/degree-$4$ gap of \cref{ex:hiding}.}
\label{fig:threshold}
\end{figure}

\subsection{The degree of robustness is the Christoffel degree}
\label{sec:perspective-bridge}

\Cref{prop:resolution} is the lower-bound counterpart of a tool already in use. \Citet{LasserrePauwels19} and \citet{PauwelsLasserreCDOutliers} detect outliers and recover supports by thresholding the \emph{empirical} Christoffel function: a point $x$ is flagged when $\lambda_n(x)$ (built from the empirical moments) is small, i.e.\ when degree-$n$ moments place $x$ outside the bulk. \Cref{prop:resolution} says that this is exactly the mechanism, and the only mechanism, available to a degree-$2t$ outlier-removal program: it can down-weight a planted corruption at $c$ if and only if the order-$(t{+}1)$ Christoffel function already resolves it, $\lambda_{t+1}(c)<\eta$. The reweighted-hinge programs of \citet{Talwar20,Shen25} run this detector at degree $2$; raising the hinge program to degree $2t$ is raising the Christoffel detector to order $t{+}1$. We can now state the law the framework yields, the \emph{degree of robustness}: the SoS degree at which a certificate first tolerates corruption rate $\eta$ at margin scale equals the Christoffel order that resolves a mass-$\eta$ cluster at the $\eta$-tail of the marginal,
\[
  t^\star(\eta)\ =\ \min\bigl\{\,t:\ \lambda_{t+1}(\sigma/\sqrt\eta)\,<\,\eta\,\bigr\}.
\]
It is a property of $\mu$ and the order, which is why the resource is degree and not sample size: more samples sharpen the \emph{empirical} Christoffel function toward its population limit, but the population value at the planted center is fixed by the marginal and the order, so an infinite sample at fixed degree buys nothing against a hidden spike. The only knob that lowers $\lambda_{t+1}(c)$ past $\eta$ is the order itself. The degree-$2t$ algorithm of \cref{sec:results} is this law made operational.

\paragraph{Relationship to the Christoffel-function program of Lasserre, Pauwels, and Putinar.}
The Christoffel function is by now a central object in moment-based data analysis \citep{PauwelsLasserre16,LasserrePauwels19,PauwelsPutinarLasserre21,LasserrePauwelsPutinar22}, where the \emph{empirical} Christoffel function serves as a \emph{detector}: a large value $\lambda_n(x)$ marks $x$ as typical, a small value flags it as an outlier, and sublevel sets recover the shape of a point cloud \citep{PauwelsLasserre16,PauwelsLasserreCDOutliers}. The algebraic engine, the identification of the moment--SoS hierarchy with the Christoffel--Darboux kernel \citep{LasserrePauwelsCDK,Lasserre23PosCert}, is the same one we use for the SoS dual of the moment LP. Our reading, however, runs the sign the other way: a \emph{large} $\lambda_{t+1}(c)$ is bad, not good. By \cref{prop:resolution} it is exactly the corruption mass an adversary can hide at $c$ from a degree-$2t$ certificate, so the centers a detector finds hardest to clear are precisely the centers a robust learner cannot defend. In one line, the \emph{detector's blind spot} is the \emph{adversary's hiding place}:
\[
  H_\eta\ \defeq\ \{\,c:\ \lambda_{t+1}(c)\ge\eta\,\},
\]
the set of centers at which a rate-$\eta$ corruption survives every degree-$2t$ certificate; raising the degree is the only way to shrink it. Where that program asks ``is $x$ typical, and at what order does the empirical detector see it?'', we ask ``how much malicious mass can survive at $c$, and at what order is it forced out?'', and we turn the answer into an \emph{impossibility} statement for the certificate method, rather than a detection rule. To our knowledge the Christoffel function has served this line as a constructive tool, for detection, support and density estimation, and typicality, and has not previously been used as a barrier from which to derive learning-theoretic lower bounds; that inversion (\cref{fig:inversion}), and the margin--degree law it yields, is what is new here.

\subsection{From barrier to law: a weighted-Chebyshev reduction}
\label{sec:perspective-tight}

The framework's lower edge leaves one quantity open: a \emph{matching} upper certificate, a degree-$2t$ minorant whose band mass meets the moment-LP ceiling. We can say precisely what closing it requires, and reduce it to one classical question about orthogonal polynomials. The key is to parametrize the minorant by a square. For $q\in\R[u]_{\le t}$ with $|q|\ge1$ on $\R\setminus B$, the degree-$2t$ polynomial $p=1-q^2$ is automatically a minorant of $\mathbf 1_B$ (on $B$, $p\le1$; off $B$, $q^2\ge1$ forces $p\le0$), and $\E_\mu[p]=1-\int q^2\,d\mu$. Optimizing over $q$ therefore yields a certifiable band mass governed by the \emph{weighted-Chebyshev extremal value}
\begin{align}\label{eq:weighted-cheb}
  \mathsf E_t(B)\ \defeq\ \min\Big\{\textstyle\int q^2\,d\mu:\ q\in\R[u]_{\le t},\ |q|\ge1\text{ on }\R\setminus B\Big\}.
\end{align}

\begin{proposition}[Reduction to a weighted extremal problem]
\label{prop:reduction}
The degree-$2t$ certificate ceiling for the band $B$ is at least $1-\mathsf E_t(B)$. Moreover the Gauss-node interpolant $q_0$ of degree $t$, defined by $q_0(u_j)=\mathbf 1[u_j\notin B]$ on the nodes $u_1,\dots,u_{t+1}$ of $\pi_{t+1}$, has $\int q_0^2\,d\mu=1-W_B$, where $W_B\defeq\sum_{u_j\in B}\lambda_{t+1}(u_j)$ is the Gauss-quadrature (Christoffel) mass of the nodes inside $B$. Thus the certifiable band mass is controlled by $W_B$, the discrete Christoffel content of the band, up to the admissibility of $q_0$.
\end{proposition}

\begin{proof}
The first claim is the construction $p=1-q^2$ above: each admissible $q$ gives a degree-$2t$ minorant of value $1-\int q^2\,d\mu$, and the supremum of these values is $1-\mathsf E_t(B)$, a lower bound on the ceiling. For the identity, $\deg q_0^2=2t\le 2t+1$, so Gauss quadrature is exact:
\[
  \int q_0^2\,d\mu=\sum_j\lambda_{t+1}(u_j)\,q_0(u_j)^2=\sum_{u_j\notin B}\lambda_{t+1}(u_j)=1-W_B,
\]
using $\sum_j\lambda_{t+1}(u_j)=\int 1\,d\mu=1$ (\cref{fact:christoffel-mass} for the weights).
\end{proof}

The reduction isolates the one missing estimate: an \emph{admissible} polynomial achieving the node-interpolant value, up to the spectral edge.

\begin{lemma}[Weighted-extremal value at the edge; the remaining estimate]
\label{lem:wext}
Let $\mu$ be centered log-concave of standard deviation $s$ and $B=[c-\tau,c+\tau]$, $\tau=\theta s$, $\theta=\Theta(1)$. For every $\delta\in(0,1)$ there is $c_\delta>0$ such that
\[
  |c|\le(1-\delta)R_{t+1}\ \Longrightarrow\ \mathsf E_t(B)\le 1-c_\delta\,\mu(B).
\]
\end{lemma}

\Cref{lem:wext} is exactly the weighted-extremal estimate that classical potential theory for exponential weights supplies: in the bulk, Christoffel asymptotics \citep{MateNevaiTotik91,Totik00} give $W_B=\mu(B)(1+o(1))$, while the infinite--finite range inequality \citep{LevinLubinsky01,KriecherbauerMcLaughlin99} controls the polynomial mass beyond the Mhaskar--Rakhmanov--Saff number $a_{t+1}\asymp R_{t+1}$, so the cost stays below $1$ up to the edge. We have not derived the uniform non-asymptotic constant $c_\delta$ from scratch, so we state \cref{lem:wext} as the residual input; granting it, the barrier becomes an exact law.

\begin{remark}[The Gaussian case, made explicit]\label{rem:gaussian-wext}
For $\mu=\mathcal N(0,1)$ the construction is explicit and localizes \cref{lem:wext} to one Hermite estimate. Take the degree-$t$ node interpolant $q_0$ with $q_0(u_j)=\mathbf 1[u_j\notin B]$ on the Gauss--Hermite nodes; by \cref{prop:reduction}, $\int q_0^2\,d\gamma=1-W_B$ exactly, and using the identity $\ell_j=\lambda_{t+1}(u_j)K_t(\cdot,u_j)$ for the fundamental Lagrange polynomials, the \emph{entire} admissibility defect collapses to the Christoffel--Darboux leakage
\[
  \mathrm{LEAK}\ \defeq\ \int_{B^c}\Big|\textstyle\sum_{u_j\in B}\lambda_{t+1}(u_j)\,K_t(x,u_j)\Big|\,d\gamma(x).
\]
A height-$\beta$ lift by the normalized kernel $L_{y_\ast}=K_t(\cdot,y_\ast)/K_t(y_\ast,y_\ast)$ at the inner-edge node $y_\ast$ makes $q=(q_0+\beta L_{y_\ast})/(1+\beta')$ admissible at $L^2(\gamma)$ cost $\lambda_{t+1}(y_\ast)$, giving \cref{lem:wext} once $\mathrm{LEAK}=o(W_B)$. Bounding $\mathrm{LEAK}$ is the crux, and it is subtler than it looks. The naive per-node bound $\mathrm{LEAK}\le\sum_{u_j\in B} \lambda_{t+1}(u_j)\int_{B^c}|K_t(\cdot,u_j)|\,d\gamma$ loses a factor $\Theta(\log t)$, the Hermite Lebesgue constant. One summation by parts over the contiguous in-band block (whose node phases are quasi-arithmetic by Plancherel--Rotach \citep{Szego39,DeiftKrMcLVZ99}, so the $\ell_j$ alternate) replaces the block by its two band-edge Lagrange functions $\tfrac12\ell_M+\tfrac12\ell_m$, each a single node, and yields the explicit \emph{single-power} bound
\[
  |e_0(x)|\ \le\ \frac{2\,\sqrt{K_t(x,x)\,\bar\lambda}}{1+t\,\rho_{\mathrm{eq}}(x)\,\mathrm{dist}(x,B)},
  \qquad \bar\lambda=\max_{u_j\in B}\lambda_{t+1}(u_j).
\]
(A \emph{squared}-power pointwise bound is false: the bulk scaling limit of $e_0$ is the partial Dirichlet kernel $\sum_j\operatorname{sinc}(s-j)$, which decays only like $1/\mathrm{dist}$, with a counterexample at each half-integer.) The single-power bound already suffices, because $\sqrt t$ beats $\log t$: it integrates to $\mathrm{LEAK}\asymp\bar\lambda\log t$ while $W_B\asymp\sqrt t\,\bar\lambda$, so $\mathrm{LEAK}=o(W_B)$ and $c_\delta\to\tfrac12$: the matching upper certificate, hence the exact law for the Gaussian, follows. The one ingredient we do not certify from scratch is the uniform non-asymptotic constant in this single-node leakage sum (the $O(1/t)$ Plancherel--Rotach error summed across the band); pinning it is the last mile to an unconditional law.
\end{remark}

\begin{remark}[Why the edge is intrinsic]
\label{rem:nogo}
An order-$t$ certificate is a lens, and like any lens it has a focal reach. Gauss quadrature replaces $\mu$ by a $(t{+}1)$-node measure matching its first $2t$ moments, and that measure places its outermost node at $R_{t+1}\asymp s\sqrt t$; beyond it every band is swept empty. So an order-$t$ certificate focuses only about $\sqrt t$ standard deviations into the tail, and a corruption parked past that edge is invisible not for want of samples but because the lens does not reach. The edge is intrinsic, not an artifact of the proof: \emph{no} elementary admissible $q$ reaches beyond $|c|=O(s)$. The linear $q=(u-c)/\tau$ has $\int q^2\,d\mu=(s^2+c^2)/\tau^2<1$ only for $|c|=O(s)$, and a hyperbolic $q\propto\cosh(a(u-c))$ is blocked by the Gaussian penalty $e^{2a^2s^2}$ at the same scale. Any admissible $q$ must be $\ge1$ on the entire tail while staying $L^2(\mu)$-small, which forces the polynomial growth $\sim u^t$ tuned to the weight, the weighted-equilibrium polynomial. The reach of degree $t$ is therefore exactly the Mhaskar--Rakhmanov--Saff number, which is why \cref{lem:wext} needs weighted potential theory and not a closed form. This both explains the barrier and pins the only missing ingredient.
\end{remark}

The reduction (\cref{prop:reduction}) together with this residual estimate (\cref{lem:wext}) supplies the matching upper certificate $2t=\Theta((|c|/s)^2)$ to the framework's lower edge, modulo a single weighted-extremal statement that is purely about orthogonal polynomials, decoupled from the learning model, and supported by the cited edge asymptotics. \Cref{sec:results} draws the learning consequence as an exact law. We regard deriving its uniform constant as the natural next target.

\section{Consequences: barriers and a matching algorithm}
\label{sec:results}

The framework of \cref{sec:perspective} has already done the work; the robust-learning statements below are its readings at specific centers. The certificate ceiling is $\Theta(\mu(B))$ inside the node range and $0$ beyond it (\cref{prop:resolution}), and each result is one corner of that dichotomy. The single resource is the \emph{degree} of the Sum-of-Squares (SoS) outlier-removal certificate, the order at which the Christoffel function resolves a corruption. The reweighted-hinge approach of \citet{Talwar20,Shen25} runs on a degree-$2$ certificate (a variance bound on the reweighted sample); we show what raising the degree to $2t$ buys, and what it does not. \Cref{thm:tradeoff} reads the resolution principle off an off-center band, tying the certifiable band mass to the SoS degree through Gauss quadrature and yielding a margin--degree dichotomy that pins down where the $\log(1/\eps)$ in \citeauthor{Shen25}'s margin comes from; \cref{thm:tight} closes it to an exact law. \Cref{thm:degbarrier} realizes the degree-$2$ corner as one explicit instance, on which degree-$2$ reweighting is stuck at the $\eta^{1/2}$ rate while degree-$4$ escapes, exactly as $\lambda_2(c)$ versus $\lambda_3(c)$ predicts in \eqref{eq:resolution-spike}. \Cref{thm:algorithm} is the matching algorithm: a degree-$2t$ program, the order-$(t{+}1)$ empirical Christoffel detector, that traces the achievable frontier $\eta^{1-1/2t}$ with $t=1$ recovering \citeauthor{Shen25} exactly. \Cref{prop:floor} is the information-theoretic floor $\eta/2$ that the algorithm matches in rate. The two barriers (\cref{thm:tradeoff,thm:degbarrier}) are against the certificate method, not against all algorithms; we state the model each time. Throughout, the marginal is the $\gamma$-margin log-concave mixture of \cref{sec:prelim}, $\bar\sigma^2(w)$ is its reweighted directional variance, and $C$ is an absolute constant whose value may change between displays. Full proofs are deferred to \cref{app:proofs}.

\subsection{The margin--degree tradeoff}

The first result reads the certificate ceiling of \cref{sec:perspective-moment} off an off-center band. By \eqref{eq:perspective-md} the strongest band-mass bound a degree-$2t$ certificate can prove is the moment-LP value $\underline m(2t)$; the tradeoff is the contrapositive of the ceiling, once the band clears the node range $\underline m(2t)=0$ and no certificate exists. The scope is the moment/SoS-certificate model of \cref{def:cert}, \citeauthor{Shen25}'s ``our approach'' made formal: it is \emph{not} an information-theoretic or unconditional-SQ lower bound, and a learner that reads the marginal in any other way is outside it.

\begin{restatable}[Margin--degree tradeoff; certificate method]{theorem}{thmtradeoff}
\label{thm:tradeoff}
Let $U=w\cdot x$ for a fixed direction $w$, with $U$ centered log-concave of standard deviation $s$, and let $B=[c-\tau,c+\tau]$ be a band of half-width $\tau=\theta s$, $\theta=\Theta(1)$, at center $c$. Suppose there is a degree-$2t$ SoS certificate that lower-bounds the band mass, a proof of $\Pr_U[U\in B]\ge\rho=\Omega(1)$ whose only input is the first $2t$ moments of $U$ (made precise in \cref{lem:gauss}). Then
\begin{align}\label{eq:tradeoff-deg}
  2t \;=\; \Omega\big((\abs{c}/s)^2\big)
  \quad\text{(sub-Gaussian marginal)},
  \qquad
  2t \;=\; \Omega\big(\abs{c}/s\big)
  \quad\text{(sub-exponential marginal)}.
\end{align}
Consequently, any certificate learner attaining inlier error $\eps$ on the marginal of \cref{sec:prelim} must pay \emph{either} SoS degree
\begin{align}\label{eq:tradeoff-dichotomy}
  t \;=\; \Omega\big(\log(1/\eps)\big)
  \quad\text{(runtime } n^{\Omega(\log(1/\eps))}\text{)},
  \qquad\text{or}\qquad
  \gamma \;=\; \Omega\big(\sqrt{\log(1/\eps)}\,/\sqrt d\big).
\end{align}
The $\log(1/\eps)$ in \citeauthor{Shen25}'s margin is conserved: it is exchangeable against SoS degree, and removable only by leaving the one-shot moment-certificate model.
\end{restatable}

\begin{proof}[Proof sketch]
A degree-$2t$ moment-only certificate of $\Pr_U[U\in B]\ge\rho$ is, by the minorant reduction (\cref{lem:gauss}\,(i)), a polynomial $p$ of degree $\le 2t$ with $p\le\mathbf 1_B$ pointwise together with a degree-$2t$ proof, from the moment axioms alone, that $\int p\,d\nu\ge\rho$. Because the proof reads only the moments, its conclusion must hold for \emph{every} nonnegative measure $\nu$ matching the moments $m_0,\dots,m_{2t}$, $m_k=\E_\mu[U^k]$. Weak moment-LP duality then caps the best certifiable value by $\inf\{\nu(B):\nu\ge0,\ \int u^k\,d\nu=m_k\ (k\le 2t)\}$, so a single matching measure that places \emph{no} mass on $B$ kills the certificate. Gauss quadrature (\cref{lem:gauss}\,(ii)) supplies one: the $(t{+}1)$-node Gauss measure matches all $2t$ moments and is supported only on the $t{+}1$ zeros $u_1,\dots,u_{t+1}$ of the orthonormal polynomial $\pi_{t+1}$. With $R_{t+1}\defeq\max_i\abs{u_i}$ the largest absolute zero, a band with $\abs{c}-\tau>R_{t+1}$ contains no node, hence carries zero Gauss mass, hence admits no certificate. The largest-root bound (\cref{lem:root}) gives $R_{t+1}=O(s\sqrt t)$ in the sub-Gaussian case and $R_{t+1}=O(st)$ in the sub-exponential case; pushing the band out to $\abs{c}-\theta s>R_{t+1}$ yields \eqref{eq:tradeoff-deg} by contraposition. The dichotomy \eqref{eq:tradeoff-dichotomy} follows by setting $c$ to the worst $\beta$-quantile $\abs{c^\star}$ with $\beta\lesssim\eps$ and reading off the log-concave tail (\cref{lem:quant}). The full proof, including the geometric step relating a standardized half-width to the physical margin, is in \cref{app:proofs}.
\end{proof}

\begin{remark}[Why \citeauthor{Shen25} pays $\log$, not $\sqrt{\log}$]\label{rem:why-log}
The two exponents in \eqref{eq:tradeoff-deg} feed the \emph{same} $\log(1/\eps)$ floor in \eqref{eq:tradeoff-dichotomy}, because the quantile and the largest root scale together. \citeauthor{Shen25}'s marginal is sub-exponential (Laplace-type tails, Freud exponent $\alpha=1$): the $\eps$-quantile is $\abs{c^\star}=\Theta(s\log(1/\eps))$ and the largest root is $R_{t+1}=O(st)$, so covering the $\eps$-tail forces $t=\Omega(\log(1/\eps))$ and the margin pays $\log(1/\eps)$ directly. The $\sqrt{\cdot}$ appears only for genuinely sub-Gaussian marginals. This is the structural reason behind the limitation \citeauthor{Shen25} records for the approach.
\end{remark}

\begin{remark}[The barrier is elementary; tightness reduces to one estimate]\label{rem:status-tradeoff}
The spine is self-contained: weak duality is elementary (\cref{app:proofs}), and the two named cases of \cref{lem:root} (sub-Gaussian and sub-exponential) are proved by Jacobi-matrix eigenvalues and Gershgorin, with no black box. The fully general log-concave case of \cref{lem:root} (Freud exponent $\alpha\in(1,2)$) invokes the exponential-weight asymptotics of \citet{LevinLubinsky01}; we flag this rather than pretend it is elementary. The matching \emph{upper} certificate that makes \eqref{eq:tradeoff-deg} an equality is supplied by \cref{thm:tight} below, through the weighted-Chebyshev reduction of \cref{sec:perspective} (\cref{prop:reduction}) and modulo the single weighted-extremal estimate \cref{lem:wext}.
\end{remark}

% claim: R01
\begin{remark}[Smoothing does not help, and where the barrier stops]\label{rem:smoothing}
One natural escape from \cref{thm:tradeoff} is to certify a \emph{smoothed} band instead of the hard indicator. It is not an escape. Any polynomial surrogate $\psi$ with $\deg\psi\le 2t$ is a moment functional, so it takes the same value on $\mu$ and on the $(t{+}1)$-node Gauss measure $\nu$ of \cref{lem:gauss}\,(ii): $\int\psi\,d\nu=\int\psi\,d\mu$. For a band avoiding the $t{+}1$ nodes, $\nu$ reports mass $0$ while agreeing with $\mu$ on every such functional, so no smoothed degree-$2t$ certificate can lower-bound that band's mass either. The node-avoidance qualifier is the entire content, not decoration: for a band \emph{containing} the nodes a positive low-degree certificate does exist, e.g.\ $\psi(u)=1-u^2/9\le\mathbf 1_{[-3,3]}$ certifies $\sigma([-3,3])\ge 8/9$ for every $\sigma$ matching the first two moments of $N(0,1)$. Note also that the witness is atomic: what is barred is a certificate valid on the \emph{whole} moment class $\mathcal M_{2t}(\mu)$, which is exactly the certificate model of \cref{def:cert}; a learner able to rule the atomic explanation out by other means is outside that model.
\end{remark}

\begin{corollary}[Freud-exponent family; the margin exponent is $1/\alpha$]
\label{cor:alpha-family}
Suppose the clean marginal has a Freud-type tail $e^{-\abs{u/s}^{\alpha}}$ with exponent $\alpha\ge1$ ($\alpha=2$ sub-Gaussian, $\alpha=1$ sub-exponential). Then the dichotomy of \cref{thm:tradeoff} reads: any certificate learner of inlier error $\eps$ must pay SoS degree $t=\Omega(\log(1/\eps))$ \emph{or} margin $\gamma=\Omega\big((\log(1/\eps))^{1/\alpha}/\sqrt d\big)$. The degree floor is \emph{universal} in $\alpha$; the margin exponent $1/\alpha$ is the reciprocal Freud exponent, equal to the growth rate of the equilibrium-measure (MRS) number $a_t\asymp s\,t^{1/\alpha}$ (\cref{fact:freud}), which is where the orthonormal-polynomial zeros end. The cases $\alpha=2,1$ recover \eqref{eq:tradeoff-dichotomy} and the $\log(1/\eps)$ margin \citet{Shen25} records.
\end{corollary}

\begin{proof}
By \cref{fact:freud} the largest zero is $R_{t+1}=(1+o(1))\,a_{t+1}\asymp s\,t^{1/\alpha}$, so the no-node condition $\abs c-\theta s>R_{t+1}$ of \cref{thm:tradeoff} forces $2t=\Omega((\abs c/s)^{\alpha})$. The $\eps$-quantile of the $e^{-\abs{u/s}^{\alpha}}$ tail is $\abs{c^\star}=\Theta\big(s(\log(1/\eps))^{1/\alpha}\big)$, whence $2t=\Omega((\abs{c^\star}/s)^{\alpha})=\Omega(\log(1/\eps))$ at fixed standardized half-width; and the geometric correspondence of \cref{lem:quant} turns the escape condition into $\gamma=\Omega(\abs{c^\star}/(s\sqrt d))=\Omega\big((\log(1/\eps))^{1/\alpha}/\sqrt d\big)$.
\end{proof}

So far the tradeoff is a lower bound on degree. The weighted-Chebyshev reduction of \cref{sec:perspective} (\cref{prop:reduction}) upgrades it to a \emph{matching} ceiling, so the threshold is not merely necessary but exact, conditional on the single residual orthogonal-polynomial estimate \cref{lem:wext}.

\begin{theorem}[Tight margin--degree threshold; conditional on \cref{lem:wext}]
\label{thm:tight}
Assume \cref{lem:wext}. For a sub-Gaussian log-concave marginal, the degree-$2t$ certificate ceiling for a band of half-width $\tau=\Theta(s)$ at center $c$ is $\Theta(\mu(B))$ for $\abs{c}\le(1-\delta)R_{t+1}$ and $0$ for $\abs{c}>R_{t+1}$ (\cref{thm:tradeoff}). Hence the threshold of \cref{thm:tradeoff} is tight: \( 2t=\Theta\big((\abs{c}/s)^2\big). \)
\end{theorem}

\begin{proof}
Upper: the reduction \cref{prop:reduction} and the edge estimate \cref{lem:wext} give a minorant of value $\ge c_\delta\,\mu(B)$ for $\abs{c}\le(1-\delta)R_{t+1}$. Lower: \cref{thm:tradeoff} gives ceiling $0$ once $\abs{c}>R_{t+1}$, i.e.\ once $2t<((\abs{c}/s)/C_1)^2$. With $R_{t+1}=\Theta(s\sqrt t)$ (\cref{lem:root}) the two meet at $2t=\Theta((\abs{c}/s)^2)$.
\end{proof}

\begin{remark}[The equilibrium-reach reading]
\label{rem:reach}
Re-reading the results above, one number governs them: how far into the tail a degree-$t$ certificate can see. We name it only to thread them; the identity is classical and already noted in \cref{rem:nogo}. The reach is the largest zero $R_{t+1}$ of $\pi_{t+1}$ (\cref{lem:witness}: a band beyond it carries certifiable mass $0$), which by \cref{fact:freud} is $(1+o(1))\,a_t$, the Mhaskar--Rakhmanov--Saff number, i.e.\ the support edge of the order-$t$ weighted-equilibrium measure of the marginal's weight \citep{LevinLubinsky01}. Off this one edge:
\begin{itemize}
\item \emph{Unconditional lower edge.} A mass-$\eta$ corruption at the $\eta$-tail is unresolvable below the order at which $a_t$ reaches the $\eta$-quantile, giving $2t=\Omega((\abs{c}/s)^\alpha)$ for a Freud weight $e^{-\abs{u/s}^\alpha}$ (\cref{thm:tradeoff,cor:alpha-family}), with margin exponent the reciprocal Freud exponent $1/\alpha$. This side is elementary: Gauss quadrature and the largest-root bound \cref{lem:root}, proved by hand for the sub-Gaussian and sub-exponential cases. Weighted potential theory (equilibrium measures, Levin--Lubinsky) enters only for general $\alpha$ and for the matching edge below.
\item \emph{Matching upper edge, tight but conditional.} Granting the weighted-extremal estimate \cref{lem:wext} (and unconditionally only for the Gaussian, \cref{rem:gaussian-wext}), the bound is tight: $2t=\Theta((\abs{c}/s)^2)$ (\cref{thm:tight}).
\item \emph{The sharp corner.} From $t=1$ to $t=2$ the reach jumps from $a_2$ to $a_3$, which is the degree-$2$/degree-$4$ gap of \cref{fact:gap}.
\end{itemize}
Throughout, ``reach'' is the \emph{one-dimensional, per-direction} certificate reach for a fixed projection $U=w\cdot x$; whether this scalar law lifts to the $d$-dimensional reweighting program over all directions jointly is open for $t\ge2$ (\cref{rem:soslb,sec:open}).
\end{remark}

\subsection{The degree-2 outlier barrier}

The degree-$2$/degree-$4$ gap (\cref{fact:gap}) is a statement about the marginal; here it is forced on a learner. At the $\eta$-tail center the certificate ceiling is $\lambda_2(c)=\Theta(\eta)$ at degree $2$ and $\lambda_3(c)=\Theta(\eta^2)$ at degree $4$ (\cref{fact:gap}, and \cref{ex:hiding} for the moment audit), so any degree-$2$ outlier program must leave the spike that a degree-$4$ program removes. The following instance makes the gap a concrete rate separation: degree $2$ cannot beat $\eta^{1/2}$ on it, while degree $4$ does. Where the tradeoff was about which bands a certificate can see, this is about which outliers a low-degree program can remove.

The relevant quantity is the dirty contribution to the hinge gradient,
\[
  \linsumnorm(q\circ\Sdirty)=\sup_{\norm{w}\le1}\sum_{i\in\Sdirty}q_i\abs{w\cdot x_i},
\]
where $q\in[0,1]^n$ are the outlier-removal weights and $\Sdirty$ is the corrupted set. The barrier is \emph{label-oblivious in expectation}: it holds for any reweighting computed without ground-truth labels, on average over which cluster points the adversary corrupts. The literal ``for all feasible $q$'' fails once the removal budget $\xi$ reaches $\eta$ (the polytope then contains the weight that zeros the spike); we state the honest in-expectation form.

\begin{restatable}[Degree-2 outlier barrier; label-oblivious]{theorem}{thmdegbarrier}
\label{thm:degbarrier}
Let $\eta\ge\Omega(1/d)$. There is a clean log-concave mixture with covariance $\Sigma\preceq\sigma^2 I$ and an $\eta$-fraction malicious set $\Sdirty$ such that every \emph{label-oblivious degree-2 (variance) reweighting}---any $q\in[0,1]^n$ with $\sum_i q_i\ge(1-\xi)n$ feasible for the directional-variance budget $\frac1n\sum_i q_i(w\cdot x_i)^2\le\bar\sigma^2$ for all $w$ (here $\bar\sigma^2=\Theta(\sigma^2)$ is the \emph{empirical} directional variance of the corrupted sample, which a degree-2 program cannot certify below without detecting the spike), produced by a procedure not given the labels---leaves, in expectation over the adversary's choice of which cluster points are corrupted,
\begin{align}\label{eq:degbarrier}
  \E\big[\linsumnorm(q\circ\Sdirty)\big]
  \;=\;\Omega\big(\eta^{1/2}\,n\,\bar\sigma\big).
\end{align}
In contrast, the degree-$2t$ program of \cref{thm:algorithm} achieves $\linsumnorm(q\circ\Sdirty)/n\le C\bar\sigma\,\eta^{1-1/2t}$; in particular degree $4$ ($t=2$) reaches $\eta^{3/4}<\eta^{1/2}$. Degree-$2$ outlier removal is stuck at the $\eta^{1/2}$ rate, i.e.\ $\eta_0\lesssim(\gamma\rho/\bar\sigma)^2$.
\end{restatable}

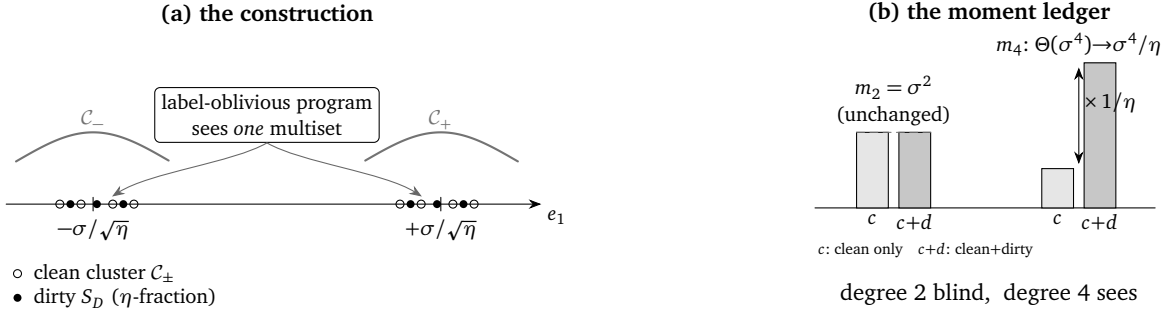
\begin{figure}[t]
\centering
\begin{tikzpicture}[>=Stealth,line cap=round,font=\scriptsize]
\begin{scope}[shift={(0,0)}]
  \def\axL{-3.45}
  \def\axR{3.45}
  \def\a{2.30}
  \def\bh{0.95}
  \def\bw{1.00}

  \node[anchor=south,font=\footnotesize] at (0,2.25) {\textbf{(a) the construction}};

  \draw[black!55,thick]
    plot[smooth,domain=-3.32:-1.30,samples=50] (\x,{\bh*exp(-((\x+\a)*(\x+\a))/(2*\bw*\bw))});
  \draw[black!55,thick]
    plot[smooth,domain=1.30:3.32,samples=50]  (\x,{\bh*exp(-((\x-\a)*(\x-\a))/(2*\bw*\bw))});
  \node[black!55,anchor=south,inner sep=1pt] at (-\a,\bh+0.00) {$\mathcal{C}_-$};
  \node[black!55,anchor=south,inner sep=1pt] at ( \a,\bh+0.00) {$\mathcal{C}_+$};

  \draw[->,black] (\axL,0) -- (\axR+0.18,0) node[below,anchor=north west,inner sep=2pt] {$e_1$};

  \draw[black] (-\a,-0.09) -- (-\a,0.09);
  \draw[black] ( \a,-0.09) -- ( \a,0.09);
  \node[anchor=north,inner sep=2pt] at (-\a,-0.12) {$-\sigma/\sqrt{\eta}$};
  \node[anchor=north,inner sep=2pt] at ( \a,-0.12) {$+\sigma/\sqrt{\eta}$};

  \foreach \x in {-2.74,-2.46,-2.04,-1.76}{\draw[black] (\x,0) circle (1.5pt);}
  \foreach \x in {-2.60,-2.25,-1.90}{\fill[black] (\x,0) circle (1.5pt);}
  \foreach \x in {1.76,2.04,2.46,2.74}{\draw[black] (\x,0) circle (1.5pt);}
  \foreach \x in {1.90,2.25,2.60}{\fill[black] (\x,0) circle (1.5pt);}

  \node[draw=black,fill=white,rounded corners=2pt,align=center,inner sep=2.5pt]
        (ms) at (0,1.18)
        {label-oblivious program\\sees \emph{one} multiset};
  \draw[->,black!60] (ms.south) .. controls (0.45,0.55) and (1.55,0.42) .. (2.05,0.10);
  \draw[->,black!60] (ms.south) .. controls (-0.45,0.55) and (-1.55,0.42) .. (-2.05,0.10);

  \draw[black] (-3.30,-0.92) circle (1.5pt);
  \node[anchor=west,inner sep=2.5pt] at (-3.18,-0.92) {clean cluster $\mathcal{C}_\pm$};
  \fill[black] (-3.30,-1.24) circle (1.5pt);
  \node[anchor=west,inner sep=2.5pt] at (-3.18,-1.24) {dirty $\Sdirty$ ($\eta$-fraction)};
\end{scope}

\begin{scope}[shift={(7.7,0)}]
  \def\barw{0.42}
  \def\g{0.14}
  \def\base{-0.05}

  \node[anchor=south,font=\footnotesize] at (1.85,2.25) {\textbf{(b) the moment ledger}};
  \draw[black] (-0.45,\base) -- (4.05,\base);

  \def\mtwo{1.00}
  \def\pxA{0.10}
  \fill[black!10] (\pxA,\base) rectangle (\pxA+\barw,\base+\mtwo);
  \draw[black]    (\pxA,\base) rectangle (\pxA+\barw,\base+\mtwo);
  \fill[black!22] (\pxA+\barw+\g,\base) rectangle (\pxA+2*\barw+\g,\base+\mtwo);
  \draw[black]    (\pxA+\barw+\g,\base) rectangle (\pxA+2*\barw+\g,\base+\mtwo);
  \draw[black!55,densely dashed] (\pxA,\base+\mtwo) -- (\pxA+2*\barw+\g,\base+\mtwo);
  \node[anchor=south,align=center,inner sep=1pt] at (\pxA+\barw+\g/2,\base+\mtwo+0.05)
        {$m_2=\sigma^2$\\(unchanged)};
  \node[anchor=north,inner sep=1.5pt] at (\pxA+\barw/2,\base-0.02) {$c$};
  \node[anchor=north,inner sep=1.5pt] at (\pxA+1.5*\barw+\g,\base-0.02) {$c{+}d$};

  \def\mfA{0.52}
  \def\mfB{1.92}
  \def\pxB{2.55}
  \fill[black!10] (\pxB,\base) rectangle (\pxB+\barw,\base+\mfA);
  \draw[black]    (\pxB,\base) rectangle (\pxB+\barw,\base+\mfA);
  \fill[black!22] (\pxB+\barw+\g,\base) rectangle (\pxB+2*\barw+\g,\base+\mfB);
  \draw[black]    (\pxB+\barw+\g,\base) rectangle (\pxB+2*\barw+\g,\base+\mfB);
  \draw[<->,black] (\pxB+\barw+\g/2,\base+\mfA+0.06) -- (\pxB+\barw+\g/2,\base+\mfB-0.06);
  \node[anchor=south west,inner sep=1pt] at (\pxB+\barw+\g/2+0.02,{\base+(\mfA+\mfB)/2}) {$\times\,1/\eta$};
  \node[anchor=south,align=center,inner sep=1pt] at (\pxB+\barw+\g/2,\base+\mfB+0.05)
        {$m_4$: $\Theta(\sigma^4){\to}\sigma^4/\eta$};
  \node[anchor=north,inner sep=1.5pt] at (\pxB+\barw/2,\base-0.02) {$c$};
  \node[anchor=north,inner sep=1.5pt] at (\pxB+1.5*\barw+\g,\base-0.02) {$c{+}d$};

  \node[anchor=north west,align=left,inner sep=1pt,font=\tiny] at (-0.45,\base-0.42)
        {$c$: clean only\quad $c{+}d$: clean$+$dirty};
  \node[anchor=north,font=\footnotesize] at (1.85,-0.92)
        {degree $2$ blind,\ \ degree $4$ sees};
\end{scope}

\end{tikzpicture}
\caption{The \emph{moment-cloaked spike} behind \cref{thm:degbarrier}.
(a) Two genuine log-concave clusters $\mathcal{C}_\pm$ at $\pm(\sigma/\sqrt{\eta})e_1$; the adversary's dirty points (filled) are drawn from the \emph{same} clusters and sit at the clean points' (hollow) locations, so a label-oblivious variance program sees one indistinguishable multiset. (b) The dirty mass spends exactly the variance budget, so the second moment $m_2=\sigma^2$ is identical with or without it (degree $2$ is blind), while the fourth moment is inflated by $1/\eta$, the tell a degree-$4$ certificate sees.}
\label{fig:cloak}
\end{figure}

\begin{proof}[Proof sketch]
The construction is a \emph{moment-cloaked spike} (\cref{fig:cloak}): two genuine log-concave clusters $\mathcal C_\pm$ at means $\pm(\sigma/\sqrt\eta)e_1$ (\cref{lem:cloak}; using two unimodal bumps rather than one $\pm a$ atom keeps each component log-concave, at the cost of the regime $\eta\ge\Omega(1/d)$). The adversary adds $\eta n$ dirty points drawn from the same $\mathcal C_\pm$, so the sample carries $\approx 3\eta n$ points near $\pm a e_1$ with $a=\sigma/\sqrt\eta$. The dirty points are engineered to be \emph{second-moment invisible} (they spend exactly the variance budget, $\frac1n\sum_{\Sdirty}(e_1\cdot x_i)^2=\eta a^2=\sigma^2$) but \emph{first-moment aligned}, $\sum_{\Sdirty}\abs{e_1\cdot x_i}=\eta n\,a=\sqrt\eta\,n\sigma$. Because the variance constraints depend only on the multiset of points, a label-oblivious program cannot tell the $\eta n$ dirty points from the $2\eta n$ clean points of $\mathcal C_\pm$ at the same locations; removing at most $\xi n$ points total, it must keep at least a $\tfrac13$ share of the combined cluster in expectation, hence $\E[\linsumnorm]\ge\tfrac13\eta n\,a=\tfrac13\sqrt\eta\,n\sigma$. The same $\Sdirty$ has an inflated fourth moment, $\frac1n\sum_{\Sdirty}(e_1\cdot x_i)^4=\bar\sigma^4/\eta \gg\bar\sigma^4$, which a degree-$4$ certificate sees and drives down to the $\eta^{3/4}$ rate, so the separation is exactly one of moment degree. Details are in \cref{app:proofs}.
\end{proof}

\begin{remark}[Label-oblivious and in expectation, not for every $q$]\label{rem:status-degbarrier}
The cloak construction rests on \cref{lem:cloak} (two log-concave bumps realizing the moment profile), stated and discharged in \cref{app:proofs} modulo the standard fact that a unimodal scale-$\sigma$ bump is log-concave; the regime $\eta\ge\Omega(1/d)$ enters there. The literal universal-$q$ statement is false for budget $\xi\ge\eta$, as noted above; the in-expectation form \eqref{eq:degbarrier} is what we prove and what matches \cref{thm:algorithm} at $t=1$.
\end{remark}

\subsection{The degree-$2t$ reweighted hinge}

The algorithm is the degree-of-robustness law of \cref{sec:perspective-bridge} run forward. Its outlier-removal step \emph{is} the order-$(t{+}1)$ empirical Christoffel detector of \citet{LasserrePauwels19} run as a certificate rather than a threshold, by the bridge raising the detector order is the only knob, and the two barriers fix exactly which rates it can and cannot reach. \Cref{alg:main} reweights the sample with this degree-$2t$ Sum-of-Squares program and then minimizes the reweighted hinge loss. The single parameter is the degree $t$. At $t=1$ the degree-$2t$ constraint is exactly the empirical directional variance, and \cref{alg:main} coincides with the reweighted-hinge method of \citet{Shen25}: \emph{their algorithm is the degree-$2$ ($t=1$) special case of ours}, and raising $t$ is the only change: the $d\times d$ Gram matrix of their degree-$2$ program becomes the $\binom{d+t-1}{t}\times\binom{d+t-1}{t}$ moment matrix of the degree-$2t$ program (\cref{app:sos-sdp}), and nothing else in the method moves. The gain is \emph{modest and explicit}: degree $2t$ removes the degree-$2$ exponent loss in the noise tolerance ($c^2\mapsto c$) at the price of the sub-exponential prefactor $(Ct)$; it does not march to $1/2$ (the tolerance is capped by the pancake density $\rho/4$). The value of the result is its \emph{tightness}: \cref{thm:degbarrier} shows the degree-$2$ corner cannot be improved within the method.

\begin{algorithm}[t]
\caption{Degree-$2t$ robust reweighted hinge}
\label{alg:main}
\begin{algorithmic}[1]
  \Require sample $S=\{(x_i,y_i)\}_{i=1}^n$; margin $\gamma$; degree $t\ge1$; variance bound
           $\bar\sigma^2$; removal budget $\xi=\Theta(\eta)$
  \Ensure unit halfspace $\hat w$
  \State \textbf{Prune.} Discard each $(x_i,y_i)$ whose $\norm{x_i}$ exceeds the clean-sample
         radius $R_{\mathrm{cl}}=C\bar\sigma\sqrt d$.
         \Comment{removes gross-norm corruptions; every clean $\norm{x}\le R_{\mathrm{cl}}$ w.h.p.\ (log-concave concentration)}
  \State \textbf{Reweight} (degree-$2t$ soft outlier removal). Find $q\in[0,1]^n$ with
         $\sum_i q_i\ge(1-\xi)n$ admitting a degree-$2t$ SoS certificate, in the indeterminate
         $w$, of
  \Statex \hfill$\displaystyle\rule{0pt}{3.6ex}
         \tfrac1n\sum_{i} q_i\,(w\cdot x_i)^{2t}\ \le\ (Ct)^{2t}\,\bar\sigma^{2t}
         \qquad\text{for all } w$\hfill\null
  \Statex \hspace{\algorithmicindent}\Comment{a semidefinite program of size $d^{O(t)}$
         (\cref{app:sos-sdp}); feasible by \cref{lem:hyper}. At $t=1$: the variance LP of \citet{Shen25}.}
  \State \textbf{Reweighted hinge.} $\displaystyle\hat w\gets\arg\min_{\norm{w}\le 1/\gamma}\
         \sum_{i} q_i\,\hinge(w;x_i,y_i)$. \Comment{convex program}
  \State \textbf{return} $\hat w/\norm{\hat w}$.
\end{algorithmic}
\end{algorithm}

\begin{restatable}[Degree-$2t$ robust hinge]{theorem}{thmalgorithm}
\label{thm:algorithm}
Run \cref{alg:main} with degree $t\ge1$. Its degree-$2t$ outlier-removal program is feasible (the clean-indicator $q$ is certified by \cref{lem:hyper}), and any feasible $q$ satisfies the Hölder bound
\begin{align}\label{eq:alg-lsn}
  \linsumnorm(q\circ\Sdirty)
  \;=\;\sup_{\norm{w}\le1}\sum_{i\in\Sdirty}q_i\abs{w\cdot x_i}
  \;\le\; n\,(Ct)\,\bar\sigma\,\eta^{1-\frac1{2t}}.
\end{align}
Consequently, under the $\gamma$-margin log-concave mixture of \cref{sec:prelim}, the output $\hat w$ has inlier error $\err_{\Dx}(\hat w)\le\eps$ for every malicious rate
\begin{align}\label{eq:alg-eta0}
  \eta \;\le\; \eta_0(t)
  \;=\;\Omega\Big(\min\Big\{\tfrac{\rho}{4},\,
    \big(\gamma\rho/((Ct)\bar\sigma)\big)^{\frac{2t}{2t-1}}\Big\}\Big),
\end{align}
in time $n^{O(t)}$ with sample size $n=d^{O(t)}\poly(1/\eps)$. At $t=1$ this is exactly \citeauthor{Shen25}'s $\eta^{1/2}$ rate; as $t\to\infty$ the exponent in \eqref{eq:alg-lsn} tends to $1$.
\end{restatable}

\begin{proof}[Proof sketch]
The bound \eqref{eq:alg-lsn} is one application of Hölder with exponent $2t$: $\sum_{\Sdirty}q_i\abs{w\cdot x_i}\le(\sum_{\Sdirty}q_i)^{1-\frac1{2t}} (\sum_{\Sdirty}q_i\abs{w\cdot x_i}^{2t})^{\frac1{2t}}$, then $\sum_{\Sdirty}q_i\le\eta n$ and, since $\abs{w\cdot x}^{2t}=(w\cdot x)^{2t}$ is an even polynomial and the sum over $\Sdirty$ is dominated by the sum over all of $S$, the certified bound $\frac1n\sum_S q_i(w\cdot x_i)^{2t}\le(Ct)^{2t}\bar\sigma^{2t}$ caps the second factor. Feasibility is \cref{lem:hyper} applied to the empirical clean measure. The noise tolerance \eqref{eq:alg-eta0} comes from feeding \eqref{eq:alg-lsn} into the robustness condition of \citet{Shen25}: the clean hinge gradient along the empirical optimum $\hat w$ is lower-bounded by the dense-pancake density $\rho$, and requiring $\frac\gamma4(\rho-2\xi)>\linsumnorm(q\circ\Sdirty)/n$ solves to $\eta_0(t)$. The full assembly is in \cref{app:proofs}.
\end{proof}

\begin{remark}[The gain is in the exponent, not a march to $1/2$]\label{rem:status-algorithm}
The gain is in the \emph{exponent} of \eqref{eq:alg-eta0}: $\tfrac{2t}{2t-1}\to1$ removes the degree-$2$ squaring loss, so $\eta_0$ rises from \citeauthor{Shen25}'s tiny constant to a larger (still constant) value, bounded by $\rho/4$. It does \emph{not} approach $1/2$. The certified hypercontractivity constant is the sub-exponential $(Ct)^{2t}$, not $(Ct)^t$; the latter is false for general log-concave marginals (Laplace: $m_{2t}/m_2^t=(\Theta(t))^{2t}$), which is why the prefactor in \eqref{eq:alg-lsn} is $(Ct)$ and not $(Ct)^{1/2}$. The robustness step cites the \citet{Shen25} framework as an external black box (its Theorem~9 / Lemma~12), with the degree-$2t$ $\linsumnorm$-bound plugged in; we do not reprove it.
\end{remark}

\subsection{The information-theoretic floor}

The algorithm's tolerance is bounded above by the barriers; its accuracy is bounded below by information theory. The next proposition is a breakdown floor that no algorithm, of any computational power, can beat. It matches the $\Theta(\eta)$ dependence of \cref{thm:algorithm} in rate, but the two live on disjoint sets of instances: the instances realizing the floor necessarily violate the dense-pancake hypothesis \cref{thm:algorithm} needs, so the two results bracket the achievable region from opposite sides rather than meeting on one instance (\cref{sec:related}).

% claim: F02a
\begin{restatable}[Breakdown floor]{proposition}{propfloor}
\label{prop:floor}
Let $d\ge2$ and $\eta\in(0,\tfrac12)$, and let $\mathcal C=\mathcal C(d,\sigma,r,\gamma)$ be the $\gamma$-margin log-concave mixture class of \cref{sec:prelim} with the number of components $K$ \emph{free}. Assume
\begin{align}\label{eq:floor-V}
  \textup{(V)}\qquad
  \sqrt2\,(\gamma+\rho_b)\;\le\;r\quad\text{for some }0<\rho_b\le\sigma\sqrt{d+2}
  \qquad(\text{e.g.\ }\rho_b=\sigma,\ \text{i.e.\ }\gamma\le r/\sqrt2-\sigma).
\end{align}
Then, in the per-example malicious model of \citet{KearnsLi93} at rate $\eta$ (each example independently replaced, with probability $\eta$, by an arbitrary pair chosen with full knowledge of the learner, the instance and the clean sample), no algorithm can output $\hat w$ with inlier error below
\begin{align}\label{eq:floor}
  \err_{\Dx}(\hat w) \;\ge\; \frac{\eta}{2(1-\eta)} \;\ge\; \frac\eta2 .
\end{align}
Precisely: $\inf_{\mathcal A}\sup_{P\in\mathcal C,\ \mathrm{adv}}\E\big[\err_P(\hat w)\big]\ \ge\ \eta/(2(1-\eta))$, where $\err_P$ is the inlier error of instance $P$, the supremum ranges over instances \emph{and} $\eta$-malicious adversaries, and the expectation is over the corrupted sample and the algorithm's internal randomness; at $K=\ceil{(1-\eta)/\eta}$ there is moreover an instance with $\Pr[\err_P(\hat w)\ge\eta/2]\ge\tfrac12$. The floor value depends on none of $\gamma,\sigma,d,r$.
\end{restatable}

\begin{proof}[Proof sketch]
Le Cam two points, with the disagreement mass carried by mixture \emph{weights} rather than by geometry. Set $R=\sqrt2(\gamma+\rho_b)\le r$ and place two balls of radius $\rho_b$ at distance $R$ from the origin: an \emph{anchor} centred at $\tfrac R{\sqrt2}(e_1+e_2)$ and a \emph{wedge} centred at $\tfrac R{\sqrt2}(e_1-e_2)$. Give the wedge $m$ of the $K$ equal mixture components and the anchor the other $K-m$; both instances share this marginal, and the wedge carries weight $q=m/K$. Label $P_1$ by $w_1^\star=e_1$ and $P_2$ by $w_2^\star=e_2$. Under (V) both labelings have margin $\gamma$: on the anchor both call every point $+1$, and they disagree exactly on the wedge, so $\TV(P_1,P_2)=q$; the construction covers every $\eta\in(0,\tfrac12)$, every margin within the factor-$\sqrt2$ ceiling the class admits, and sits inside the regime $r=O(\sigma\sqrt d)$ of \cref{sec:prelim} (take $\rho_b=\sigma$). Choosing $m=\floor{K\eta/(1-\eta)}$ puts $q$ at the Kearns--Li malicious envelope $q\le\eta/(1-\eta)$ (\cref{fact:kl}), and the adversary realizing it is explicit: each instance flips the labels of a $(1-\eta)q\le\eta$ fraction of its examples, all inside the wedge, so under both instances the wedge shows balanced $\pm$ labels and the two corrupted sample laws coincide exactly, at every $n$. Since the two labelings differ on all of the wedge, every $\hat w$ obeys $\err_{P_1}(\hat w)+\err_{P_2}(\hat w)\ge q$ pointwise, so one of the two instances charges at least $q/2$; the per-$K$ bound holds simultaneously for every admissible $K$, and $\sup_K q=\eta/(1-\eta)$, which is the constant in \eqref{eq:floor}. That constant is the exact value of the two-point game and not merely a lower bound for it: a learner that outputs $w_1^\star$ or $w_2^\star$ with probability $\tfrac12$ has expected error $\TV/2\le\eta/(2(1-\eta))$ on both instances. Details in \cref{app:floor}.
\end{proof}

% claim: F02a
\begin{remark}[Rate-tight, not threshold-tight; and what the construction cannot avoid]\label{rem:status-floor}
Four readings of \cref{prop:floor}. \emph{(i) Rate, not threshold, and why.} It matches \cref{thm:algorithm} on the \emph{rate} $\Theta(\eta)$, not on the breakdown \emph{threshold}, and the obstruction is the method rather than this instance: the Kearns--Li envelope caps \emph{every} two-point pair at $\TV\le\eta/(1-\eta)$ (\cref{fact:kl}), so no pair of instances can force expected error past $\eta/(2(1-\eta))$, the method ceiling computed in \cref{app:floor}. Reaching the conjectured $\gamma,\sigma$-dependent breakdown $\eta^\star=\Theta(\gamma/\sigma)$ therefore needs a multi-point or SQ band-mass argument, not a sharper two-point pair (\cref{sec:open}).
\emph{(ii) The two extravagances are forced.} A hard margin quantizes disagreement: each component's support is convex, so it lies wholly on one side of every margin-respecting hyperplane, and the disagreement mass between two margin-respecting labelings is therefore an integer multiple of $1/K$. An $\eta$-indistinguishable pair needs that mass to be positive and at most $\eta/(1-\eta)$, forcing $K\ge(1-\eta)/\eta$; a pair with both an agreeing and a disagreeing component forces $r\ge\sqrt2\,\gamma$ (hard-margin quantization, \cref{app:floor}). The $K=\Theta(1/\eta)$ blow-up and condition (V) are thus necessity statements, not conveniences. The same quantization is why $d\ge2$ cannot be dropped: at $d=1$ the disagreement mass is $0$ or $1$, and majority vote drives the error to $0$.
\emph{(iii) Tightness from above.} \citet{Blanc26} attains $\tfrac12\cdot\eta/(1-\eta)$ distribution-free with randomized hypotheses, so the constant in \eqref{eq:floor} is matched from above \emph{exactly}. What \cref{prop:floor} adds is that the constant \emph{survives} the restriction to $\gamma$-margin log-concave mixtures: the classical lower bound of \citet{KearnsLi93} attains it on an adversarially chosen four-point marginal and says nothing about a fixed nice class.
\emph{(iv) The fixed-fraction reading.} \Cref{prop:floor} is stated in the per-example model, in which the two corrupted sample laws coincide exactly at every $n$. Under the fixed-fraction reading of \cref{sec:prelim} the same construction gives, for every rate slack $\eps\in(0,\eta)$ and every $n$ with $n\eps\ge1$, the floor $\frac{\eta-\eps}{2(1-\eta+\eps)}-e^{-2n(\eps-1/n)^2}$, recovering the slack-free constant as $n\to\infty$; for $n\eps<1$ nothing is asserted (\cref{app:floor}).
\end{remark}

\subsection{Supporting lemmas}

We collect the four lemmas the theorems rest on. Each is stated here and proved in \cref{app:proofs}.

\begin{restatable}[Certifiable hypercontractivity]{lemma}{lemhyper}
\label{lem:hyper}
Let $\Dx=\frac1K\sum_{j=1}^K\Dcal_j$ on $\R^d$, each $\Dcal_j$ log-concave with mean $\mu_j$ ($\norm{\mu_j}\le r$) and $\Sigma_j\preceq\sigma^2 I$, in the regime $r=O(\sigma\sqrt d)$; write $\bar\mu=\frac1K\sum_j\mu_j$ and $\bar\sigma^2(w)=\E_{\Dx}[(w\cdot(x-\bar\mu))^2]$. Then for every $t\ge1$ there is a degree-$2t$ SoS proof in $w$ of
\begin{align}\label{eq:hyper}
  \E_{x\sim\Dx}\big[(w\cdot(x-\bar\mu))^{2t}\big]
  \;\le\;(Ct)^{2t}\,\big(\bar\sigma^2(w)\big)^{t},
\end{align}
with $C$ absolute. The constant is the sub-exponential $(Ct)^{2t}$; the sub-Gaussian $(Ct)^t$ holds only under added sub-Gaussianity of the components.
\end{restatable}

\begin{restatable}[Minorant reduction and Gauss quadrature]{lemma}{lemgauss}
\label{lem:gauss}
Let $\mu$ be a measure with all moments and infinite support, $m_k=\int u^k\,d\mu$. \textup{(i)} A degree-$2t$ moment-only certificate of $\Pr_\mu[U\in B]\ge\rho$ is equivalent to a polynomial $p$, $\deg p\le 2t$, with $p\le\mathbf 1_B$ pointwise, together with a degree-$2t$ SoS proof from the moment axioms $\{\int u^k\,d\nu=m_k\}_{k\le2t}$ that $\int p\,d\nu\ge\rho$. \textup{(ii)} The $(t{+}1)$-node Gauss measure $\nu_{\mathrm{GQ}}=\sum_i\lambda_i\delta_{u_i}$ (nodes the zeros of the orthonormal polynomial $\pi_{t+1}$, Christoffel weights $\lambda_i>0$) is nonnegative, atomic on the $t{+}1$ nodes, and exact on degrees $\le 2t+1$: $\int g\,d\nu_{\mathrm{GQ}}=\int g\,d\mu$.
\end{restatable}

\begin{restatable}[Largest-root bound]{lemma}{lemroot}
\label{lem:root}
For centered log-concave $\mu$ at scale $s$, the largest absolute zero $R_{t+1}=\max_i\abs{u_i}$ of the orthonormal polynomial $\pi_{t+1}$ satisfies: $R_{t+1}\le 2s\sqrt{t+1}$ for a sub-Gaussian tail ($e^{-\abs{u/s}^2}$, exact for Hermite), and $R_{t+1}=O(s\,t)$ for a sub-exponential tail ($e^{-\abs{u/s}}$, Freud exponent $\alpha=1$). More generally, for a Freud exponent $\alpha\in[1,2]$, $R_{t+1}\asymp s\,t^{1/\alpha}$.
\end{restatable}

\begin{restatable}[Log-concave quantiles and margin geometry]{lemma}{lemquant}
\label{lem:quant}
For centered log-concave $\mu$ at scale $s$, the $\beta$-quantile satisfies $\abs{c^\star}=\Theta(s\sqrt{\log(1/\beta)})$ in the sub-Gaussian case and $\Theta(s\log(1/\beta))$ in the sub-exponential case. Under near-isotropy ($\norm{w}\le1/\gamma$, $s=\Theta(1)$), a standardized half-width $\tau=\Theta(s)$ corresponds to physical margin $\gamma=\Theta(\tau\sqrt d)$.
\end{restatable}

The cloak construction behind \cref{thm:degbarrier} is also deferred; we state it as a lemma for the appendix.

\begin{restatable}[Two-cluster moment cloak]{lemma}{lemcloak}
\label{lem:cloak}
For $\eta\ge\Omega(1/d)$ there exist two log-concave components $\mathcal C_\pm$ on $\R^d$ with $\Sigma_{\mathcal C_\pm}\preceq\sigma^2 I$ and means $\pm(\sigma/\sqrt\eta)e_1$ (so $\norm{\mu_\pm}=\sigma/\sqrt\eta\le r$ fits the regime $r=O(\sigma\sqrt d)$ exactly when $\eta\ge\Omega(1/d)$) such that the combined $3\eta n$-point cluster (clean $2\eta n$ plus dirty $\eta n$, both drawn from $\mathcal C_\pm$) spends exactly the $e_1$-variance budget $\sigma^2$ and is invariant under permuting dirty and clean cluster points.
\end{restatable}

\section{Discussion}
\label{sec:discussion}

\subsection{Context and comparisons}
\label{sec:related}

\paragraph{The reweighted-hinge cluster, and the two limitations we explain.}
Hinge-loss minimization is robust because the loss is Lipschitz: a corrupted point can tilt the gradient by only so much. \citet{Talwar20} made this quantitative through the dense-pancake condition and a first-order (KKT) analysis, tolerating a malicious fraction $\eta=\Omega(\gamma)$ at margin $\gamma=\Omega(\log d/\sqrt d)$. \citet{Shen25} sharpened the constant to $\eta_0$ a fixed absolute constant by \emph{reweighting} the hinge: an outlier-removal program assigns weights $q\in[0,1]^n$ and the only distributional fact it uses about the clean marginal is a \emph{degree-$2$} (empirical-variance) bound on the reweighted directional moment, through the controlling quantity $\linsumnorm(q\circ\Sdirty)$. \citet{ZengShen25} carry the same degree-$2$ reweighting program into the attribute-efficient setting, inheriting the same two features, and those two features are exactly what we name. The multiclass work of \citet{AdhikariZeng26} instead removes outliers by cluster-based pruning with a \emph{standard} multiclass hinge rather than by reweighting, so it lies outside the reweighted-hinge program our barriers target. \citeauthor{Shen25} states only that the breakdown rate ``cannot be made close to the optimal breakdown point of $\tfrac12$ due to inherent limitations of our approach,'' without identifying the limitation, and that the margin has a ``logarithmic dependence on $1/\eps$ that appears less natural,'' which he conjectures ``can be improved to $\gamma=\Omega(1/\sqrt d)$'' through an active-learning subroutine. We make the limitation precise: it is the \emph{degree} of the certificate. \Cref{thm:degbarrier} gives a clean log-concave instance on which degree $2$ is \emph{provably} stuck at the $\eta^{1/2}$ rate that forces \citeauthor{Shen25}'s small constant, while degree $4$ already reaches $\eta^{3/4}$; the constant is small because the exponent is $\tfrac12$, a property of degree $2$, and raising the degree lifts it toward the pancake cap (\cref{thm:algorithm}, $\eta^{1-1/2t}$, not a march to $1/2$). The margin he calls less natural is, we show, \emph{forced}: \cref{thm:tradeoff} proves the $\log(1/\eps)$ is conserved by every one-shot certificate of fixed degree, so it is the intrinsic price of the degree, not an artifact of his analysis, removable only by raising the degree (in $n^{\Omega(\log 1/\eps)}$ time), enlarging the margin, or leaving the one-shot model by iterating, the very active-learning route he conjectures (\cref{sec:open}). The companion \cref{thm:algorithm} raises the program to degree $2t$ and recovers the matching upper side $\eta^{1-1/2t}$, so $t=1$ (\citeauthor{Shen25}) is the exact degree-$2$ corner of a frontier rather than an isolated bound. The floor of \cref{prop:floor} brackets that frontier from the other side, though not on the same instances: its two-cluster construction concentrates the marginal in two balls and therefore necessarily violates the dense-pancake hypothesis \cref{thm:algorithm} needs, so the two results meet in \emph{rate} ($\Theta(\eta)$) while their hypotheses hold on disjoint sets of instances: a bracket from opposite sides, not a matched pair on one instance.

\paragraph{The machinery we import.}
Degree raising is borrowed wholesale from the sum-of-squares (SoS) literature on robust statistics. The mechanism that converts a degree-$2t$ certified-moment bound into an outlier-removal guarantee at rate $\eta^{1-1/2t}$ is the proofs-to-algorithms paradigm of \citet{HopkinsLi18} and \citet{KothariSteinhardtSteurer18}, and the property that makes it apply (certifiable hypercontractivity, an SoS proof that $\E[(w\cdot x)^{2t}]$ is bounded by a power of $\E[(w\cdot x)^2]$) is due to \citet{KlivansKothariMeka18}, with the general subgaussian case settled by \citet{DiakonikolasHopkinsPensiaTiegel24}. For us the certifiable constant is $(Ct)^{2t}$, the sub-exponential rate of a general log-concave marginal (a single Laplace coordinate already forces $m_{2t}/m_2^t=(\Theta(t))^{2t}$, so the often-quoted $(Ct)^{t}$ is false at this generality); this is what couples the achievable $\eta$ to the SoS degree in \cref{thm:algorithm}.

\paragraph{The dual direction we provably do not have.}
The same SoS toolkit can certify \emph{anti}-concentration: \citet{BakshiKothariRTV24} give quasi-polynomial SoS certificates that a slab carries \emph{at most} a prescribed mass, beyond the Gaussian case. Our dense-pancake premise needs the opposite (a degree-$2t$ certificate that an off-center band carries \emph{at least} $\rho$ mass), and \cref{thm:tradeoff} is precisely the statement that no such \emph{pro}-concentration certificate exists below degree $\Omega((|c|/s)^2)$. This is the cone asymmetry of \cref{sec:perspective-moment} read at the learning scale: the two faces of $\mathcal M_{2t}(\mu)$ are certifiable anti-concentration (the upper envelope) and the pro-concentration we need (the lower envelope), and only the former is low-degree. The conserved $\log(1/\eps)$ is that asymmetry's price.

\paragraph{What does not subsume our result.}
\citet{KlivansSTV25} achieve a $2\eta+\eps$ guarantee for learning halfspaces under heavy contamination by iterative filtering, an $O(1)$ noise rate that may look like it dominates our modest constant. It does not reach our setting on four counts: it is Gaussian, not general log-concave mixture; it assumes \emph{no} margin and so has no $\log(1/\eps)$ to conserve; it is a filtering algorithm, not a moment/SoS certificate, so the degree-cost question we study does not arise for it; and it targets agnostic (boundary-volume) error rather than the clean-inlier error $\err_{\Dx}(\hat w)$ of the malicious model. We cite it as a reference point for what \emph{is} achievable absent these constraints, not as a result that absorbs the reweighted-hinge line. The closest prior \emph{method}-specific barrier to \cref{thm:tradeoff} is \citet{LongServedio11}, who show that learners minimizing a convex proxy for the $0/1$ loss cannot tolerate malicious noise beyond $\Theta(\eps\gamma)$ when learning $\gamma$-margin halfspaces to error $\eps$. The scoped class is what differs: theirs is defined by the \emph{objective} (convex-proxy minimizers), ours by the \emph{distributional input} (bounded-degree moment certificates), and our barrier is paired with an upper side at the same exponent (\cref{thm:algorithm}). The agnostic and nasty-noise halfspace tradition, \citet{DiakonikolasKaneStewart18} via iterative spectral filtering and the near-optimal SQ lower bounds of \citet{DiakonikolasKaneZarifis20}, lives in a genuinely different model (OPT$+\eps$ error, not the margin-plus-malicious model here) and provides our methodological foil below.

\paragraph{Method-specific barriers, and the SQ contrast.}
Our lower bounds rule out a specific class of methods. \cref{thm:tradeoff,thm:degbarrier} are barriers against the \emph{moment/SoS-certificate method}, any learner whose only distributional input is the degree-$\le 2t$ certified moments of the clean marginal. They are not information-theoretic, and they are not statistical-query hardness. This scoping is what turns \citeauthor{Shen25}'s informal ``limitations of our approach'' into a theorem about the approach, while leaving open that a procedure outside the model does better. The contrast with the unconditional/SQ tradition is thus between two kinds of impossibility: a SQ lower bound \citep{DiakonikolasKaneZarifis20} says no efficient query algorithm of a broad class succeeds, whereas our degree barrier says no certificate of a bounded resource succeeds, and names the resource (SoS degree) whose increase provably buys the way out (traced by \cref{thm:algorithm}). The barrier and the algorithm meet at the same exponent, which is why we state them together.

\subsection{Open problems}
\label{sec:open}

\begin{itemize}
\item \textbf{The tight breakdown $\eta^\star=\Theta(\gamma/\sigma)$.} \cref{prop:floor} is rate-tight ($\err\ge\eta/(2(1-\eta))$) but not threshold-tight: it certifies the $\Theta(\eta)$ floor, not the constant at which learning fails. The two-point method itself is capped at $\eta/(2(1-\eta))$ by the Kearns--Li envelope (\cref{app:floor}, method ceiling in \cref{rem:app-method-ceiling}), so closing the gap to $\eta^\star=\Theta(\gamma/\sigma)$ (and to the upper cap $\eta_0=\rho/4$) requires a multi-point or SQ band-mass argument that exploits the low boundary density $\rho=\Theta(\gamma/\sigma)$, not a sharper pair of instances. We conjecture the tight breakdown is $\eta^\star=\Theta(\gamma/\sigma)$.
\item \textbf{The matching certificate-degree upper bound (now reduced to one estimate).} \cref{thm:tradeoff} lower-bounds the degree of an off-center pro-concentration certificate by $\Omega((|c|/s)^2)$. \Cref{prop:reduction} turns the matching upper bound into the weighted-Chebyshev extremal problem $\mathsf E_t(B)$, and \cref{thm:tight} makes the threshold tight, $2t=\Theta((|c|/s)^2)$, \emph{modulo} a single estimate, \cref{lem:wext}: that the extremal cost stays below $1$ up to the Mhaskar--Rakhmanov--Saff edge. Classical Christoffel and finite--infinite-range asymptotics \citep{MateNevaiTotik91,Totik00,LevinLubinsky01} deliver this asymptotically; deriving its uniform non-asymptotic constant is the residual open problem, and by \cref{rem:nogo} it genuinely requires weighted potential theory rather than an elementary construction.
\item \textbf{A self-contained general-log-concave largest-root bound.} The Gauss-quadrature obstruction needs the largest orthogonal-polynomial root for the corrupted weight. We prove the two endpoints elementarily (sub-Gaussian $2s\sqrt{t+1}$ and sub-exponential $O(st)$, via Jacobi-matrix eigenvalues and Gershgorin), but the fully general log-concave case currently invokes the exponential-weight asymptotics of \citet{LevinLubinsky01}. A self-contained largest-root bound for general log-concave weights would remove that black box.
\item \textbf{Does iterative localization escape the barrier?} Our barriers are scoped to one-shot certificates. \citet{Talwar20} and the line after it raise the possibility of an \emph{iterative} localization in the style of \citet{AwasthiBalcanLong17} that re-certifies on a shrinking band. Such a procedure leaves the one-shot-certificate model by construction, so \cref{thm:tradeoff} does not forbid it; whether iteration provably escapes the margin--degree barrier, and at what compute, is open and is the most likely route to removing the $\log(1/\eps)$ without paying $n^{\Omega(\log 1/\eps)}$. One \emph{heuristic} accounting suggests the answer is yes: a first-order recursion for band-conditioned re-certification contracts quadratically, $\sigma_{k+1}\approx\sqrt{\sigma_k\eta}$, which would reach the certificate-method floor in $O(\log\log(1/\eta))$ stages. We flag this as heuristic, not a claim: the halfspace truncation-bias step in that recursion is unproven, and it is exactly the step an honest analysis would have to supply.
% claim: R03
\item \textbf{A degree-$2t$ SoS lower bound for the joint program.} By \cref{rem:soslb}, the margin--degree barrier is a degree-$2t$ Sum-of-Squares lower bound for the \emph{one-dimensional}, per-direction band-mass certification. Lifting it to a degree-$2t$ SoS lower bound for the full $d$-dimensional reweighting program of \cref{alg:main} over all directions \emph{jointly} is open for $t\ge2$: the per-direction reduction is lossy (one degree-$2t$ pseudo-moment matrix couples all directions), and the one-dimensional nonnegativity-equals-SoS exactness that powers \cref{rem:soslb} fails in $\R^{d}$. The obstruction to simply gluing the axiswise witnesses can be stated exactly. For $d\ge2$, no nonzero nonnegative measure on $\R^{d}$ gives zero mass to \emph{every} central width-$w$ band: take any point of its support and any unit vector orthogonal to that point, and the open slab around it already carries positive mass. The axiswise Gauss witnesses, each of which empties its own band, therefore admit no joint realization as a measure. Signed objects do escape: there is an explicit signed radial measure whose every central band has mass exactly $0$, so a joint refutation would have to be a genuine pseudo-expectation rather than a measure, which is precisely what would make the lift an SoS lower bound for the method itself rather than a bookkeeping exercise.
\item \textbf{Statistical-query hardness for the learning problem.} Our barriers limit the moment/SoS-certificate method; they are not algorithm-independent. Whether margin-halfspace learning under malicious noise is hard for the broad statistical-query class \citep{DiakonikolasKaneZarifis20}---a different object, needing a family of nearly-indistinguishable hard instances rather than a single certification barrier---is open, and would be the unconditional counterpart to the method-specific barriers proved here.
\end{itemize}

\subsection{Conclusion}
\label{sec:conclusion}

Read through the degree-cost lens, robust margin-halfspace learning by reweighted hinge has a single governing quantity: the SoS degree of the outlier-removal certificate. Degree $2$ is the \citeauthor{Shen25} corner; degree $2t$ buys outlier tolerance $\eta^{1-1/2t}$ and costs $n^{O(t)}$ time; and two barriers fence the achievable region from below: one (\cref{thm:degbarrier}) showing degree $2$ cannot beat $\eta^{1/2}$, the other (\cref{thm:tradeoff}) showing any fixed degree conserves the $\log(1/\eps)$ margin, with the algorithm of \cref{thm:algorithm} tracing the frontier they bound. The barriers are method-specific, against the certificate itself; making the degree the resource is what lets the algorithm and the limits be stated to the same exponent. The quantity that sets that degree is not new: it is the Christoffel function of the marginal, the same object data analysis reads to spot an outlier, here measuring instead the room an adversary has to hide one.

\bibliographystyle{plainnat}
\bibliography{refs}

\appendix
\newpage
\section{Notation and map of results}
\label{app:notation}

For convenience we collect the recurring notation and chart how the results depend on one another. \Cref{tab:notation} lists the symbols; \cref{fig:map} is the dependency graph of the theorems, lemmas, and external facts.

\subsection{Notation}
\label{app:notation-table}

{\footnotesize
\setlength{\tabcolsep}{4.5pt}
\renewcommand{\arraystretch}{1.2}
\begin{center}
\resizebox{\textwidth}{!}{%
\begin{tabular}{@{}l l @{\hspace{1.1em}} l l@{}}
\toprule
Symbol & Meaning & Symbol & Meaning \\
\midrule
\multicolumn{4}{@{}l}{\emph{Noise model and algorithm}}\\
$\eta$ & malicious noise rate & $\eta_0(t)$ & tolerated rate at degree $2t$\\
$\eps$ & target inlier error & $\err_{\Dx}$ & inlier error of $\hat w$\\
$S=\Sclean\cup\Sdirty$ & clean / dirty sample & $q\in[0,1]^n$ & reweighting weights\\
$\hinge$ & hinge loss & $\linsumnorm(q\circ\Sdirty)$ & dirty-gradient quantity \eqref{eq:lsn}\\
$t$ & degree parameter (certificate degree $2t$) & $n$ & sample size\\
\midrule
\multicolumn{4}{@{}l}{\emph{Clean distribution}}\\
$\Dx=\tfrac1K\sum_j\D_j$ & clean marginal & $\D_j$ & log-concave component\\
$\Sigma_j\preceq\sigma^2 I$ & component covariance & $r$ & mean bound $\norm{\mu_j}\le r$\\
$\gamma$ & margin & $w^\star$ & target halfspace\\
$U=w\cdot x$ & one-dimensional projection & $s=\sqrt{w^\top\Sigma w}$ & projected std.\ dev.\\
$\bar\sigma^2(w)$ & reweighted directional variance & $\mu$ & law of $U$\\
\midrule
\multicolumn{4}{@{}l}{\emph{Band, certificate, moments}}\\
$B=[c-\tau,c+\tau]$ & band & $\tau,\ c$ & half-width, center\\
$m_k=\E_\mu[U^k]$ & moments & $\R[u]_{\le 2t}$ & polynomials of degree $\le 2t$\\
$\Esudo$ & degree-$2t$ pseudo-expectation & $p$ & polynomial minorant ($p\le\mathbf 1_B$)\\
$\rho$ & pancake density & $\pancake^\tau_w$ & dense pancake (band)\\
$\mathsf E_t(B)$ & weighted-Chebyshev value & $\mathrm{LEAK}$ & Christoffel--Darboux leakage\\
\midrule
\multicolumn{4}{@{}l}{\emph{Orthogonal polynomials (\cref{sec:perspective})}}\\
$\pi_k$ & orthonormal polynomials of $\mu$ & $K_n(x,y)$ & Christoffel--Darboux kernel\\
$\lambda_n(x)=K_n(x,x)^{-1}$ & Christoffel function & $u_j$ & Gauss(--Hermite) nodes\\
$\lambda_j=\lambda_{t+1}(u_j)$ & Christoffel (Gauss) weights & $\ell_j$ & fundamental Lagrange polys\\
$R_{t+1}$ & largest root of $\pi_{t+1}$ & $\rho_{\mathrm{eq}}$ & equilibrium density\\
$W_B=\sum_{u_j\in B}\lambda_j$ & Christoffel band mass & $e_0=\sum_{u_j\in B}\ell_j$ & in-band block sum\\
\bottomrule
\end{tabular}}
\captionof{table}{Summary of notations.}
\label{tab:notation}
\end{center}
}

\subsection{Map of results}
\label{app:map}

Arrows in \cref{fig:map} point from each ingredient \emph{up} to the result that uses it: the external facts feed the lemmas, which feed the headline results. The Christoffel framework of \cref{sec:perspective} (\textbf{Prop~\ref{prop:resolution}}, \textbf{Prop~\ref{prop:reduction}}) is the organizing layer from which the barriers and the algorithm of \cref{sec:results} are read off.

\clearpage
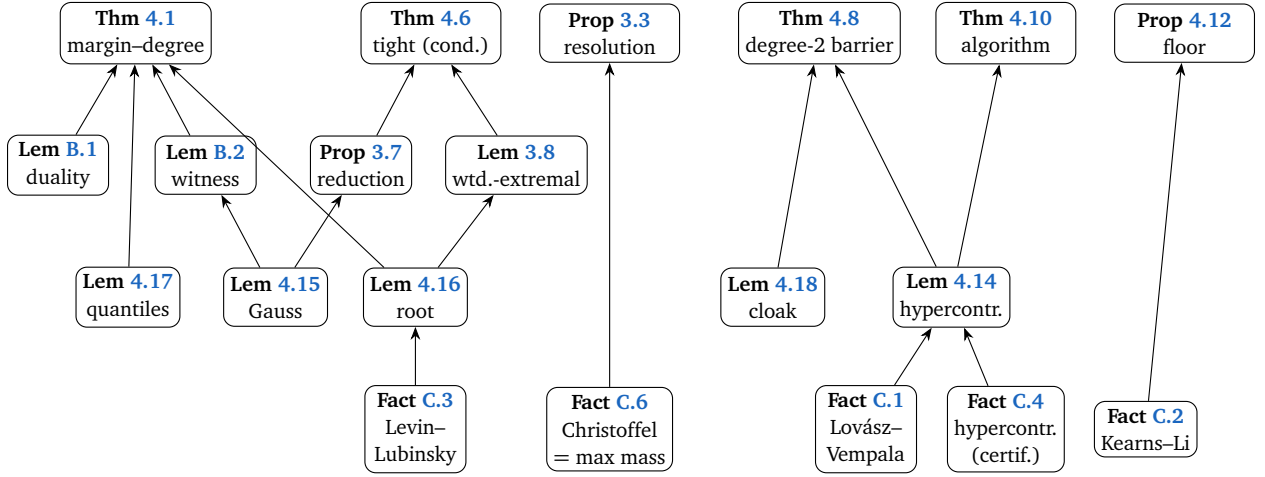
\begin{figure}[t]
\centering
\resizebox{\textwidth}{!}{%
\begin{tikzpicture}[
  >=Stealth,
  res/.style={draw, rounded corners, align=center, inner sep=2.5pt, font=\scriptsize,
              minimum height=7.5mm, minimum width=18mm},
  mid/.style={draw, rounded corners, align=center, inner sep=2.5pt, font=\scriptsize,
              minimum height=7.5mm, minimum width=13mm},
  fct/.style={draw, rounded corners, align=center, inner sep=2.5pt, font=\scriptsize,
              minimum height=7mm, minimum width=13mm},
  e/.style={<-, thin}
]
\node[res] (t1)   at (1.5,0)     {\textbf{Thm~\ref{thm:tradeoff}}\\margin--degree};
\node[res] (tt)   at (5.3,0)     {\textbf{Thm~\ref{thm:tight}}\\tight (cond.)};
\node[res] (pres) at (7.6,0)     {\textbf{Prop~\ref{prop:resolution}}\\resolution};
\node[res] (t2)   at (10.3,0)    {\textbf{Thm~\ref{thm:degbarrier}}\\degree-2 barrier};
\node[res] (t3)   at (12.7,0)    {\textbf{Thm~\ref{thm:algorithm}}\\algorithm};
\node[res] (p4)   at (15.0,0)    {\textbf{Prop~\ref{prop:floor}}\\floor};
\node[mid] (ld)   at (0.5,-1.7)  {\textbf{Lem~\ref{lem:duality}}\\duality};
\node[mid] (lwi)  at (2.4,-1.7)  {\textbf{Lem~\ref{lem:witness}}\\witness};
\node[mid] (pr)   at (4.4,-1.7)  {\textbf{Prop~\ref{prop:reduction}}\\reduction};
\node[mid] (lw)   at (6.4,-1.7)  {\textbf{Lem~\ref{lem:wext}}\\wtd.-extremal};
\node[mid] (lq)   at (1.4,-3.4)  {\textbf{Lem~\ref{lem:quant}}\\quantiles};
\node[mid] (lg)   at (3.3,-3.4)  {\textbf{Lem~\ref{lem:gauss}}\\Gauss};
\node[mid] (lr)   at (5.1,-3.4)  {\textbf{Lem~\ref{lem:root}}\\root};
\node[mid] (lc)   at (9.7,-3.4)  {\textbf{Lem~\ref{lem:cloak}}\\cloak};
\node[mid] (lh)   at (12.0,-3.4) {\textbf{Lem~\ref{lem:hyper}}\\hypercontr.};
\node[fct] (ffr)  at (5.1,-5.1)  {\textbf{Fact~\ref{fact:freud}}\\Levin--\\Lubinsky};
\node[fct] (fcm)  at (7.6,-5.1)  {\textbf{Fact~\ref{fact:maxmass}}\\Christoffel\\$=$ max mass};
\node[fct] (flv)  at (10.9,-5.1) {\textbf{Fact~\ref{fact:lv}}\\Lov\'asz--\\Vempala};
\node[fct] (fhe)  at (12.7,-5.1) {\textbf{Fact~\ref{fact:hyper-ext}}\\hypercontr.\\(certif.)};
\node[fct] (fkl)  at (14.5,-5.1) {\textbf{Fact~\ref{fact:kl}}\\Kearns--Li};
\draw[e] (t1) -- (ld);  \draw[e] (t1) -- (lwi);  \draw[e] (t1) -- (lq);  \draw[e] (t1) -- (lr);
\draw[e] (lwi) -- (lg);
\draw[e] (tt) -- (pr);  \draw[e] (tt) -- (lw);
\draw[e] (t2) -- (lc);  \draw[e] (t2) -- (lh);
\draw[e] (t3) -- (lh);
\draw[e] (pr) -- (lg);
\draw[e] (lw) -- (lr);
\draw[e] (lr) -- (ffr);
\draw[e] (pres) -- (fcm);
\draw[e] (lh) -- (flv);  \draw[e] (lh) -- (fhe);
\draw[e] (p4) -- (fkl);
\end{tikzpicture}%
}
\caption{Map of results: arrows point from each ingredient \emph{up} to the result that uses it.
External facts (bottom, restated in \cref{app:external}) feed the lemmas, which feed the headline results (top). The \S\ref{sec:perspective} framework (\textbf{Prop~\ref{prop:resolution}}, \textbf{Prop~\ref{prop:reduction}}) is the Christoffel-function layer that organizes the margin--degree results, which are proved from the lemmas below.}
\label{fig:map}
\end{figure}

\section{Deferred proofs}
\label{app:proofs}

This appendix reproduces the full proofs. Each result is restated verbatim (via the starred macro, so the appendix copy cannot drift from \cref{sec:results}) and then proved. We proceed in dependency order: the margin--degree barrier and its supporting lemmas first, then the algorithm and the certifiable-hypercontractivity lemma it needs, then the degree-$2$ barrier and its cloak, and finally the breakdown floor. Where a step rests on an external theorem or a not-yet-self-contained sub-lemma, a remark after the proof says so explicitly.

\subsection{The margin--degree tradeoff}
\label{app:tradeoff}

Throughout this subsection $\mu$ is a centered probability measure on $\R$ with all moments finite and infinite support, $m_k=\int u^k\,d\mu$ its moments and $s$ its standard deviation; $B=[c-\tau,c+\tau]$ is a band of half-width $\tau=\theta s$, $\theta=\Theta(1)$, and $\rho=\Omega(1)$ a target mass. A \emph{degree-$2t$ moment-only certificate} of $\mu(B)\ge\rho$ is a polynomial minorant $p\le\mathbf 1_B$, $\deg p\le 2t$, together with a degree-$2t$ SoS proof \emph{from the moment axioms} $\{\int u^k\,d\nu=m_k\}_{k\le 2t}$ alone that $\int p\,d\nu\ge\rho$ (\cref{lem:gauss}\,(i)). Reading only the moments, such a proof certifies $\mu(B)\ge\E_\mu[p]\ge\rho$ for the true law and, simultaneously, for \emph{every} nonnegative $\nu$ matching $m_0,\dots,m_{2t}$; hence the best certifiable value is
\begin{align}\label{eq:app-ceil}
  \overline{\rho}(2t)
  \;\defeq\;
  \sup_{\substack{p:\ \deg p\le 2t\\ p\le\mathbf 1_B}}\;
  \inf_{\substack{\nu\ge0\\ \int u^k\,d\nu=m_k\,(k\le 2t)}}\;
  \int p\,d\nu ,
\end{align}
and a certificate exists only if $\overline{\rho}(2t)\ge\rho$. The next two lemmas isolate the model-independent core (weak duality and a Gauss-quadrature witness) after which \cref{thm:tradeoff} is one per-class root estimate.

\begin{lemma}[weak moment--LP duality]\label{lem:duality}
The certifiable ceiling is at most the moment-LP band minimum, $\overline{\rho}(2t)\le\underline{m}(2t)$, where
\begin{align}\label{eq:app-md}
  \underline{m}(2t)\;\defeq\;\inf\big\{\nu(B):\nu\ge0,\ \textstyle\int u^k\,d\nu=m_k\ \forall k\le 2t\big\}.
\end{align}
In particular, if some nonnegative $\nu$ matches $m_0,\dots,m_{2t}$ and has $\nu(B)=0$, then $\underline{m}(2t)=\overline{\rho}(2t)=0$ and no degree-$2t$ certificate of $\mu(B)\ge\rho$ exists.
\end{lemma}
\begin{proof}
For any admissible $p$ ($\deg p\le 2t$, $p\le\mathbf 1_B$) and any admissible $\nu$, the integral $\int p\,d\nu=\sum_{k\le 2t}p_k m_k=\E_\mu[p]$ is the \emph{same} for all admissible $\nu$, since it depends on $\nu$ only through the matched moments. Let $\nu_\star$ attain \eqref{eq:app-md}. Then
\begin{align*}
  \E_\mu[p]\;=\;\int p\,d\nu_\star\;\le\;\int\mathbf 1_B\,d\nu_\star\;=\;\nu_\star(B)\;=\;\underline{m}(2t),
\end{align*}
using $p\le\mathbf 1_B$ and $\nu_\star\ge0$; taking the supremum over $p$ gives the inequality. The ``in particular'' is the case $\underline m(2t)=0$ via \eqref{eq:app-ceil}. Only this direction (weak duality) is used; the reverse inequality is a separate tightness question, not load-bearing here.
\end{proof}

\begin{lemma}[Gauss quadrature empties a root-free band]\label{lem:witness}
Let $u_1<\dots<u_{t+1}$ be the real simple zeros of the degree-$(t{+}1)$ orthonormal polynomial $\pi_{t+1}$ of $\mu$, and $R_{t+1}\defeq\max_{1\le i\le t+1}\abs{u_i}$. If the band avoids them all, $\abs c-\theta s>R_{t+1}$, then $\underline{m}(2t)=0$; hence, by \cref{lem:duality}, no degree-$2t$ certificate of $\mu(B)\ge\rho$ exists.
\end{lemma}
\begin{proof}
Let $\nu_{\mathrm{GQ}}=\sum_i\lambda_i\delta_{u_i}$ be the Gauss measure with Christoffel weights $\lambda_i>0$. By \cref{lem:gauss}\,(ii) it is nonnegative, atomic on the $t{+}1$ nodes, and exact on degrees $\le 2t+1$; in particular it matches $m_0,\dots,m_{2t}$ and is admissible in \eqref{eq:app-md}. Therefore $\nu_{\mathrm{GQ}}(B)=\sum_{i:\,u_i\in B}\lambda_i$, which is $0$ \emph{iff no node lies in $B$}. A node can lie in $B=[c-\theta s,c+\theta s]$ only if $\abs{u_i}\ge\abs c-\theta s$, so $\abs c-\theta s>R_{t+1}$ leaves $B$ node-free and $\nu_{\mathrm{GQ}}(B)=0$; as $\nu_{\mathrm{GQ}}$ is admissible in \eqref{eq:app-md}, $\underline m(2t)=0$.
\end{proof}

\thmtradeoff*

\begin{proof}
Fix one direction $w$; a lower bound needs only one bad pair $(w,c)$, so no uniformity-over-$w$ issue arises. Apply the setup above to $\mu=$ the centered log-concave law of $U=w\cdot x$ (standard deviation $s$; all moments finite and the law of infinite support, as log-concave laws are sub-exponential), with band $B=[c-\theta s,c+\theta s]$. By \cref{lem:duality,lem:witness}, a degree-$2t$ certificate of $\rho>0$ forces $B$ to \emph{contain} a zero of $\pi_{t+1}$, i.e.\ $R_{t+1}\ge\abs c-\theta s$. It remains to bound $R_{t+1}$ in each marginal class, the one place the marginal's tail enters.

\emph{Sub-Gaussian marginal.} By \cref{lem:root}, an $e^{-\abs{u/s}^2}$ tail (exactly Hermite for the Gaussian) has $R_{t+1}\le C_1 s\sqrt{t+1}$. Combined with $R_{t+1}\ge\abs c-\theta s$ this forces $\sqrt{t+1}\ge(\abs c/s-\theta)/C_1$, i.e.\ $2t=\Omega\big((\abs c/s)^2\big)$ for $\abs c/s\ge 2\theta$, the first half of \eqref{eq:tradeoff-deg}.

\emph{Sub-exponential marginal.} By \cref{lem:root}, an $e^{-\abs{u/s}}$ tail (Freud exponent $\alpha=1$, e.g.\ Laplace) has $R_{t+1}\le C_1' s\,t$, so $R_{t+1}\ge\abs c-\theta s$ forces $2t=\Omega(\abs c/s)$, the second half of \eqref{eq:tradeoff-deg}. The power differs, but the dichotomy below shows both feed the same $\log(1/\eps)$.

\emph{The margin--degree dichotomy.} A certificate learner with inlier error $\eps$ must certify a dense pancake around all but a $\beta\lesssim\eps$ fraction of the clean sample projections; the worst center is the $\beta$-quantile $\abs{c^\star}$. By the log-concave tail bounds of \cref{lem:quant}, $\abs{c^\star}=\Theta(s\sqrt{\log(1/\beta)})$ in the sub-Gaussian case and $\abs{c^\star}=\Theta(s\log(1/\beta))$ in the sub-exponential case. Substituting into \eqref{eq:tradeoff-deg}:
\begin{align*}
  \text{sub-Gaussian:}\quad &2t=\Omega\big((\abs{c^\star}/s)^2\big)=\Omega(\log(1/\eps)),\\
  \text{sub-exponential:}\quad &2t=\Omega(\abs{c^\star}/s)=\Omega(\log(1/\eps)),
\end{align*}
the same floor $t=\Omega(\log(1/\eps))$ at fixed standardized half-width $\tau=\Theta(s)$. The only way out is to widen $\tau$. Under near-isotropy ($\norm{w}\le1/\gamma$, $s=\Theta(1)$) the geometric correspondence of \cref{lem:quant} turns a standardized half-width $\tau=\Theta(s)$ into a physical margin $\gamma=\Theta(\tau\sqrt d)$, so escaping the degree floor requires $\gamma=\Omega(\sqrt{\log(1/\eps)}/\sqrt d)$. This is the dichotomy \eqref{eq:tradeoff-dichotomy}: degree $\Omega(\log(1/\eps))$ or margin $\Omega(\sqrt{\log(1/\eps)}/\sqrt d)$.

\end{proof}

\begin{remark}[Why \citeauthor{Shen25} pays $\log$, not $\sqrt{\log}$]
\label{rem:shen-log}
\citeauthor{Shen25}'s marginal is sub-exponential log-concave (Laplace tails, $\alpha=1$). Its $\eps$-quantile is $\abs{c^\star}=\Theta(s\log(1/\eps))$ and its largest root is $R_{t+1}=O(s\,t)$, so covering the $\eps$-tail needs $t=\Omega(\log(1/\eps))$ and the margin pays $\log(1/\eps)$ directly; the $\sqrt{\cdot}$ appears only for genuinely sub-Gaussian marginals. This recovers, and explains, the limitation \citet{Shen25} records for the approach.
\end{remark}

\begin{remark}[The barrier is elementary; the matching upper bound is open]
\label{rem:app-status-tradeoff}
The spine (\cref{lem:duality,lem:witness} and the dichotomy) is self-contained: weak duality is the elementary argument above, and the dichotomy follows from the named-case root bounds and the standard log-concave quantiles. The dependence on \citet{LevinLubinsky01} enters only through the general-$\alpha$ clause of \cref{lem:root}, which is \emph{not} used by either named corollary; the two cases that matter (sub-Gaussian, sub-exponential) are proved elementarily in \cref{app:root}. A matching upper certificate of degree $O((\abs c/s)^2)$ (which would upgrade \eqref{eq:tradeoff-deg} to $2t=\Theta((\abs c/s)^2)$) is open.
\end{remark}

\subsection{Minorant reduction and Gauss quadrature}
\label{app:gauss}

\lemgauss*

\begin{proof}
\textup{(i)} For the lower-bound direction: if $p\le\mathbf 1_B$ with a degree-$2t$ SoS proof from the moment axioms that $\int p\,d\nu\ge\rho$, then applying that proof to the genuine pseudo-expectation induced by $\mu$ (which is a valid degree-$2t$ pseudo-expectation, as $\mu$ has all moments) and using soundness gives $\Pr_\mu[U\in B]=\E_\mu[\mathbf 1_B]\ge\E_\mu[p]\ge\rho$. Conversely, any degree-$2t$ moment-only Positivstellensatz lower bound on $\Pr_\mu[U\in B]$ produces such a minorant $p$ as the polynomial certificate it manipulates: the certificate cannot evaluate $\mathbf 1_B$ itself, so it must bound it below by a degree-$\le 2t$ polynomial and argue about that polynomial's moment integral. This is the standard indicator-sandwich characterization of moment lower bounds.

\textup{(ii)} The zeros of $\pi_{t+1}$ are real and simple, the Christoffel weights are positive, and the $(t{+}1)$-point Gauss rule integrates every polynomial of degree $\le 2t+1$ exactly against $\mu$; these are the classical properties of Gaussian quadrature. In particular $\int g\,d\nu_{\mathrm{GQ}}=\int g\,d\mu$ for $\deg g\le 2t+1$, so the moments $m_0,\dots,m_{2t}$ are matched, and $\nu_{\mathrm{GQ}}\ge0$ is supported on exactly the $t{+}1$ nodes.
\end{proof}

\begin{remark}[Standard orthogonal-polynomial theory, cited not reproved]
\label{rem:app-status-gauss}
This is textbook orthogonal-polynomial theory. The Gauss-quadrature facts in (ii), real simple zeros, positive Christoffel weights, degree-$(2t+1)$ exactness, are in \citet[Ch.~3 and~15]{Szego39} and \citet{Gautschi04}; the indicator-sandwich in (i) is the standard reduction for moment lower bounds. We cite rather than reprove.
\end{remark}

\subsection{The largest-root bound}
\label{app:root}

\lemroot*

\begin{proof}
The zeros of $\pi_{t+1}$ are exactly the eigenvalues of the $(t{+}1)\times(t{+}1)$ Jacobi (symmetric tridiagonal) truncation $J_{t+1}$ of multiplication-by-$u$ in the orthonormal basis. Writing the three-term recurrence as
\begin{align}\label{eq:app-recur}
  u\,\pi_k\;=\;b_k\,\pi_{k+1}+a_k\,\pi_k+b_{k-1}\,\pi_{k-1},
  \qquad a_k=\int u\,\pi_k^2\,d\mu,\quad b_k=\int u\,\pi_k\pi_{k+1}\,d\mu>0,
\end{align}
the matrix $J_{t+1}$ has diagonal $(a_0,\dots,a_t)$ and off-diagonal $(b_0,\dots,b_{t-1})$. By Gershgorin's theorem, every eigenvalue lies within $\abs{a_k}+b_k+b_{k-1}$ of some diagonal entry $a_k$, so
\begin{align}\label{eq:app-gershgorin}
  R_{t+1}\;=\;\max_i\abs{u_i}\;\le\;\max_k\big(\abs{a_k}+b_k+b_{k-1}\big).
\end{align}
It remains to bound the recurrence coefficients in each case.

\emph{Sub-Gaussian (Hermite).} For $\mu=\mathcal N(0,s^2)$ the orthonormal polynomials are the scaled Hermite polynomials, with $a_k=0$ and $b_k=s\sqrt{(k+1)/2}$. Then \eqref{eq:app-gershgorin} gives $R_{t+1}\le b_t+b_{t-1}\le 2s\sqrt{(t+1)/2}\le 2s\sqrt{t+1}$. (The exact largest Hermite zero is $s\sqrt{2(t+1)}\,(1-o(1))$, confirming the $\sqrt t$ scaling.) For a general sub-Gaussian Freud weight $e^{-\abs{u/s}^2}$ the same conclusion holds with the recurrence coefficients controlled at scale $s\sqrt{k}$, giving $R_{t+1}\le C_1 s\sqrt{t+1}$.

\emph{Sub-exponential (Laplace, $\alpha=1$).} For a Freud weight $e^{-\abs{u/s}}$ the recurrence coefficients scale as $a_k,b_k\asymp s\,k$, so \eqref{eq:app-gershgorin} gives $R_{t+1}=O(s\,t)$. The Mhaskar--Rakhmanov--Saff (MRS) number $a_n\asymp s\,n$ is the right scale here.

\emph{General Freud exponent $\alpha\in(1,2)$.} For a convex weight $e^{-\abs{u/s}^\alpha}$ the recurrence coefficients satisfy $a_k,b_k=O(s\,k^{1/\alpha})$, where the genuine $O(s\,t^{1/\alpha})$ scaling (rather than the cruder $O(s\,t^2)$ that a naive telescoping of \eqref{eq:app-recur} would give) is the MRS number, supplied by the exponential-weight asymptotics of \citet{LevinLubinsky01} (\cref{fact:freud}; see also \citet{KriecherbauerMcLaughlin99} for the Riemann--Hilbert treatment of the zero distribution). This yields $R_{t+1}\asymp s\,t^{1/\alpha}$.
\end{proof}

\begin{remark}[The named cases are elementary; general $\alpha$ is the one citation]
\label{rem:app-status-root}
The two endpoints used by \cref{thm:tradeoff}, sub-Gaussian ($2s\sqrt{t+1}$) and sub-exponential ($O(s\,t)$), are fully elementary: Jacobi-matrix eigenvalues plus Gershgorin, with the Hermite and Laplace recurrence coefficients computed (or read off) directly. The general-$\alpha$ clause is the one place a black box enters: the genuine universal $O(s\,t^{1/\alpha})$ scaling (as opposed to the elementary but loose $O(s\,t^2)$ telescoping bound) consumes the Freud-weight MRS asymptotics of \citet{LevinLubinsky01}. Since the two named corollaries of \cref{thm:tradeoff} use only the endpoints, the headline barrier does not wait on this citation; we flag it so the division of labor is transparent.
\end{remark}

\subsection{Log-concave quantiles and margin geometry}
\label{app:quant}

\lemquant*

\begin{proof}
\emph{Quantiles.} For a centered log-concave law at scale $s$, the upper tail satisfies $\Pr[\abs U\ge ts]\le e^{-\Omega(t)}$ for $t\ge1$ (sub-exponential), with the stronger $e^{-\Omega(t^2)}$ under sub-Gaussianity; these are the standard one-dimensional log-concave tail bounds \citep{LovaszVempala07}. Inverting at level $\beta$ gives the $\beta$-quantile $\abs{c^\star}=\Theta(s\log(1/\beta))$ in the sub-exponential case and $\Theta(s\sqrt{\log(1/\beta)})$ in the sub-Gaussian case.

\emph{Margin geometry.} Homogenize $x\mapsto(x,1)$ so that a halfspace with margin becomes a linear threshold, and normalize $\norm w\le1/\gamma$. Under $\Sigma\preceq\sigma^2 I$ with $\sigma=1/\sqrt d$ (near-isotropy, $s=\Theta(1)$), a standardized half-width $\tau=\Theta(s)$ in the projected coordinate $U=w\cdot x$ corresponds to a physical band of width $\Theta(\tau)$ scaled by $\norm w$, i.e.\ a margin $\gamma=\Theta(\tau\sqrt d)$.
\end{proof}

\begin{remark}[Routine log-concave estimates and homogenization]
\label{rem:app-status-quant}
Both parts are routine. The tail/quantile bounds are the standard one-dimensional log-concave estimates \citep{LovaszVempala07}; the geometric correspondence is the usual homogenization under near-isotropy. Recorded here so the geometric translation in the margin--degree dichotomy of \cref{thm:tradeoff} is not waved through.
\end{remark}

\subsection{The degree-$2t$ reweighted hinge}
\label{app:algorithm}

We first prove the certifiable-hypercontractivity lemma the algorithm needs, then the algorithm itself.

\lemhyper*

\begin{proof}
\emph{Step 1 (the correct scalar constant).} A centered one-dimensional log-concave $Z=w\cdot z$ has $\norm{Z}_{L^{2t}}\le a\cdot 2t\,\norm{Z}_{L^2}$ with $a$ absolute (the $\psi_1$ / sub-exponential moment growth, \emph{linear} in the order, of a log-concave law; \cref{fact:lv}, \citealp{LovaszVempala07}). Hence $m_{2t}\le(C_0 t)^{2t}(s^2)^t$ with $C_0=2a$. This constant is tight in the worst case: for the Laplace law, $m_{2t}/m_2^t=(2t)!/2^t=(\Theta(t))^{2t}$, so \emph{no} $(C t)^t$ bound can hold for all $t$. This is why the sub-exponential constant $(Ct)^{2t}$, and not the sub-Gaussian $(Ct)^t$, is the correct one for general log-concave marginals; it is consistent with the $\alpha=1$ root scaling in \cref{lem:root}.

\emph{Step 2 (single-component SoS certificate).} For a single component $\Dcal_j$, the tensor inequality
\[
  \langle M_j^{(2t)},w^{\otimes 2t}\rangle\le(C_0 t)^{2t}(w^\top\Sigma_j w)^t
\]
is degree-$2t$ SoS-certifiable in $w$, the right-hand side being a $t$-th power of a PSD quadratic and hence SoS. For sub-Gaussian components this certificate is given by \citet{KlivansKothariMeka18,DiakonikolasHopkinsPensiaTiegel24} (\cref{fact:hyper-ext}); for general log-concave components we use the scalar$\to$SoS lift (\cref{rem:app-status-hyper} below), the SoS Cauchy--Schwarz/Hölder pyramid that bootstraps the scalar bound of Step~1 through the certified moment tensors $M^{(2k)}$, $k\le t$.

\emph{Step 3 (the mixture, via a binomial split).} The mixture is not log-concave, so we split each component contribution. With $Z_j=w\cdot(x-\mu_j)$ and the bounded spread $\beta_j\cdot w$ where $\beta_j=\mu_j-\bar\mu$ (so $\abs{\beta_j\cdot w}\le r$ for $\norm w=1$),
\begin{align}\label{eq:app-binom}
  \big(w\cdot(x-\bar\mu)\big)^{2t}
  \;=\;(Z_j+\beta_j\cdot w)^{2t}
  \;\le\;2^{2t-1}\big(Z_j^{2t}+(\beta_j\cdot w)^{2t}\big),
\end{align}
by convexity of $u\mapsto u^{2t}$. The gap in \eqref{eq:app-binom} is, for fixed $j$, a univariate nonnegative polynomial in the linear forms together with a positive multiple of the certified tensor term, hence \eqref{eq:app-binom} is an SoS inequality in $w$.

\emph{Step 4 (averaging preserves SoS).} Averaging \eqref{eq:app-binom} over the components, and using that a nonnegative combination of SoS proofs is SoS,
\begin{align*}
  \E_{\Dx}\big[(w\cdot(x-\bar\mu))^{2t}\big]
  \;\le\;2^{2t-1}\,\frac1K\sum_{j}\Big(\E_{\Dcal_j}[Z_j^{2t}]+(\beta_j\cdot w)^{2t}\Big)
  \;=\;2^{2t-1}\big(\Phi_t(w)+\Psi_t(w)\big),
\end{align*}
with the within-component term $\Phi_t(w)=\frac1K\sum_j\E_{\Dcal_j}[Z_j^{2t}]$ and the spread term $\Psi_t(w)=\frac1K\sum_j(\beta_j\cdot w)^{2t}$. The total-variance identity
\begin{align}\label{eq:app-totvar}
  \bar\sigma^2(w)\;=\;\frac1K\sum_j w^\top\Sigma_j w+\frac1K\sum_j(\beta_j\cdot w)^2
\end{align}
bounds both terms by $\bar\sigma^2(w)^t$: the term $\Phi_t(w)$ by Step~2 applied componentwise (each $\E_{\Dcal_j}[Z_j^{2t}]\le(C_0 t)^{2t}(w^\top\Sigma_j w)^t$) and the PSD-domination $w^\top\Sigma_j w\le\bar\sigma^2(w)\cdot O(1)$ supplied by the isotropy constant $\kappa=O(1)$ in the regime $r=O(\sigma\sqrt d)$; the term $\Psi_t(w)$ by $(\beta_j\cdot w)^{2t}=((\beta_j\cdot w)^2)^t\le\bar\sigma^2(w)^t$ from \eqref{eq:app-totvar}. Assembling, with $C=4C_0\sqrt\kappa$, gives the SoS-certified bound \eqref{eq:hyper}.
\end{proof}

\begin{remark}[Why the constant is $(Ct)^{2t}$, not $(Ct)^t$]
\label{rem:app-status-hyper}
The constant is the sub-exponential $(Ct)^{2t}$; the Laplace computation in Step~1 shows $(Ct)^t$ is \emph{false} for general log-concave marginals, so we do not claim it. The one place a general-log-concave certificate is needed beyond the cited sub-Gaussian work of \citet{KlivansKothariMeka18,DiakonikolasHopkinsPensiaTiegel24} is the single-component step; we discharge it by the scalar$\to$SoS lift (the SoS Cauchy--Schwarz/Hölder pyramid on the certified moment tensors, bootstrapped from the scalar Lovász--Vempala bound). The $\alpha=1$ constant tracking in that pyramid is the residual technical point; the mixture, binomial-split, and averaging steps are self-contained.
\end{remark}

\thmalgorithm*

\begin{proof}
\emph{The $\linsumnorm$ bound \eqref{eq:alg-lsn}.} Fix $w$ with $\norm w\le1$. By Hölder's inequality with conjugate exponents $\tfrac{2t}{2t-1}$ and $2t$,
\begin{align*}
  \sum_{i\in\Sdirty}q_i\abs{w\cdot x_i}
  \;\le\;\Big(\sum_{i\in\Sdirty}q_i\Big)^{1-\frac1{2t}}
         \Big(\sum_{i\in\Sdirty}q_i\abs{w\cdot x_i}^{2t}\Big)^{\frac1{2t}}.
\end{align*}
The first factor is at most $(\eta n)^{1-\frac1{2t}}$ since $\sum_{i\in\Sdirty}q_i\le\abs{\Sdirty}\le\eta n$. For the second, $\abs{w\cdot x}^{2t}= (w\cdot x)^{2t}$ is an even (hence polynomial) power, the summand is nonnegative, and $\Sdirty\subseteq S$, so
\[
  \sum_{i\in\Sdirty}q_i(w\cdot x_i)^{2t}\ \le\ \sum_{i\in S}q_i(w\cdot x_i)^{2t}\ \le\
  n\,(Ct)^{2t}\bar\sigma^{2t}
\]
by the degree-$2t$ certificate that $q$ satisfies. Hence
\begin{align*}
  \sum_{i\in\Sdirty}q_i\abs{w\cdot x_i}
  \;\le\;(\eta n)^{1-\frac1{2t}}\big(n(Ct)^{2t}\bar\sigma^{2t}\big)^{\frac1{2t}}
  \;=\;n\,(Ct)\,\bar\sigma\,\eta^{1-\frac1{2t}},
\end{align*}
and taking the supremum over $\norm w\le1$ gives \eqref{eq:alg-lsn}. At $t=1$ this is $n\,C\bar\sigma\,\eta^{1/2}$, exactly \citeauthor{Shen25}'s Lemma~6; as $t\to\infty$ the exponent of $\eta$ tends to $1$.

\emph{Feasibility.} The clean-indicator weight ($q_i=1$ on $\Sclean$, $q_i=0$ on $\Sdirty$, so $\xi=\eta$) has the required degree-$2t$ certificate: it is exactly the SoS bound of \cref{lem:hyper} applied to the empirical clean measure, which holds with high probability once $n\ge d^{O(t)}$ (enough samples for the empirical moment tensors to concentrate around the population tensors that \cref{lem:hyper} certifies). The program is thus nonempty, and the SDP has size $d^{O(t)}$ and runtime $n^{O(t)}$.

\emph{The noise tolerance \eqref{eq:alg-eta0}.} Plug \eqref{eq:alg-lsn} into the robustness condition of \citet{Shen25}. Along the empirical optimum $\hat w$, the clean hinge gradient is lower-bounded by the dense-pancake density $\rho$ (the in-regime log-concave band mass along $\hat w$); the condition for the dirty gradient not to overwhelm it is $\frac\gamma4(\rho-2\xi)>\linsumnorm(q\circ\Sdirty)/n$. Using \eqref{eq:alg-lsn} this reads $\frac\gamma4(\rho-2\xi)>(Ct)\bar\sigma\,\eta^{1-1/2t}$. Balancing the two bottlenecks (the pancake budget $\xi\lesssim\rho$, independent of $t$, and the outlier term) solves to
\begin{align*}
  \eta\;\le\;\eta_0(t)\;=\;\Omega\Big(\min\Big\{\tfrac{\rho}{4},\,
    \big(\gamma\rho/((Ct)\bar\sigma)\big)^{\frac{2t}{2t-1}}\Big\}\Big),
\end{align*}
which is \eqref{eq:alg-eta0}. Under this $\eta$, the \citet{Shen25} analysis yields $\err_{\Dx}(\hat w)\le\eps$, in time $n^{O(t)}$ with $n=d^{O(t)}\poly(1/\eps)$.
\end{proof}

\begin{remark}[The gain is in the exponent, not a march to $1/2$]
\label{rem:app-status-algorithm}
The $\linsumnorm$ bound \eqref{eq:alg-lsn} and feasibility are self-contained modulo \cref{lem:hyper}. The conversion to \eqref{eq:alg-eta0} uses the \citet{Shen25} robustness framework (its Theorem~9 / Lemma~12) as an external black box: we plug the degree-$2t$ $\linsumnorm$-bound into \citeauthor{Shen25}'s gradient-domination argument rather than reproving it, and the dense-pancake lower bound along $\hat w$ is the log-concave band mass guaranteed in regime. The improvement over \citeauthor{Shen25} is in the \emph{exponent} $\tfrac{2t}{2t-1}\to1$ of \eqref{eq:alg-eta0}: it removes the degree-$2$ squaring loss, raising $\eta_0$ from a tiny constant to a larger one bounded by $\rho/4$. It does \emph{not} approach $1/2$; by \cref{thm:degbarrier} the $t=1$ corner is already optimal within the method, which is where the value lies.
\end{remark}

\subsection{The degree-2 outlier barrier}
\label{app:degbarrier}

We first record the cloak, then prove the barrier.

\lemcloak*

\begin{proof}
Take each $\mathcal C_\pm$ to be an isotropic log-concave bump of scale $\sigma$ centered at $\pm a e_1$ with $a=\sigma/\sqrt\eta$; a unimodal scale-$\sigma$ density is log-concave, which is the reason for using two bumps rather than a single $\pm a$ atom (the latter would be bimodal, hence not log-concave). The mean offset is $\norm{\mu_\pm}=\sigma/\sqrt\eta$, which is $\le r$ within the certified regime $r=O(\sigma\sqrt d)$ exactly when $\eta\ge\Omega(1/d)$. Draw $2\eta n$ clean and $\eta n$ dirty points from the same $\mathcal C_\pm$, so the combined cluster has $\approx 3\eta n$ points near $\pm a e_1$. The empirical second moment in the $e_1$ direction over the dirty points is $\frac1n\sum_{\Sdirty}(e_1\cdot x_i)^2\approx\eta\cdot a^2=\eta\cdot \sigma^2/\eta=\sigma^2$, so the cluster spends exactly the $e_1$-variance budget. Since clean and dirty points are drawn from the identical law at identical locations, the multiset $\{x_i\}$ is invariant under any permutation that exchanges dirty and clean cluster points.
\end{proof}

\begin{remark}[The only subtlety is keeping the cluster log-concave]
\label{rem:app-status-cloak}
The only nontrivial input is that a unimodal scale-$\sigma$ bump is log-concave (avoiding the bimodality that would break a single $\pm a$ cluster); variance and permutation symmetry are direct computations. The regime $\eta\ge\Omega(1/d)$ is a genuine restriction inherited by \cref{thm:degbarrier} and is stated there.
\end{remark}

\thmdegbarrier*

\begin{proof}
Use the construction of \cref{lem:cloak}: two log-concave clusters $\mathcal C_\pm$ at means $\pm a e_1$, $a=\sigma/\sqrt\eta$, with $2\eta n$ clean and $\eta n$ dirty points drawn from the same law. The adversary's only freedom is \emph{which} of the $\approx 3\eta n$ cluster points are labeled dirty.

\emph{The cloak.} By \cref{lem:cloak}, (a) the dirty points are \emph{second-moment invisible}: $\frac1n\sum_{\Sdirty}(e_1\cdot x_i)^2=\eta a^2=\sigma^2$, which is within the directional-variance budget $\bar\sigma^2$ and so does not flag them; and (b) they are \emph{first-moment aligned}: $\sum_{\Sdirty}\abs{e_1\cdot x_i}=\eta n\cdot a= \eta n\cdot\sigma/\sqrt\eta=\sqrt\eta\,n\sigma$, which is exactly the $\linsumnorm$ mass the program must somehow remove.

\emph{The label-oblivious lower bound.} The literal ``for all feasible $q$'' claim is \emph{false} once the budget $\xi\ge\eta$: the feasibility polytope then contains the weight $q$ that zeros all of $\Sdirty$ (and indeed the whole spike), which costs at most $3\eta n\le\xi n$ removed points if $\xi\ge3\eta$, and second-moment feasibility is unaffected. So a budget-aware adversary cannot force \emph{every} feasible $q$ to keep the spike. The correct statement is in expectation over which cluster points are dirty. The degree-$2$ variance constraints depend only on the multiset $\{x_i\}$, which by \cref{lem:cloak} is invariant under exchanging dirty and clean cluster points; hence a \emph{label-oblivious} program (one not given the ground-truth corruption labels) must treat all $\approx 3\eta n$ cluster points identically. It removes at most $\xi n=\kappa\eta n$ points total, so the combined cluster retains weight $\ge(3-\kappa)\eta n$. By symmetry, in expectation over which $\eta n$ of the $3\eta n$ cluster points are dirty, the dirty share of the retained weight is $\tfrac13$, so the expected retained dirty weight is $\ge\tfrac13(3-\kappa)\eta n\ge\tfrac13\eta n$ for $\kappa\le2$. Each retained dirty point contributes $\abs{e_1\cdot x_i}\approx a$ to $\linsumnorm$ along $w=e_1$, so
\begin{align*}
  \E\big[\linsumnorm(q\circ\Sdirty)\big]
  \;\ge\;\tfrac13\eta n\cdot a
  \;=\;\tfrac13\eta n\cdot\frac{\sigma}{\sqrt\eta}
  \;=\;\tfrac13\sqrt\eta\,n\sigma
  \;=\;\Omega\big(\eta^{1/2}n\bar\sigma\big),
\end{align*}
which is \eqref{eq:degbarrier} (using $\bar\sigma=\Theta(\sigma)$). Comparing with the $\eta^{1-1/2t}$ rate of \cref{thm:algorithm} at $t=1$ shows degree $2$ is stuck at $\eta^{1/2}$, i.e.\ $\eta_0\lesssim(\gamma\rho/\bar\sigma)^2$.

\emph{Degree $4$ escapes.} The same $\Sdirty$ has an \emph{inflated} empirical fourth moment in the $e_1$ direction: $\frac1n\sum_{\Sdirty}(e_1\cdot x_i)^4\approx\eta\cdot a^4 =\eta\cdot\sigma^4/\eta^2=\sigma^4/\eta\gg\bar\sigma^4$. A degree-$4$ certificate sees this; degree-$4$ feasibility forces the reweighted fourth moment down, which by the $t=2$ instance of \eqref{eq:alg-lsn} drives $\linsumnorm(q\circ\Sdirty)/n\le C'\bar\sigma\,\eta^{3/4}<C'\bar\sigma\,\eta^{1/2}$. So the same instance defeats degree $2$ but not degree $4$, and the separation is exactly one of moment degree.
\end{proof}

\begin{remark}[Label-oblivious and in expectation, not for every $q$]
\label{rem:app-status-degbarrier}
The cloak is \cref{lem:cloak}; the barrier is the label-oblivious in-expectation argument above, which is what matches \cref{thm:algorithm} at $t=1$ (itself label-oblivious). The literal universal-$q$ form is false for $\xi\ge\eta$, as shown; a deterministic ``for all $q$'' statement holds only for $\xi<\eta$. The budget is modeled as the empirical max-direction variance $\bar\sigma^2\approx2\sigma^2$; a stronger robust estimate targeting $\sigma^2$ would have to \emph{detect} the spike, which degree $2$ provably cannot. This affects the constant, not the $\eta^{1/2}$ exponent. The regime $\eta\ge\Omega(1/d)$ is inherited from \cref{lem:cloak}.
\end{remark}

\subsection{The breakdown floor}
\label{app:floor}

% claim: F02a
\propfloor*

\begin{proof}
Two clean instances share one marginal and differ only in how the mixture \emph{weights} are labelled; a hard margin forces that difference to come in quanta of $1/K$, which is what makes both the construction and its necessity clauses work. Throughout, ``log-concave'' is used in the density sense (each $\D_j$ has a log-concave density on $\R^d$), the reading the moment machinery of \cref{lem:hyper} uses; $S_j\defeq\supp\D_j$ is then the closure of the positivity set of that density, hence convex.

\smallskip
\emph{Step 1 (hard-margin quantization).} Let $\Dx=\frac1K\sum_{j\le K}\D_j\in\mathcal C$ carry a $\gamma$-margin labeling with $\gamma>0$ and unit normal $w$, so $\supp\Dx\subseteq\{\abs{w\cdot x}\ge\gamma\}$. \emph{(Lemma Q)} Each $S_j$ lies wholly inside $\{w\cdot x\ge\gamma\}$ or wholly inside $\{w\cdot x\le-\gamma\}$: otherwise $S_j$ contains points on both sides, and, being convex, the segment joining them, which lies in $S_j\subseteq\supp\Dx$, meets $\{w\cdot x=0\}$ by the intermediate value theorem, contradicting the margin.
\emph{(Q1)} Consequently any margin-respecting labeling is constant on each $S_j$, so for two such labelings $y_1,y_2$ the disagreement region $U=\{x:y_1(x)\ne y_2(x)\}$ contains each $S_j$ or misses it, and
\begin{align}\label{eq:app-quant}
  \Dx(U)\in\{0,\tfrac1K,\tfrac2K,\dots,1\}.
\end{align}
\emph{(Q2)} If the two instances $P_1,P_2$ (same marginal, labelings $y_1,y_2$) satisfy $\TV(P_1,P_2)\le\eta/(1-\eta)<1$ and $\Dx(U)>0$, then by Step 4 below $\Dx(U)=\TV(P_1,P_2)$, so \eqref{eq:app-quant} gives $1/K\le\eta/(1-\eta)$, i.e.\ $K\ge(1-\eta)/\eta$. At $K=1$ the labelings are globally constant, so either $P_1=P_2$ or $\TV=1$: a single component cannot host an indistinguishable pair.
\emph{(Q3)} Suppose $0<\Dx(U)<1$, i.e.\ some component disagrees and some agrees, with unit normals $w_1,w_2$ at angle $\theta\in(0,\pi)$. A disagreeing component lies in $\{w_1\cdot x\ge\gamma\}\cap\{w_2\cdot x\le-\gamma\}$ (or its reflection), and any $x$ there has $(w_1-w_2)\cdot x\ge2\gamma$, so $\norm{x}\ge2\gamma/\norm{w_1-w_2}=\gamma/\sin(\theta/2)$; an agreeing component lies in $\{w_1\cdot x\ge\gamma\}\cap\{w_2\cdot x\ge\gamma\}$ (or its reflection), where $(w_1+w_2)\cdot x\ge2\gamma$ gives $\norm{x}\ge\gamma/\cos(\theta/2)$. Each region is closed and convex, so the barycentre $\mu_j$ of a component supported in it obeys the same bound. With both types present, $r\ge\max\{\gamma/\sin(\theta/2),\gamma/\cos(\theta/2)\}\ge\sqrt2\,\gamma$, with equality only at $\theta=\pi/2$. This proves the two necessity clauses of the statement.

\smallskip
\emph{Step 2 (the instance).} Assume (V) and put $R\defeq\sqrt2(\gamma+\rho_b)\le r$, $w_1^\star=e_1$, $w_2^\star=e_2$, and
\[
  c_a\defeq\tfrac{R}{\sqrt2}(e_1+e_2),\qquad c_w\defeq\tfrac{R}{\sqrt2}(e_1-e_2),
  \qquad \norm{c_a}=\norm{c_w}=R .
\]
Fix $K\ge(1-\eta)/\eta$, set $m\defeq\floor{K\eta/(1-\eta)}\ge1$ and $q\defeq m/K$, and take $\D_j=\mathrm{Unif}(B(c_w,\rho_b))$ for $j\le m$ (the \emph{wedge}) and $\D_j=\mathrm{Unif}(B(c_a,\rho_b))$ for $j>m$ (the \emph{anchor}); $m<K$ because $q\le\eta/(1-\eta)<1$. Four checks put this in $\mathcal C$. (2.1) A uniform law on a convex body is log-concave. (2.2) Its covariance is exactly $\frac{\rho_b^2}{d+2}I\preceq\sigma^2I$, which is where the collar bound $\rho_b\le\sigma\sqrt{d+2}$ of (V) is used. (2.3) $\norm{\mu_j}=R\le r$. (2.4) Margins: on $B(c_a,\rho_b)$ both coordinates satisfy $x_1,x_2\ge R/\sqrt2-\rho_b=\gamma$, while on $B(c_w,\rho_b)$ we have $x_1\ge\gamma$ and $x_2\le-\gamma$. Hence $P_1$, labelled by $w_1^\star=e_1$, is all-$+$, and $P_2$, labelled by $w_2^\star=e_2$, is $+$ on the anchor and $-$ on the wedge; both have margin $\gamma$ and both marginals are the same mixture. (If distinct components are wanted, perturb the centres by $\varepsilon$ and apply (V) with $\gamma+\varepsilon$.)

\smallskip
\emph{Step 3 (disagreement mass).} $U\supseteq B(c_w,\rho_b)$ and $U\cap B(c_a,\rho_b)=\emptyset$, so $\Dx(U)=q=m/K\le\eta/(1-\eta)$. At the minimal choice $K=\ceil{(1-\eta)/\eta}\le1/\eta$ we get $q\ge1/K\ge\eta$.

\smallskip
\emph{Step 4 (total variation).} For \emph{any} two instances with a common marginal whose labelings differ exactly on $U$, $\TV(P_1,P_2)=\Dx(U)$: coupling through the shared $x$ shows the two joint laws differ only on $U$, giving $\TV\le\Dx(U)$, and the event $\{y=y_1(x)\}$, of $P_1$-probability $1$ and $P_2$-probability $1-\Dx(U)$, gives the matching lower bound. (This is the identity used in Q2; it needs nothing from Step 2.) For the instance of Step 2, $\TV(P_1,P_2)=q$.

\smallskip
\emph{Step 5 (indistinguishability at the envelope, with the adversary written out).} Let $C$ be the anchor law tensored with the label $+1$, and $E_1,E_2$ the wedge law tensored with $+1$ and $-1$ respectively, so $P_1=(1-q)C+qE_1$ and $P_2=(1-q)C+qE_2$. Since $q\le\eta/(1-\eta)$, the number $\lambda\defeq(1-\eta)q/\eta$ lies in $(0,1]$, and the two noise distributions
\[
  A_1\defeq\lambda E_2+(1-\lambda)C,\qquad A_2\defeq\lambda E_1+(1-\lambda)C
\]
satisfy
\[
  (1-\eta)P_1+\eta A_1
  \;=\;\big(1-2(1-\eta)q\big)\,C+(1-\eta)q\,(E_1+E_2)
  \;=\;(1-\eta)P_2+\eta A_2\;=:\;Q ,
\]
the $C$-coefficient being nonnegative because $(1-\eta)q\le\eta\le\tfrac12$. In words: each adversary flips the labels of a $(1-\eta)q\le\eta$ fraction of examples, all of them inside the wedge, so under both instances the wedge shows balanced $\pm$ labels. This is the ``if'' direction of \cref{fact:kl} made constructive at exactly the parameters used.
\emph{(5.1) Per-example model.} Each example is independently corrupted with probability $\eta$, so the learner's sample law is $Q^{\otimes n}$ under \emph{both} instances, for every $n$: the observation total variation is $\delta=0$.
\emph{(5.2) Fixed-fraction model.} Run the same adversary at rate $\eta-\eps$, marking examples i.i.d.\ with probability $\eta-\eps$ and spending any surplus budget on fresh draws from the instance's own clean law. The marked count concentrates: by Hoeffding, the probability that it exceeds the budget $\eta n$ is at most $e^{-2n(\eps-1/n)^2}$, the deviation being $t=n\eps-1$, which requires $n\eps\ge1$; for $n\eps<1$ the argument gives nothing, and indeed at such $n$ the fixed-fraction adversary may have zero budget. On the complementary event the two sample laws coincide as in (5.1), so $\delta\le2e^{-2n(\eps-1/n)^2}$.

\smallskip
\emph{Step 6 (the pointwise two-point inequality).} For $x\in U$ the labels differ, $y_1(x)\ne y_2(x)$, so $\sgn(w\cdot x)$ disagrees with at least one of them. Hence, for \emph{every} $w$,
\begin{align}\label{eq:app-twopoint}
  \err_{P_1}(w)+\err_{P_2}(w)\;\ge\;\Dx(U)\;=\;q ,
  \qquad\text{so}\qquad \max_i\err_{P_i}(w)\ge q/2 .
\end{align}

\smallskip
\emph{Step 7 (endgame, on the product of sample and seed).} Let $\mathcal A$ be any algorithm, randomized with seed $\xi\sim\nu$ independent of the sample, and let $Q_i$ be the sample law under instance $i$, $\delta\defeq\TV(Q_1,Q_2)$. Put $B_i\defeq\{(s,\xi):\err_{P_i}(\mathcal A(s,\xi))<q/2\}$; these are measurable because the disagreement indicator is Borel in $(w,x,y)$ and Fubini applies. By \eqref{eq:app-twopoint}, $B_1\cap B_2=\emptyset$, and $\TV(Q_1\otimes\nu,Q_2\otimes\nu)=\delta$, so
\[
  (Q_1\otimes\nu)(B_1)+(Q_2\otimes\nu)(B_2)
  \;\le\;(Q_1\otimes\nu)(B_1)+(Q_1\otimes\nu)(B_2)+\delta\;\le\;1+\delta ,
\]
whence $\max_i\Pr_{s,\xi}[\err_{P_i}\ge q/2]\ge(1-\delta)/2$, for a \emph{fixed} index $i$, not one chosen after the seed, since the disjointness is on the product space. For the expectation form, \eqref{eq:app-twopoint} gives $\E_{Q_1\otimes\nu}[\err_1+\err_2]\ge q$, while $\err_2\in[0,1]$ and the sup-over-events convention give $\E_{Q_2\otimes\nu}[\err_2]\ge\E_{Q_1\otimes\nu}[\err_2]-\delta$; adding,
\begin{align}\label{eq:app-endgame}
  \max_i\E\big[\err_{P_i}(\hat w)\big]\;\ge\;\frac{q-\delta}{2}.
\end{align}
With $\delta=0$ (Step 5.1) and $K=\ceil{(1-\eta)/\eta}$, so $q\ge\eta$, the high-probability clause follows: some fixed instance has $\Pr[\err_{P_i}(\hat w)\ge\eta/2]\ge\tfrac12$.

\smallskip
\emph{Step 8 (the constant).} By \eqref{eq:app-endgame} with $\delta=0$, every algorithm suffers $\max_i\E[\err_{P_i}]\ge q/2$ \emph{simultaneously for every admissible $K$}, so the supremum over $K$ passes through the infimum over algorithms. Now
\[
  \sup_{K\ge(1-\eta)/\eta}\ \frac{\floor{K\eta/(1-\eta)}}{K}\;=\;\frac{\eta}{1-\eta},
\]
since $\floor{Kx}/K\le x$ always and $\floor{Kx}/K\ge x-1/K\to x$. (The ratio is \emph{not} monotone in $K$: at $x=3/7$ it runs $\tfrac13,\tfrac14,\tfrac25,\tfrac13,\tfrac37$ for $K=3,\dots,7$, but the supremum is still $x$, attained at the denominators of a rational $x$.) Hence
$\inf_{\mathcal A}\sup_{P,\mathrm{adv}}\E[\err_P(\hat w)]\ge\eta/(2(1-\eta))\ge\eta/2$, which is \eqref{eq:floor}. The fixed-fraction corollary is \eqref{eq:app-endgame} with the rate $\eta-\eps$ and $\delta\le2e^{-2n(\eps-1/n)^2}$ from Step 5.2, valid exactly when $n\eps\ge1$.
\end{proof}

% claim: F02a
\begin{remark}[The method ceiling: $\eta/(2(1-\eta))$ is the exact two-point value]
\label{rem:app-method-ceiling}
The constant cannot be improved along this route, and the reason is worth recording because it is what makes \cref{prop:floor} rate-tight but not threshold-tight. For \emph{any} two-point pair, the randomized hedge that outputs $w_1^\star$ or $w_2^\star$ with probability $\tfrac12$ achieves $\max_i\E[\err_{P_i}]=\TV(P_1,P_2)/2\le\eta/(2(1-\eta))$, since $\err_{P_i}(w_i^\star)=0$ and $\err_{P_i}(w_j^\star)\le\TV$ for $j\ne i$. So $\eta/(2(1-\eta))$ is the exact value of the two-point game in expectation, and the instance above attains it, here even deterministically: $\hat w=(e_1+e_2)/\sqrt2$ is orthogonal to $c_w$, so by central symmetry it bisects the wedge ball, and it labels the whole anchor correctly, giving $\err_{P_1}=\err_{P_2}=q/2$ exactly. One caveat, in the other direction: the hedge caps the \emph{expectation} form, not the high-probability form. With an indivisible disagreement region (a degenerate atomic component) every $\hat w$ mislabels the atom under one instance and the high-probability floor can reach $\eta/(1-\eta)$; such components lie outside the density-sense class fixed above, so this is an out-of-class remark and the in-class high-probability constant is open.
\end{remark}

% claim: F02a
\begin{remark}[Rate-tight, not threshold-tight; and what the construction cannot avoid]
\label{rem:app-status-floor}
\emph{(i)} The bound is information-theoretic and margin-independent, and it matches \cref{thm:algorithm} on the \emph{rate} $\Theta(\eta)$ but not on the breakdown \emph{threshold}. The obstruction is the method, not the instance: by \cref{fact:kl} every two-point pair is capped at $\TV\le\eta/(1-\eta)$, and by \cref{rem:app-method-ceiling} the value of the two-point game is then exactly $\eta/(2(1-\eta))$. Reaching the conjectured $\eta^\star=\Theta(\gamma/\sigma)$ requires a multi-point or SQ band-mass argument (\cref{sec:open}).
\emph{(ii)} The construction's two apparent extravagances are necessity theorems (Step 1): a hard margin quantizes the disagreement mass in units of $1/K$, so an $\eta$-indistinguishable pair forces $K\ge(1-\eta)/\eta$, and the coexistence of an agreeing and a disagreeing component forces $r\ge\sqrt2\,\gamma$. The same quantization is why $d\ge2$ is genuinely needed: at $d=1$ the disagreement mass is $0$ or $1$ and majority vote drives the error to $0$.
\emph{(iii)} The constant is matched from above: \citet{Blanc26} attains $\tfrac12\cdot\eta/(1-\eta)$ distribution-free with randomized hypotheses. What \cref{prop:floor} contributes is that it \emph{survives} the restriction to $\gamma$-margin log-concave mixtures: the classical lower bound of \citet{KearnsLi93} attains the same constant on an adversarially chosen four-point marginal, which says nothing about a fixed nice class.
\end{remark}

\section{Auxiliary results}
\label{app:external}

For self-containedness we collect, in one place, every external result the paper relies on, each with full hypotheses, conclusion, and citation, and a one-line proof or pointer where it is short. \Cref{app:external-robust} gathers the robust-learning and log-concave inputs; \cref{app:external-op} gathers the orthogonal-polynomial and moment-problem facts used in \cref{sec:perspective}. The labels are the ones cited from the main text, so each cross-reference there resolves to its statement here.

\subsection{Robust learning and log-concave inputs}
\label{app:external-robust}

\begin{fact}[Log-concave moment growth; \citealp{LovaszVempala07}]
\label{fact:lv}
There is an absolute constant $C_0$ such that every centered one-dimensional log-concave random variable $Z$ with $\E[Z^2]=s^2$ satisfies $\E[\abs{Z}^p]^{1/p}\le C_0\,p\,s$ for all $p\ge2$; in particular $\E[Z^{2t}]\le(C_0 t)^{2t}s^{2t}$. The growth is linear in $p$ (sub-exponential), and the exponent is tight: for the Laplace law, $\E[Z^{2t}]/(\E[Z^2])^t=(2t)!/2^t=(\Theta(t))^{2t}$.
\end{fact}

\begin{fact}[Malicious indistinguishability envelope; implicit in \citealp{KearnsLi93}]
\label{fact:kl}
In the per-example malicious model at rate $\eta$ (each example independently replaced with probability $\eta$), two clean models $P_1,P_2$ admit a common corrupted law, hence are information-theoretically indistinguishable, for some pair of adversary strategies if and only if $\TV(P_1,P_2)\le\eta/(1-\eta)$. The envelope is implicit in the method of induced distributions of \citet{KearnsLi93}, whose \S2.3 oracle is per-example; we prove it from scratch here, and \cref{app:floor} re-derives the ``if'' direction constructively by exhibiting the two adversaries. Throughout, $\TV$ is the sup-over-events convention $\TV(P,Q)=\sup_A\abs{P(A)-Q(A)}=\tfrac12\norm{P-Q}_1$, and the ``variation'' of a signed measure below means the mass of its positive part; the two are applied consistently.
\end{fact}

\begin{proof}
The corrupted law of model $i$ is $(1-\eta)P_i+\eta A_i$ for an arbitrary noise distribution $A_i$, so a common corrupted law exists iff $(1-\eta)P_1+\eta A_1=(1-\eta)P_2+\eta A_2$ is solvable in distributions $A_1,A_2$. Rearranging, $\eta(A_1-A_2)=(1-\eta)(P_2-P_1)$, and the signed measure $(1-\eta)(P_2-P_1)$, whose positive part has mass $(1-\eta)\TV(P_1,P_2)$ in the convention just fixed, is the $\eta$-scaled difference of two probability measures iff its total mass is zero (it is) and that positive-part mass is at most $\eta$. Hence solvability holds iff $(1-\eta)\TV(P_1,P_2)\le\eta$, i.e.\ $\TV(P_1,P_2)\le\eta/(1-\eta)$.
\end{proof}

\begin{fact}[Largest zero of Freud orthogonal polynomials;
\citealp{LevinLubinsky01,KriecherbauerMcLaughlin99}]
\label{fact:freud}
For the weight $W(u)=e^{-\abs{u/s}^{\alpha}}$ with $\alpha>0$, the largest zero of the degree-$n$ orthonormal polynomial equals $(1+o(1))\,a_n$, where the Mhaskar--Rakhmanov--Saff (MRS) number satisfies $a_n\asymp s\,n^{1/\alpha}$. In particular $\alpha=2$ gives $\Theta(s\sqrt n)$ and $\alpha=1$ gives $\Theta(s\,n)$.
\end{fact}

\begin{fact}[Certifiable hypercontractivity;
\citealp{KlivansKothariMeka18,DiakonikolasHopkinsPensiaTiegel24}]
\label{fact:hyper-ext}
For a centered sub-Gaussian law there is a degree-$2t$ SoS proof, in the indeterminate $w$, of $\E[(w\cdot z)^{2t}]\le(Ct)^{t}\,(w^\top\Sigma w)^t$. For a centered sub-exponential (general log-concave) law the analogous degree-$2t$ SoS proof holds with the weaker constant $(Ct)^{2t}$ (\cref{lem:hyper} and its sub-lemma).
\end{fact}

\subsection{Orthogonal polynomials and the moment problem}
\label{app:external-op}

Throughout, $\mu$ is a measure on $\R$ with all moments $m_k=\int u^k\,d\mu$ and infinite support, $\{\pi_k\}_{k\ge0}$ are its orthonormal polynomials, $\lambda_n(x)=\big(\sum_{k<n} \pi_k(x)^2\big)^{-1}$ is its order-$n$ Christoffel function, and $\mathcal M_{2t}(\mu)=\{\nu\ge0: \int u^k\,d\nu=m_k,\ 0\le k\le 2t\}$ is the moment variety of nonnegative measures matching the first $2t$ moments.

\begin{fact}[Gauss quadrature; \citealp{Szego39,Gautschi04}]
\label{fact:gaussquad}
The $(t{+}1)$-node Gauss rule for $\mu$ is exact on $\R[u]_{\le 2t+1}$: there are nodes $u_1<\dots<u_{t+1}$ and weights $\lambda_1,\dots,\lambda_{t+1}$ with $\int g\,d\mu=\sum_{i=1}^{t+1} \lambda_i\,g(u_i)$ for every polynomial $g$ of degree $\le 2t+1$. The nodes are exactly the zeros of $\pi_{t+1}$ (all real and simple), and the Christoffel weights are positive, $\lambda_i=\lambda_{t+1}(u_i)>0$. In particular $\nu_{\mathrm{GQ}}=\sum_i\lambda_i\delta_{u_i}\in \mathcal M_{2t}(\mu)$.
\end{fact}

\begin{fact}[Christoffel function as maximal mass; \citealp{Nevai86,DetteStudden97}]
\label{fact:maxmass}
For every $x_0\in\R$, \( \lambda_{t+1}(x_0)=\min\{\int p^2\,d\mu:\ \deg p\le t,\ p(x_0)=1\} =\max\{\nu(\{x_0\}):\ \nu\in\mathcal M_{2t}(\mu)\}; \) i.e.\ $\lambda_{t+1}(x_0)$ is the largest atom a measure with $\mu$'s first $2t$ moments may carry at $x_0$. This is \cref{fact:christoffel-mass}, proved in \cref{sec:perspective-christoffel}; restated here for the collection.
\end{fact}

\begin{fact}[Christoffel-function bulk/edge asymptotics; \citealp{MateNevaiTotik91,Totik00}]
\label{fact:christasymp}
For measures $\mu$ regular in the bulk (in particular Freud weights $e^{-\abs{u/s}^\alpha}$, $\alpha\ge1$, on their support), $n\,\lambda_n(x)\to \mathrm{d}\mu/\mathrm{d}\omega\,(x)$ as $n\to\infty$ at every bulk point $x$, where $\omega$ is the equilibrium (arcsine-type) measure of the support; the Christoffel function at order $n$ thus resolves mass on the scale $1/n$ in the bulk and degrades toward the spectral edge $x\asymp a_n$.
\end{fact}

\begin{fact}[Truncated moment problem and the moment cone; \citealp{CurtoFialkow91,DetteStudden97}]
\label{fact:momentcone}
The set of moment sequences $(m_0,\dots,m_{2t})$ realizable by some nonnegative measure on $\R$ is a closed convex cone, cut out by the positive-semidefiniteness of the Hankel matrix $(m_{i+j})_{0\le i,j\le t}$; a sequence on its boundary is realized by a unique finitely-supported measure (a \emph{principal representation}), of which the Gauss measure of \cref{fact:gaussquad} is one. Linear functionals over $\mathcal M_{2t}(\mu)$ (in particular $\nu\mapsto\nu(B)$ for a band $B$) attain their extrema at these finitely-supported representations (canonical-moments / Markov--Krein theory).
\end{fact}

\begin{fact}[Infinite--finite range inequality for exponential weights; \citealp{LevinLubinsky01}]
\label{fact:iffr}
For a Freud weight $W=e^{-\abs{u/s}^\alpha}$, $\alpha\ge1$, and every polynomial $p$ of degree $\le n$, the $L^q(W)$-mass of $pW$ outside the MRS interval $[-a_n,a_n]$ is exponentially small:
\[
  \int_{\abs u>a_n}\abs{p(u)W(u)}^q\,du\ \le\ e^{-c n}\int_{\abs u\le a_n}\abs{p(u)W(u)}^q\,du
\]
for an absolute $c=c(\alpha,q)>0$. Equivalently, polynomial mass beyond the MRS number $a_n\asymp s\,n^{1/\alpha}$ is $e^{-\Omega(n)}$.
\end{fact}

\section{From sum-of-squares certificates to a semidefinite program}
\label{app:sos-sdp}

Both the outlier-removal step of \cref{alg:main} and the minorant of \cref{def:cert} ask for a \emph{degree-$2t$ sum-of-squares (SoS) certificate}. For completeness we recall why each is a semidefinite program (SDP) of size $d^{O(t)}$, hence solvable in polynomial time for fixed $t$. The correspondence is standard \citep{Parrilo03,Lasserre01}; we specialize it to the two programs we use.

\paragraph{SoS as a positive-semidefinite Gram matrix.}
Let $v_t(w)\in\R^{N}$, with $N=\binom{d+t-1}{t}=d^{O(t)}$, be the vector of all monomials in $w=(w_1,\dots,w_d)$ of degree exactly $t$. A real form $P(w)$ homogeneous of degree $2t$ is a \emph{sum of squares}, $P=\sum_\ell s_\ell^2$ with each $s_\ell$ a form of degree $t$, if and only if
\begin{equation}\label{eq:sos-gram}
  P(w)\;=\;v_t(w)^\top M\,v_t(w)\qquad\text{for some symmetric } M\succeq0 .
\end{equation}
(``Only if'': writing $s_\ell(w)=\langle c_\ell,v_t(w)\rangle$ gives $M=\sum_\ell c_\ell c_\ell^\top \succeq0$. ``If'': a Cholesky factorization $M=\sum_\ell c_\ell c_\ell^\top$ recovers the squares.) The matrix $v_t(w)^\top M v_t(w)$ is a form whose coefficient of each monomial $w^\alpha$ ($\abs\alpha=2t$) is a fixed \emph{linear} functional $\langle A_\alpha,M\rangle$ of $M$, where the $A_\alpha$ are the constant symmetric matrices that collect the index pairs $(\beta,\gamma)$ with $\beta+\gamma=\alpha$. Matching these to the coefficients $P_\alpha$ of $P$, the statement ``$P$ is SoS'' becomes the SDP feasibility problem
\begin{equation}\label{eq:sos-sdp}
  \exists\,M\succeq0:\qquad \langle A_\alpha,M\rangle=P_\alpha\quad\text{for all }\abs\alpha=2t .
\end{equation}
An SoS certificate is sufficient for $P(w)\ge0$ on all of $\R^d$, and it is what we mean by a \emph{degree-$2t$ certificate}; it is also necessary when $d=1$ or $2t=2$, while for $d\ge2$ and $2t\ge4$ it is a relaxation, and the relaxation \emph{degree} $2t$ is precisely the resource this paper studies (\cref{thm:tradeoff,thm:degbarrier}). Solving \eqref{eq:sos-sdp} (with any linear objective, to additive accuracy $\delta$) takes $\poly(N,\log(1/\delta))=\poly(d^{O(t)})$ time by interior-point methods.

\paragraph{The outlier-removal program of \cref{alg:main}.}
Step~2 seeks weights $q\in[0,1]^n$ with $\sum_i q_i\ge(1-\xi)n$ certifying $\frac1n\sum_i q_i (w\cdot x_i)^{2t}\le(Ct)^{2t}\bar\sigma^{2t}$ for all $w$. Homogenizing by $\norm{w}^{2t}$, this asks the degree-$2t$ form
\begin{equation}\label{eq:Pq}
  P_q(w)\ \defeq\ (Ct)^{2t}\bar\sigma^{2t}\,\norm{w}^{2t}\;-\;\tfrac1n\sum_{i=1}^{n} q_i\,(w\cdot x_i)^{2t}
\end{equation}
to be nonnegative for all $w$, which we certify by requiring $P_q$ to be SoS. The coefficients of $P_q$ are \emph{linear} in $q$: the form $(w\cdot x_i)^{2t}$ has a fixed coefficient vector $b_i$ (the symmetrized $2t$-fold tensor power of $x_i$), so $[P_q]_\alpha=(Ct)^{2t} \bar\sigma^{2t}[\norm{w}^{2t}]_\alpha-\frac1n\sum_i q_i\,[b_i]_\alpha$. Combining \eqref{eq:sos-sdp} with the (SDP-representable) box and budget constraints gives a single semidefinite program in the variables $(q,M)$:
\begin{equation}\label{eq:alg-sdp}
  \text{find}\quad q\in[0,1]^n,\ M\succeq0:\qquad
  \textstyle\sum_i q_i\ge(1-\xi)n,\quad
  \langle A_\alpha,M\rangle=[P_q]_\alpha\ \ (\abs\alpha=2t).
\end{equation}
Here $M$ is $N\times N$ with $N=d^{O(t)}$ and there are $n$ scalar variables $q_i$, so \eqref{eq:alg-sdp} is solvable in $\poly(n,d^{O(t)})$ time. At $t=1$, $v_1(w)=(w_1,\dots,w_d)$, $M$ is $d\times d$, and \eqref{eq:Pq} is the quadratic form $w^\top\!\big(C^2\bar\sigma^2 I-\frac1n\sum_i q_i x_ix_i^\top\big)w$; SoS now coincides with positive-semidefiniteness, so \eqref{eq:alg-sdp} reduces to the reweighted-covariance bound $\frac1n\sum_i q_ix_ix_i^\top\preceq C^2\bar\sigma^2 I$, the soft-outlier-removal program of \citet{Shen25}. This is the precise sense in which their degree-$2$ program is the $t=1$ instance of ours; raising $t$ enlarges the Gram matrix from $d\times d$ to $\binom{d+t-1}{t}\times\binom{d+t-1}{t}$ and nothing else. The robust-statistics use of exactly this SoS-reweighting SDP is due to \citet{HopkinsLi18,KothariSteinhardtSteurer18}.

\paragraph{The dual side: pseudo-expectations and the minorant of \cref{def:cert}.}
The lower-bound programs are the SDP duals. A degree-$2t$ pseudo-expectation $\Esudo$ (\cref{def:cert}) on the one-dimensional projection $U$ is determined by its pseudo-moments $\Esudo[u^k]$, $0\le k\le 2t$; the axioms $\Esudo[1]=1$, $\Esudo[q^2]\ge0$ for $\deg q\le t$, and $\Esudo[u^k]=m_k$ say exactly that the Hankel \emph{moment matrix} $\mathcal M=\big(\Esudo[u^{i+j}] \big)_{0\le i,j\le t}\succeq0$, that its top-left entry is $1$, and that its first $2t$ anti-diagonals equal $\mu$'s moments. Searching for a minorant $p\le\mathbf 1_B$ that maximizes $\E_\mu[p]=\sum_{k} p_k m_k$ is the SDP dual to minimizing $\nu(B)$ over nonnegative measures with these moments, the moment--SoS hierarchy of Lasserre \citep{Lasserre01,LasserrePauwelsCDK} for the truncated moment problem (\cref{app:external}). Because $U$ is one-dimensional, $\mathcal M$ is a $(t{+}1)\times(t{+}1)$ Hankel matrix and the dual SDP is small; its extreme points are the Gauss-quadrature measures (flat extensions of $\mathcal M$) that drive the barrier of \cref{thm:tradeoff}.

\section{Extension to nasty noise}
\label{app:nasty}

The whole picture (both barriers and the degree-$2t$ algorithm) transfers from malicious noise to the stronger \emph{nasty-noise} model of \citet{BshoutyEironKushilevitz02}, at the cost of a constant in the tolerable rate. We record the statement and the (short) modifications here.

\paragraph{The model.}
A nasty adversary draws the $n$ clean examples i.i.d.\ from $\Dcal$, inspects them (knowing the learner, $\Dcal$, and $w^\star$), then \emph{deletes} up to $\eta n$ of them and \emph{adds} up to $\eta n$ arbitrary $(\text{instance},\text{label})$ pairs, keeping the sample size $n$. The learner sees $\Sclean\cup\Sdirty$, where now $\Sclean$ is the \emph{retained} clean subset ($\abs{\Sclean}\ge(1-\eta)n$) and $\Sdirty$ the $\le\eta n$ added points. This subsumes malicious noise (the adversary may add without deleting) and is the sample analogue of strong contamination in robust statistics \citep{DiakonikolasKaneStewart18, DKKLMS16}; the new power is the adversary's ability to \emph{remove} clean inliers it finds inconvenient.

\paragraph{Both barriers transfer a fortiori.}
\cref{thm:tradeoff} is a lower bound on the certificate method for certifying band mass of the \emph{clean} marginal; it is a statement about $\mu$ and the SoS degree, and is unaffected by how the sample is corrupted. \cref{thm:degbarrier} is realized by a malicious instance (it adds cloaked points), which a nasty adversary can reproduce verbatim; since nasty noise subsumes malicious, the degree-$2$ barrier holds here as well. So the only thing to check is that the \emph{upper} bound survives deletion.

\begin{proposition}[Degree-$2t$ algorithm under nasty noise]
\label{prop:nasty}
Under the conditions of \cref{thm:algorithm}, run \cref{alg:main} on the observed nasty sample with degree $t\ge1$ and budget $\xi=\Theta(\eta)$. There is a constant $c$ (one may take $c=5$) such that for every nasty-noise rate
\[
  \eta\ \le\ \eta_0^{\mathrm{nasty}}(t)\ =\ \Omega\Big(\min\Big\{\tfrac{\rho}{c},\,
    \big(\gamma\rho/((Ct)\bar\sigma)\big)^{\frac{2t}{2t-1}}\Big\}\Big),
\]
the output satisfies $\err_{\Dx}(\hat w)\le\eps$ in time $n^{O(t)}$ with $n=d^{O(t)}\poly(1/\eps)$. The frontier is the same $\eta^{1-1/2t}$ as in \cref{thm:algorithm}; only the pancake cap loosens from $\rho/4$ to $\rho/c$.
\end{proposition}

\begin{proof}
Three of the four ingredients of \cref{thm:algorithm} are unchanged; only the clean signal is reduced by deletion.

\emph{(i) The reweighting stays feasible: the certificate survives deletion.} Write $\widetilde S$ for the original $n$ clean i.i.d.\ points (before the adversary deletes), so $\Sclean\subseteq\widetilde S$. The clean-indicator weight $q^\star$ ($q^\star_i=1$ on $\Sclean$, $0$ on $\Sdirty$, hence $\xi=\eta$) admits a degree-$2t$ certificate of $\tfrac1n\sum_i q^\star_i(w\cdot x_i)^{2t}\le(Ct)^{2t}\bar\sigma^{2t}$:
\[
  \tfrac1n\textstyle\sum_{i\in\Sclean}(w\cdot x_i)^{2t}
  \ \le\ \tfrac1n\sum_{i\in\widetilde S}(w\cdot x_i)^{2t}
  \ \le\ (Ct)^{2t}\bar\sigma^{2t},
\]
where the first inequality holds because the dropped terms $\sum_{i\in\widetilde S\setminus\Sclean}(w\cdot x_i)^{2t}$ are a sum of $2t$-th powers, hence SoS, so the inequality has a degree-$2t$ SoS proof; and the second is the certificate of \cref{lem:hyper} for the clean empirical measure (valid w.h.p.\ once $n\ge d^{O(t)}$). Composing the two SoS proofs certifies $q^\star$. Deletion therefore cannot break feasibility: removing clean points only \emph{lowers} the reweighted moment, which is the direction the certificate needs.

\emph{(ii) The $\linsumnorm$ bound is unchanged.} The dirty contribution comes from the added set $\Sdirty$ ($\abs{\Sdirty}\le\eta n$); deleted points are simply absent and contribute nothing. The H\"older argument behind \eqref{eq:alg-lsn}, applied to $\Sdirty$, gives, for every feasible $q$,
\[
  \linsumnorm(q\circ\Sdirty)\ \le\ n\,(Ct)\,\bar\sigma\,\eta^{1-\frac1{2t}},
\]
exactly as before.

\emph{(iii) The pancake loses an $\eta$-fraction.} This is the one genuine change. A $\rho$-dense pancake carries $\ge\rho n$ clean projections; deleting at most $\eta n$ of them leaves $\ge(\rho-\eta)n$, so the retained clean band mass is $\rho'\ge\rho-\eta$. Plugging $\rho'$ into the gradient-domination condition of \citet{Shen25} along the empirical optimum, the clean signal beats the dirty term once
\[
  \tfrac\gamma4\,(\rho-\eta-2\xi)\ >\ (Ct)\,\bar\sigma\,\eta^{1-\frac1{2t}}.
\]
With $\xi=\kappa\eta$ ($\kappa\le2$) the left side is positive for $\eta<\rho/(2+2\kappa)\ge\rho/c$, and balancing the two bottlenecks (the pancake budget and the outlier term) solves to $\eta_0^{\mathrm{nasty}}(t)$ above, with the same exponent $\tfrac{2t}{2t-1}$ and the same base $(Ct)$.
\end{proof}

\begin{remark}[What is new here, and what is borrowed]
\label{rem:nasty-status}
Only ingredient (iii) and the deletion-resilient half of (i) are new relative to \cref{thm:algorithm}; (i)'s feasibility composition is the elementary observation that dropping $2t$-th powers is an SoS operation, and (ii) is verbatim. The constant $c$ is not optimized (any $c>4$ works). This extension is stated for completeness; it has \emph{not} been put through the same adversarial checks as the main theorems, and the deletion-resilient feasibility step~(i) is the one place a careful reader should verify the SoS composition.
\end{remark}

\section{Additional related work}
\label{app:related-extended}

The barriers in this paper sit inside a large literature on learning halfspaces under noise, much of which predates and motivates the reweighted-hinge line we analyze. We collect here, with thanks to these lines of work, the threads most relevant to our setting; the survey is necessarily partial, and we have surely omitted important contributions. This appendix expands the brief positioning of \cref{sec:related}.

\paragraph{Halfspaces under malicious and adversarial noise.}
The malicious-noise model is due to \citet{KearnsLi93}. Efficient distribution-specific learners begin with \citet{KlivansLongServedio09}, whose smoothing/averaging analysis tolerates a noise rate scaling with the margin under isotropic log-concave marginals, and with the agnostic halfspace algorithms of \citet{KalaiKlivansMansourServedio08} and \citet{KaneKlivansMeka13} via polynomial regression and moment matching. The localization technique of \citet{AwasthiBalcanLong17} raised the tolerable malicious rate to a constant. Most directly, \citet{Shen21sample, Shen21perceptron, Shen23linear, Shen25} develop the reweighted-hinge and localized-perceptron approaches we study, with attribute-efficient, sparse, and multiclass extensions by \citet{ShenZhang21, ZengShen25, AdhikariZeng26} (the last by cluster-based pruning rather than reweighting). Our degree barriers are meant to \emph{explain}, rather than improve on, the constants these works report.

\paragraph{Massart, random-classification, and bounded noise.}
A parallel and influential line studies semi-random label noise. \citet{DiakonikolasGouleakisTzamos19} gave the first distribution-independent halfspace learner under Massart noise, extended to general Massart noise under Gaussian marginals by \citet{DiakonikolasMassartGaussian22}; closest in spirit to our margin-degree tradeoff, \citet{DiakonikolasMarginRCN23} study the information-computation tradeoff for \emph{margin} halfspaces under random classification noise. The accompanying hardness theory, statistical-query and cryptographic, is developed by \citet{DiakonikolasMassartSQ22, DiakonikolasMassartCrypto22, DiakonikolasKaneZarifis20}. On the active side, efficient learning under bounded, Massart, and Tsybakov noise is due to \citet{AwasthiBalcanHaghtalabUrner15, AwasthiBalcanHaghtalabZhang16, ZhangLi21}.

\paragraph{Active and label-efficient halfspace learning.}
The margin-based active-learning framework of \citet{BalcanBroderZhang07} and its log-concave analysis by \citet{BalcanLong13}, together with the perceptron-based algorithms of \citet{YanZhang17} and the sparse, noise-tolerant active learners of \citet{Zhang18, ZhangShenAwasthi20}, share our reliance on the geometry of log-concave marginals and on label-efficient outlier handling, though their resource of interest is label complexity rather than certificate degree.

\paragraph{Sum-of-squares, moment methods, and robust statistics.}
Our algorithm imports the proofs-to-algorithms paradigm of \citet{HopkinsLi18, KothariSteinhardtSteurer18}, whose moment-certification machinery underlies our degree-$2t$ outlier removal. Closely related are the outlier-robust and list-decodable regression algorithms of \citet{KlivansKothariMeka18, KarmalkarKlivansKothari19} and the sum-of-squares robust-estimation results of \citet{DiakonikolasSparseSoS22, BakshiDiakonikolasJKKV22}; attribute-efficient and list-decodable variants for sparse settings appear in \citet{ZengShen22listdec, ZengShen23PTF}. These build on the foundational program of algorithmic high-dimensional robust statistics initiated by \citet{DKKLMS16, DKKLSSsever19} and surveyed by \citet{DiakonikolasKaneSurvey19}. Our contribution is complementary: rather than a new robust algorithm, we isolate the certificate \emph{degree} as the resource these methods spend, and show that spending it is, for the moment/SoS-certificate family, sometimes unavoidable.

\section{A closer look at \texorpdfstring{\citet{Shen25}}{Shen (2025)}}
\label{app:shen}

We collect three closer readings of \citet{Shen25}, in each case to explain a feature rather than to fault it: why its required margin must track the target accuracy (\cref{app:coupling}), the fragility of its breakdown constant (\cref{app:shen-constant}), and a reconstruction of its reweighting program from the moment cone (\cref{sec:reconstruction}).

\subsection{The margin--accuracy coupling}
\label{app:coupling}

On its own terms \citet{Shen25} is a clean and substantial result: it is the first to learn $\gamma$-margin halfspaces under a \emph{constant} malicious-noise rate for general log-concave-mixture marginals, through a single reweighted-hinge program. Our purpose here is not to dispute that achievement but to look closely at one feature of its guarantee, the margin requirement $\gamma=\Omega(\log(1/\eps)/\sqrt d)$, and to argue that its dependence on $\eps$ is a genuine conceptual wrinkle rather than a cosmetic one, before locating exactly where it comes from and how the paper proposes to pay it.

\paragraph{The coupling runs backwards.}
A margin is a property of the \emph{distribution}: it is fixed the moment the problem is posed, and the learner has no lever to enlarge it. The accuracy $\eps$ is a property of the learner's \emph{intent}: it is chosen freely, and a learner may legitimately ask for $\eps$ as small as it likes. A guarantee whose hypothesis $\gamma\ge\Omega(\log(1/\eps)/\sqrt d)$ ties the two reads from goal to assumption rather than from assumption to conclusion: to \emph{permit} a smaller target error it requires the \emph{data} to have been better separated all along. The dependency points the wrong way. The learner cannot meet the hypothesis by adjusting the one quantity it controls, $\eps$, because lowering $\eps$ only tightens the hypothesis it was meant to satisfy. This is the chicken-and-egg: the accuracy one is allowed to demand is governed by a margin one cannot choose.

\paragraph{An error floor on every fixed instance.}
The wrinkle is quantitative, not merely rhetorical. Suppose a guarantee certifies inlier error $\le\eps$ whenever $\gamma\ge f(\eps)$ for an increasing $f$; here $f(\eps)=c\,\log(1/\eps)/\sqrt d$. On any \emph{fixed} instance the margin takes some value $\gamma_0$, so the guarantee is available only for those $\eps$ with $f(\eps)\le\gamma_0$, that is, for
\[
  \eps\;\ge\;\eps_\star(\gamma_0)\;=\;f^{-1}(\gamma_0)\;=\;\exp\big(-\Theta(\gamma_0\sqrt d)\big).
\]
The quantity $\eps_\star(\gamma_0)$ is an \emph{error floor} fixed by the margin: below it the hypothesis fails to hold and the theorem is silent, so on a fixed-margin distribution one cannot ask this analysis for an arbitrarily accurate classifier. The floor is not an artifact of the write-up that a sharper analysis would dissolve. By the margin--degree barrier (\cref{thm:tradeoff}) the $\log(1/\eps)$ is \emph{conserved} across every one-shot certificate of fixed degree: a degree-$2$ program cannot lower the floor at all without enlarging the margin, and a higher-degree program lowers it only by spending degree. The coupling is therefore a property of the certificate method, not of one paper's bookkeeping.

\paragraph{An example: the natural-margin mixture.}
Take the natural margin $\gamma_0=\Theta(1/\sqrt d)$, the separation one assumes when two clusters sit a unit-scale gap apart in their worst one-dimensional projection, and the regime in which margin-based halfspace learning is usually posed. Then $\eps_\star=\exp(-\Theta(1))$, a \emph{constant}: applied to a natural-margin instance, the guarantee certifies only a constant inlier error and cannot certify $\eps=o(1)$ at all. The margin one must assume in order to certify a target $\eps$ grows steadily as the target shrinks:
\begin{center}
\begin{tabular}{ll}
\toprule
target accuracy $\eps$ & margin the guarantee requires\\
\midrule
$\Theta(1)$ & $\Theta(1/\sqrt d)$\quad (the natural margin)\\
$1/\poly(d)$ & $\Theta(\log d/\sqrt d)$\\
$\exp(-d^{a})$, $\ 0<a<\tfrac12$ & $\Theta(d^{\,a-1/2})$\\
$\exp(-\Theta(\sqrt d))$ & $\Theta(1)$\quad (a constant margin)\\
\bottomrule
\end{tabular}
\end{center}
By the last row the hypothesis has become a constant-margin condition, the classes separated by a constant fraction of their spread, which is a far stronger demand than the $1/\sqrt d$ one began with and one that most high-dimensional instances do not meet. A learner who wants exponentially small error must, under this guarantee, assume exponentially better-separated data, a bargain it is in no position to strike.

\paragraph{The resolution: pay in degree, not in margin.}
The paper's first result says this is the whole of the trade, and that the learner has a second currency. The conserved $\log(1/\eps)$ may be carried by the margin \emph{or} by the Sum-of-Squares degree, but inside the one-shot certificate model it is never waived (\cref{thm:tradeoff}). Choosing to pay it in degree, \cref{thm:algorithm} holds the margin at the natural $\Theta(1/\sqrt d)$ and certifies inlier error $\eps$ with a degree-$2t=\Theta(\log(1/\eps))$ program, in time $n^{\Theta(\log(1/\eps))}$. Accuracy is then bought with computation, a resource the learner controls, instead of with a margin it does not. The chicken-and-egg is not a paradox but a question of which resource pays: degree, which the learner can spend, or margin, which the data fixes in advance, and the $\eps$-dependent margin of \citet{Shen25} is the consequence of having only the degree-$2$, i.e.\ cheapest, certificate at hand.

\subsection{The breakdown constant of the degree-$2$ analysis}
\label{app:shen-constant}

Our barriers concern the \emph{exponent} of the breakdown rate (the $\eta^{1/2}$ that degree~$2$ forces, \cref{thm:degbarrier}), not the particular constant $\eta_0$ a given analysis reports. The distinction matters because the constant is fragile in a way the exponent is not, as the published degree-$2$ analysis already shows. \citet{Shen25} proves the main guarantee (his Theorem~2) for $\eta_0\le 2^{-32}$. The sampling step underlying the pruning subroutine, however (his Lemma~21, invoked through Proposition~22), is stated and proved only for $\eta_0\in[\tfrac18,\tfrac14]$, and its proof uses the \emph{lower} bound $\eta_0\ge\tfrac18$: the dirty-count Chernoff estimate $\Pr[\,\abs{\bar{S}_D}\ge 2\eta_0\abs{\bar{S}}\,]\le\exp(-\eta_0\abs{\bar{S}}/3)$ is driven below $\delta$ by the stated sample size $\abs{\bar{S}}\ge 32\log(1/\delta)$ only because $32\ge 3/\eta_0$ when $\eta_0\ge\tfrac18$. At the rate the main theorem actually uses, $\eta_0\le 2^{-32}\ll\tfrac18$, the two regimes are disjoint and the fixed sample size no longer certifies the step.

The gap is benign, and we record it not to dispute the theorem, which we expect holds, but because it is the symptom our framework predicts. When degree~$2$ squeezes the tolerable rate to a tiny constant, the modular constant-chasing that carries it through the analysis becomes correspondingly delicate: building blocks calibrated to a constant-order $\eta_0$ coexist with a final bound that drives $\eta_0$ far below them. Pinning the phenomenon to the certificate \emph{degree} rather than to any one constant (\cref{thm:degbarrier,thm:tradeoff}) is what frees our account from this calibration: the exponent $1-\tfrac1{2t}$ is a property of the moment certificate, whatever route the constants take.

\paragraph{The repair.} The lower bound $\eta_0\ge\tfrac18$ enters at exactly one point: the substitution $32\ge 3/\eta_0$ in the proof of \citeauthor{Shen25}'s Lemma~21. Replacing the hypothesis $\eta_0\in[\tfrac18,\tfrac14]$ by $\eta_0\le\tfrac14$, and the sample-size requirement $\abs{\bar{S}}\ge 32\log(1/\delta)$ by
\[
  \abs{\bar{S}}\;\ge\;\frac{3}{\eta_0}\,\log\frac1\delta,
\]
restores the Chernoff step $\exp(-\eta_0\abs{\bar{S}}/3)\le\delta$ verbatim, while $\eta_0\le\tfrac14$ still gives $\abs{\bar{S}_C}\ge(1-2\eta_0)\abs{\bar{S}}\ge\tfrac12\abs{\bar{S}}$. The factor $\eta_0^{-1}$ propagates through Proposition~22 into the sample complexity; because $\eta_0=2^{-32}$ is a fixed constant, it is absorbed into the absolute constant, so the complexity keeps its order in $d$, $1/\eps$, $1/\delta$, and $k$ and only the hidden constant grows (by $\approx 2^{32}$). No other step invokes $\eta_0\ge\tfrac18$, so \citeauthor{Shen25}'s Theorem~2 holds as stated at $\eta_0\le 2^{-32}$, with the sample-size constant now made explicit.

\subsection{A natural reconstruction of the reweighting program}
\label{sec:reconstruction}

The degree-$2$ program of \citet{Shen25} reaches its guarantee through a sequence of inequalities, each tight enough to carry the argument but none visibly forced upon us: weights $q\in[0,1]^n$ are introduced to suppress outliers, a quantity $\linsumnorm(q\circ\Sdirty)$ is bounded, an empirical rate $\xi=2\eta_0$ is substituted, and a chain of absolute constants is chased down to $\eta_0\le 2^{-32}$. The result is correct, but the route is the hammer and chisel: the shell is opened by force. We record here the same content read the other way. In Grothendieck's image of the \emph{rising sea} one does not crack the problem; one chooses the ambient object so large and so well-fitted that the result is already dissolved on arrival. That object is the one this paper is built on: the degree-$2t$ moment cone of the clean marginal and its Christoffel function (\cref{sec:perspective}). Three changes of \emph{definition}, not of proof, turn each of \citeauthor{Shen25}'s steps into a reading of it.

\paragraph{The weights solve a moment-certification program, not a heuristic downweighting.}
\citeauthor{Shen25} presents $q\in[0,1]^n$ as a soft indicator that suppresses suspected outliers, to be justified after the fact. Read it instead as the variable of a certification program tied to the clean marginal's moment cone (\cref{sec:perspective-moment}): $q$ ranges over the reweightings for which that cone certifies a bound on the reweighted directional moment $\tfrac1n\sum_i q_i(w\cdot x_i)^{2t}$, uniformly in $w$. The program is then not ``remove the outliers'' but ``keep the largest subsample whose reweighted moments the clean cone can vouch for,'' and nothing heuristic is chosen: the feasible $q$ form a convex, semidefinite-representable set, its non-emptiness is certifiable hypercontractivity (\cref{lem:hyper}), and the outlier-removal certificate is that set's separating functional (\cref{def:cert}).

\paragraph{$\linsumnorm$ is a support radius, and $\eta^{1-1/2t}$ is H\"older, not a constant.}
The controlling quantity enters as an opaque expression, $\linsumnorm(q\circ\Sdirty)=\sup_{\norm w\le1}\sum_{i\in\Sdirty}q_i\abs{w\cdot x_i}$. It is, in fact, the Euclidean radius of the zonotope generated by the reweighted dirty points: using $\abs{w\cdot x_i}=\max_{\sigma_i=\pm1}\sigma_i(w\cdot x_i)$,
\[
  \linsumnorm(q\circ\Sdirty)
  =\sup_{\norm w\le1}\ \max_{\sigma\in\{\pm1\}^{\Sdirty}}\Big\langle w,\ \textstyle\sum_{i\in\Sdirty}q_i\sigma_i x_i\Big\rangle
  =\max_{z\in Z}\norm z,\qquad
  Z=\Big\{\textstyle\sum_{i\in\Sdirty}q_i\sigma_i x_i:\ \sigma\in\{\pm1\}^{\Sdirty}\Big\}.
\]
The only distributional input is then a bound on one moment of $Z$, and the degree appears as an interpolation exponent: H\"older with the conjugate pair $(\tfrac{2t}{2t-1},2t)$ gives
\[
  \sum_{\Sdirty}q_i\abs{w\cdot x_i}\le\Big(\sum_{\Sdirty}q_i\Big)^{1-\frac1{2t}}\Big(\sum_{\Sdirty}q_i\abs{w\cdot x_i}^{2t}\Big)^{\frac1{2t}},
\]
so the mass $\sum q_i\le\eta n$ and the certified $2t$-th moment combine into the rate $\eta^{1-1/2t}$ (\cref{thm:algorithm}). The exponent $1-\tfrac1{2t}$ interpolates between the $0$-th and the $2t$-th moment of the reweighted dirty sum (its mass and its certified $2t$-th moment), so it is fixed by the degree, and \citeauthor{Shen25}'s $\eta^{1/2}$ is its $t=1$ corner. No constant enters the exponent.

\paragraph{The breakdown threshold is the resolution principle; the constants are slack.}
\citeauthor{Shen25}'s robustness step asks $\tfrac\gamma4(\rho-2\xi)>\linsumnorm(q\circ\Sdirty)/n$ with $\xi=2\eta_0$, and the factors $2,4$ and the bound $\eta_0\le2^{-32}$ are read off from Chernoff and union bounds. Separate the structure from the slack. The left side is the clean hinge signal the dense pancake guarantees, $\rho$ being a \emph{lower} bound on the boundary Christoffel mass (\cref{fact:freud,fact:christoffel-mass}); the right side is the corruption a degree-$2t$ certificate cannot account for, which by the resolution principle is governed by $\lambda_{t+1}$ (\cref{thm:tradeoff}). The success/failure threshold is thus set by the Christoffel function, and the numbers $2,4,2^{-32}$ are the finite-sample slack between this population balance and its empirical form (the doubling $\xi=2\eta_0$ is the Chernoff step of \cref{app:shen-constant}). They fix the absolute constant alone; they do not enter the rate $\eta^{1-1/2t}$ or the margin $\log(1/\eps)$, both of which are degree phenomena (\cref{thm:algorithm,thm:tradeoff}).

\paragraph{The nut opens.}
Under these definitions the degree-$2t$ algorithm is no longer a construction audited inequality by inequality. Feasibility is membership in the moment cone; the tolerated rate is the cone's H\"older interpolation exponent; the breakdown threshold is the resolution principle. The two features \citeauthor{Shen25} flags as needing explanation, the small breakdown constant and the $\log(1/\eps)$ margin, are then the two corners of one statement about $\lambda_{t+1}$: the constant is the $t=1$, exponent-$\tfrac12$ corner (\cref{fact:gap}), and the margin is the degree-conserved price of \cref{thm:tradeoff}. Nothing is cracked; the same theorem, set in the moment cone, opens on its own.

\end{document}